\documentclass[11pt,english]{article}
\usepackage{amsmath,amsfonts}
\usepackage{algorithmic}
\usepackage{algorithm}
\usepackage{array}
\usepackage[caption=false,font=normalsize,labelfont=sf,textfont=sf]{subfig}
\usepackage{textcomp}
\usepackage{stfloats}
\usepackage{url}
\usepackage{verbatim}
\usepackage{cite}
\usepackage{graphicx}
\usepackage{authblk}
\usepackage{hyperref}

\usepackage{nicefrac}
\DeclareMathOperator*{\argmin}{arg\,min}
\usepackage{amsmath}
\usepackage{amssymb}
\usepackage{amsfonts}
\usepackage{comment}
\usepackage{mathtools}
\usepackage{algorithm}
\usepackage{graphicx}

\usepackage{xcolor}
\usepackage{enumitem}

\usepackage{bbding}

\newcommand{\mbb}{\mathbb}

\newcommand{\mbe}{\mathbb E}

\newcommand{\lbr}{\left[}
\newcommand{\rbr}{\right]}

\newcommand{\bg}{{\mathbf g}}
\newcommand{\bgt}{{\mathbf g^{t}}}

\newcommand{\brt}{{\mathbf r^{t}}}
\newcommand{\brk}{{\mathbf r^{k}}}

\newcommand{\bx}{{\mathbf x}}
\newcommand{\bxt}{{\mathbf x^{t}}}
\newcommand{\bxk}{{\mathbf x^{k}}}
\newcommand{\bxtp}{{\mathbf x^{t+1}}}
\newcommand{\xit}{\xi^t}

\newcommand{\be}{{\mathbf e}}
\newcommand{\bet}{{\mathbf e^{t}}}
\newcommand{\bek}{{\mathbf e^{k}}}

\newcommand{\bz}{{\mathbf z}}
\newcommand{\bzt}{{\mathbf z^{t}}}
\newcommand{\bzk}{{\mathbf z^{k}}}
\newcommand{\bztj}{{\mathbf z^{t}_j}}
\newcommand{\wbzt}{\widetilde{\mathbf z}^t}
\newcommand{\obzt}{\overline{\mathbf z}^t}

\newcommand{\wP}{\widetilde{P}}

\newcommand{\bPsi}{{\boldsymbol{\Psi}}}
\newcommand{\bPsit}{{\bPsi^{t}}}
\newcommand{\bPhi}{{\boldsymbol{\Phi}}}
\newcommand{\bPhit}{{\bPhi^{t}}}
\newcommand{\obPhi}{{\overline{\boldsymbol{\Phi}}}}
\newcommand{\obPhit}{{\obPhi^{t}}}

\newcommand{\by}{{\mathbf y}}

\usepackage{xcolor}

\usepackage{tikz}
\usetikzlibrary{calc,trees,positioning,arrows,chains,shapes.geometric,%
	decorations.pathreplacing,decorations.pathmorphing,shapes,%
	matrix,shapes.symbols}

\tikzstyle{startstop} = [rectangle, draw, rounded corners, align=center, minimum width=3cm, minimum height=1cm,text centered]
\tikzstyle{decision} = [diamond, draw, fill=blue!20, 
text width=4.5em, text badly centered, node distance=3cm, inner sep=0pt]
\tikzstyle{block} = [rectangle, draw, fill=blue!10, align=center, rounded corners, minimum width=3cm, minimum height=1cm]
\tikzstyle{blockcast} = [rectangle, draw, fill=red!10, align=center, rounded corners, minimum width=3cm, minimum height=0.45cm]
\tikzstyle{line} = [draw, -latex']
\tikzstyle{cloud} = [draw, ellipse,fill=red!20, node distance=3cm,
minimum height=2em]

\newcommand{\R}{\mathbb{R}}
\newcommand{\E}{\mathbb{E}}
\newcommand{\N}{\mathbb{N}}

\newcommand{\calN}{\mathcal{N}}

\newcommand{\bs}{\mathbf{s}}
\newcommand{\bu}{\mathbf{u}}

\newcommand{\Prob}{\mathbb{P}}
\usepackage{amsthm,amsfonts,amsmath,changepage,threeparttable,booktabs,multirow,pifont,xspace}
\newtheorem{assumption}{Assumption}
\newtheorem{remark}{Remark}
\newtheorem{lemma}{Lemma}
\newtheorem{theorem}{Theorem}
\newtheorem{corollary}{Corollary}
\newtheorem{innercustomthm}{Lemma}

\theoremstyle{definition}
\newtheorem{example}{Example}
\theoremstyle{plain}

\usepackage[font=scriptsize]{caption}
\usepackage{a4wide}

\newenvironment{customthm}[1]
  {\renewcommand\theinnercustomthm{#1}\innercustomthm}
  {\endinnercustomthm}

\newcommand{\tnorm}[1]{\left|\mkern-1mu\left|\mkern-1mu\left| #1 \right|\mkern-1mu\right|\mkern-1mu\right|}
\newcommand{\sfo}{$\mathcal{SFO}$\xspace}
\newcommand{\sgd}{\textbf{\texttt{SGD}}\xspace}
\newcommand{\nsgd}{\textbf{\texttt{N-SGD}}\xspace}
\newcommand{\sge}{\textbf{\texttt{SGE}}\xspace}
\newcommand{\nsge}{\textbf{\texttt{N-SGE}}\xspace}
\newcommand{\msge}{\textbf{\texttt{MSGE}}\xspace}
\newcommand{\nmsge}{\textbf{\texttt{N-MSGE}}\xspace}

\begin{document}

\title{Sharp High-Probability Rates for Nonlinear SGD under Heavy-Tailed Noise via Symmetrization}

\author[1]{Aleksandar Armacki}
\author[2]{Dragana Bajovi\'{c}}
\author[3]{Du\v{s}an Jakoveti\'{c}}
\author[4]{Soummya Kar}
\affil[1]{\'Ecole Polytechnique F\'ed\'erale de Lausanne, Lausanne, Switzerland,\linebreak \texttt{aleksandar.armacki@epfl.ch}}
\affil[2]{Faculty of Technical Sciences, University of Novi Sad, Novi Sad, Serbia\\ \texttt{dbajovic@uns.ac.rs}}
\affil[3]{Faculty of Sciences, University of Novi Sad, Novi Sad, Serbia\\ \texttt{dusan.jakovetic@dmi.uns.ac.rs}}
\affil[4]{Carnegie Mellon University, Pittsburgh, PA, USA, \linebreak \texttt{soummyak@andrew.cmu.edu }}

\date{}

\maketitle

\begingroup
\renewcommand\thefootnote{}\footnotetext{The work of A. Armacki and S. Kar was partially supported by NSF, Grant No. 2330196. The work of D. Jakoveti\'{c} was supported by the Ministry of Science, Technological Development and Innovation of the Republic of Serbia (Grants No. 451-03-137/2025-03/200125 \& 451-03-136/2025-03/200125), and by the Science Fund of the Republic of Serbia, Grant No. 7359, Project title LASCADO. The majority of work was done while A. Armacki was a PhD student at Carnegie Mellon University.}
\addtocounter{footnote}{-1}
\endgroup

\begin{abstract}
    We study convergence in high-probability of \sgd-type methods in non-convex optimization and the presence of heavy-tailed noise. To combat the heavy-tailed noise, a general black-box nonlinear framework is considered, subsuming nonlinearities like sign, clipping, normalization and their smooth counterparts. Our first result shows that nonlinear \sgd (\nsgd) achieves the rate $\widetilde{\mathcal{O}}(t^{-1/2})$, {where $t$ denotes the time index}, for any noise with unbounded moments and a symmetric probability density function (PDF). Crucially, \nsgd has exponentially decaying tails, matching the performance of \emph{linear} \sgd under \emph{light-tailed} noise. To handle non-symmetric noise, we propose two novel estimators, based on the idea of noise symmetrization. The first, dubbed Symmetrized Gradient Estimator (\sge), assumes a noiseless gradient at any reference point is available at the start of training, while the second, dubbed Mini-batch SGE (\msge), uses mini-batches to estimate the noiseless gradient. Combined with the nonlinear framework, we get \nsge and \nmsge methods, respectively, both achieving the same convergence rate and exponentially decaying tails as \nsgd, while allowing for non-symmetric noise with unbounded moments and PDF satisfying a mild technical condition, with \nmsge additionally requiring bounded noise moment of order $p \in (1,2]$. Compared to works assuming noise with bounded $p$-th moment, our results: 1) are based on a novel symmetrization approach; 2) provide a unified framework and relaxed moment conditions; 3) imply an oracle complexity of \nsgd and \nsge that is strictly better than existing works when $p < 2$, while the complexity of \nmsge is close to existing works. Compared to works assuming symmetric noise with unbounded moments, we: 1) provide a sharper analysis and improved rates; 2) facilitate state-dependent symmetric noise; 3) extend the strong guarantees to non-symmetric noise.
\end{abstract}

\section{Introduction}

In this work we are interested in the problem of non-convex stochastic optimization under the presence of heavy-tailed noise. In particular, we consider problems of the form
\begin{equation}\label{eq:problem}
    \argmin_{\bx \in \R^d} F(\bx) = \mbe_{\xi } [f(\bx; \xi)],
\end{equation} where $\bx \in \R^d$ represent the model parameters we are trying to learn, $f: \R^d \times \Xi \mapsto \R$ is a loss function, $\xi$ is a random variable residing in some space $\Xi$, while $F: \R^d \mapsto \mbb R$ is the population cost. Formulation \eqref{eq:problem} is widely encountered in many machine learning applications, subsuming problems such as online learning and empirical risk minimization. While many methods for solving \eqref{eq:problem} have been proposed, arguably the most widely-used in modern applications is the celebrated stochastic gradient descent (\sgd). First introduced in the seminal work of Robbins and Monro \cite{robbins1951stochastic}, \sgd and its variants have become the go-to methods for training deep learning models, such as transformers and large language models, e.g., \cite{pmlr-v28-sutskever13,yu2024distributed}. Among them, \emph{nonlinear} \sgd (\nsgd) methods, such as sign, clipped and normalized \sgd, have been attracting increased attention recently, due to their ability to stabilize and speed up training in challenging settings, such as the lack of classical smoothness, e.g., \cite{zhang2019gradient,cutkosky20normalized_SGD,zhang2020improved,sign-momentum}, as well as the privacy and security benefits they bring, e.g., \cite{chen2020understanding,zhang2022clip_FL_icml,shuhua-clipping}. The use of nonlinear methods has a long history dating back to the 70's and 80's and the works of Nevel'son and Has'minski{\u\i}, as well as Polyak and Tsypkin, e.g., \cite{nevel1976stochastic,polyak-adaptive-estimation,polyak1984criterial}, with applications in areas such as estimation and adaptive filtering, e.g., \cite{al2001adaptive,transient-nonlin-filters,kar_estimaton}. Another motivation for using nonlinear methods recently stems from the emergence of heavy-tailed noise in modern learning applications, e.g., \cite{csimcsekli2019heavy,heavy-tail-book}. Starting with the work \cite{zhang2020adaptive}, which noted that the distribution of gradient noise during training of BERT resembles a Levy $\alpha$-stable distribution, a number of works have studied convergence of \nsgd methods in the presence of heavy-tailed noise with \emph{bounded moments of order $p \in (1,2]$}, focusing on clipping and normalization, e.g., \cite{cutkosky2021high,sadiev2023highprobability,nguyen2023improved,hubler2024normalization,liu2025nonconvex} and recently sign \cite{kornilov2025sign}. Bounded $p$-th moment generalizes the bounded variance condition, to include a broad class of heavy-tailed noises, however, convergence rates established under this condition exhibit exponents depending on noise moments and vanishing as $p \rightarrow 1$, failing to explain the strong performance of \nsgd methods observed even in the presence of noise with unbounded first moments, such as the Cauchy noise, see, e.g., \cite{gorbunov2023breaking}. A related line of work shows both empirical and theoretical evidence of \emph{symmetric} noise arising in many modern learning applications, such as during training of deep learning models, e.g., \cite{bernstein2018signsgd,barsbey-heavy_tails_and_compressibility,heavy-tail-phenomena,pmlr-v238-battash24a}. Utilizing symmetry, works \cite{polyak-adaptive-estimation,chen2020understanding,jakovetic2023nonlinear,armacki2024_ldp+mse,armacki2023high}, are able to show strong guarantees for a broad range of popular nonlinearities, including sign, normalization, quantization and clipping. However, such assumption rules out many non-symmetric noises following naturally occurring distributions, like exponential, Pareto, Weibull and Gumbel, to name a few. 

\begin{table*}[!htp]
\caption{Oracle complexity for reaching an $\epsilon$-stationary point in non-convex optimization, under heavy-tailed noise and in high-probability sense, i.e., a point $\bx \in \R^d$, such that $\|\nabla F(\bx)\| \leq \epsilon$, with probability at least $1 - \delta$, for any $\delta \in (0,1)$ and $\epsilon > 0$. Columns ``Noise'' and ``Nonlinearity'' outline the noise assumptions and the nonlinearities facilitated in the work, with ``IID'' standing for \emph{independent, identically distributed}. Columns ``Batch size'' and ``Complexity'' refer to the batch size and overall oracle complexity for reaching an $\epsilon$-stationary point, with $\mathcal{O}(1)$ implying a constant batch is used. All works achieve a logarithmic dependence on the confidence parameter $\delta$, implying an exponentially decaying tail bound. We can see that both \nsgd and \nsge match the {(optimal) complexity} of linear \sgd in the presence of light-tailed noise, strictly better than \cite{nguyen2023improved,hubler2024normalization,kornilov2025sign} whenever $p < 2$, i.e., \emph{any heavy-tailed} noise and the same as majority vote sign \sgd from \cite{kornilov2025sign}, without requiring a mini-batch and under a much broader nonlinear framework. While \nmsge achieves a slightly worse sample complexity compared to \cite{nguyen2023improved,hubler2024normalization,kornilov2025sign}, as we discuss in Section \ref{sec:main} ahead, this is a direct consequence of our black-box treatment of the nonlinearity and the resulting metric. Finally, we can see that our analysis for \nsgd results in an improved complexity compared to \cite{armacki2023high}, while relaxing the IID condition.}
\label{tab:comp}
\begin{adjustwidth}{-1in}{-1in} 
\begin{center}
\begin{threeparttable}
\begin{small}
\begin{sc}
\begin{tabular}{ccccc}
\toprule
\multicolumn{1}{c}{\scriptsize Work} & \multicolumn{1}{c}{\scriptsize Noise} & \multicolumn{1}{c}{\scriptsize Nonlinearity} & \multicolumn{1}{c}{\scriptsize Batch size} & \multicolumn{1}{c}{\scriptsize Complexity}\\
\midrule
\scriptsize\cite{nguyen2023improved} & $\substack{\text{\scriptsize{unbiased, bounded}} \\ \text{\scriptsize{$p$-th moment}}}$ & \scriptsize{clipping} & \scriptsize $\mathcal{O}(1)$ & \scriptsize ${\mathcal{O}}\left(\epsilon^{-\frac{3p - 2}{p - 1}} \right)$ \\
\midrule
\scriptsize \cite{hubler2024normalization} & $\substack{\text{\scriptsize{unbiased, bounded}} \\ \text{\scriptsize{$p$-th moment}}}$ & \scriptsize{normalization} & \scriptsize $\mathcal{O}\left(\epsilon^{-\frac{p}{p-1}}\right)$ & \scriptsize ${\mathcal{O}}\left(\epsilon^{-\frac{3p - 2}{p - 1}} \right)$ \\
\midrule
\scriptsize \cite{kornilov2025sign} & $\substack{\text{\scriptsize{unbiased, component-wise}} \\ \text{\scriptsize{bounded $p$-th moment}}}$ & \scriptsize{sign} & \scriptsize $\mathcal{O}\left(\epsilon^{-\frac{p}{p-1}}\right)$ & \scriptsize ${\mathcal{O}}\left(\epsilon^{-\frac{3p - 2}{p - 1}} \right)$ \\
\midrule
\scriptsize \cite{kornilov2025sign} & \scriptsize unimodal, symmetric & \scriptsize{majority vote sign} & \scriptsize $\mathcal{O}\left(\epsilon^{-2}\right)$ & \scriptsize ${\mathcal{O}}\left(\epsilon^{-4} \right)$ \\
\midrule
\scriptsize \nsgd (\cite{armacki2023high}) & $\substack{\text{\scriptsize{iid and symmetric pdf,}} \\ \text{\scriptsize{positive around origin}}}$ & \scriptsize{general framework} & \scriptsize $\mathcal{O}(1)$ & \scriptsize $\mathcal{O}\left(\epsilon^{-8}\right)$ \\
\midrule
\scriptsize \nsgd (Theorem \ref{thm:n-sgd}) & $\substack{\text{\scriptsize{symmetric pdf,}} \\ \text{\scriptsize{positive around origin}}}$ & \scriptsize{general framework}$^\dagger$ & \scriptsize $\mathcal{O}(1)$ & \scriptsize $\mathcal{O}\left(\epsilon^{-4}\right)$ \\
\midrule
\scriptsize \nsge (Theorem \ref{thm:n-sge}) & $\substack{\text{\scriptsize iid and pdf positive on set} \\ \text{\scriptsize of positive leb. measure}^\ddagger}$ & \scriptsize{general framework} & \scriptsize $\mathcal{O}(1)$ & \scriptsize $\mathcal{O}\left(\epsilon^{-4}\right)$ \\
\midrule
\scriptsize \nmsge (Theorem \ref{thm:n-msge}) & $\substack{\text{\scriptsize unbiased, iid, bounded $p$-th} \\ \text{\scriptsize moment, with pdf positive on} \\ \text{\scriptsize set of positive leb. measure} \\}$ & $\substack{\text{\scriptsize smooth general} \\ \text{\scriptsize framework}^\S}$  & \scriptsize $\mathcal{O}\left(\epsilon^{-\frac{2p}{p-1}} \right)$ & \scriptsize $\mathcal{O}\left(\epsilon^{-\frac{6p-4}{p-1}}\right)$
\\
\bottomrule
\end{tabular}
\end{sc}
\end{small}
\begin{tablenotes}\scriptsize
    \item[$\dagger$] For state-dependent noise the general framework subsumes joint nonlinearities like clipping and normalization and smooth component-wise nonlinearities like smooth versions of sign and component-wise clipping. For IID noise, the framework additionally includes non-smooth component-wise nonlinearities, like standard sign and component-wise clipping, see Section \ref{sec:main} ahead. 
    
    \item[$\ddagger$] Formally, the condition requires the PDF $P: \R^d \mapsto \R_+$ to be such that, for some $E_0 > 0$ and each $\bx \in \R^d$ such that $\|\bx\|\leq E_0$, the set $S_{\bx} = \{\by \in \R^d: P(\by) > 0 \text{ and } P(\by - \bx) > 0 \}$ is of positive Lebesgue measure, see Section \ref{sec:main} for further details.

    \item[$\S$] Including nonlinearities like smooth versions of sign, normalization, component-wise and joint clipping, see Section \ref{sec:main} for details.

\end{tablenotes}
\end{threeparttable}
\end{center}
\vskip -0.1in
\end{adjustwidth}
\end{table*}

\paragraph{Literature review} We next review the related literature on learning in the presence of heavy-tailed noise. In particular, we focus on the results assuming bounded moment of order $p \in (1,2]$, as well as the ones assuming symmetric noise. {The use of the bounded moment of order $p > 1$ condition in the context of (nonlinear) \sgd methods can be traced to the 1980's and the work \cite{nemirovski1983problem}, where the authors study the complexity of a class of stochastic convex problems}. More recently, the work \cite{zhang2020adaptive} first noticed that the distribution of gradient noise during training of BERT resembles a $\alpha$-stable Levy distribution and proposed the use of bounded $p$-th moment assumption. Under this assumption, they establish a mean-squared error (MSE) convergence rate lower bound $\Omega(t^{2(1-p)/(3p-2)})$ for non-convex costs and show the lower bound is achieved by clipped \sgd, with vanilla \sgd failing to converge for any fixed step-size. This result was complemented in the high-probability sense by \cite{sadiev2023highprobability}, who, among other things, showed that vanilla \sgd can not achieve a logarithmic dependence on the confidence parameter,\footnote{We refer here to the parameter $\delta \in (0,1)$, in the sense of the statement ``with probability at least $1 - \delta$''.} even when $p = 2$ (i.e., bounded variance). The authors in \cite{cutkosky2021high} prove that a version of clipped and normalized momentum \sgd matches the said lower bound in the high-probability sense, however, they {require bounded $p$-th moment of both stochastic gradients and noise (implying uniformly bounded gradients of the cost)}, a fixed step-size and a preset time horizon. The bounded stochastic gradient condition was later relaxed by \cite{liu2023breaking}, who additionally showed an accelerated rate, achieved by a version of STORM \cite{cutkosky2019momentum}, combining normalization and gradient clipping. The work \cite{nguyen2023improved} shows that the lower bound (up to a logarithmic factor) is achieved in the high-probability sense for clipped \sgd, using a time-varying step-size without a preset time horizon. Recent works \cite{hubler2024normalization,kornilov2025sign} show that normalized and sign \sgd achieve the optimal oracle complexity for reaching an $\epsilon$-stationary point under the bounded $p$-th moment condition in the high-probability sense.\footnote{Both normalization and sign require an increasing mini-batch, therefore the guarantees are naturally expressed in terms of the overall oracle complexity.} On the other hand, the use of symmetric noise when analyzing nonlinear methods can be traced back to the work of Polyak and Tsypkin \cite{polyak-adaptive-estimation}, who analyzed asymptotic guarnatees of estimation with nonlinear measurements. The authors in \cite{bernstein2018signsgd_iclr} and \cite{chen2020understanding} analyze MSE guarantees of sign and clipped \sgd for non-convex costs utilizing noise symmetry and show the optimal convergence rate $\mathcal{O}(t^{-1/2})$, with \cite{bernstein2018signsgd_iclr} requiring bounded variance, while \cite{chen2020understanding} make no moment requirements. The work \cite{jakovetic2023nonlinear} provides unified guarantees for a broad framework of \nsgd methods under symmetric noise with bounded first moments and strongly convex costs. This work is complemented by \cite{armacki2024_ldp+mse,armacki2023high}, which note that the moment requirement can be removed and study guarantees of the same nonlinear framework for non-convex costs. In particular, the authors in \cite{armacki2024_ldp+mse} establish a large deviation upper bound, implying an asymptotic exponential tail decay, with decay rate $\sqrt{t}/\log(t)$, as well as finite-time MSE convergence, with rate $\widetilde{\mathcal{O}}(t^{-1/2})$, matching the MSE lower-bound for first-order stochastic methods established in \cite{Arjevani2023}. On the other hand, the work \cite{armacki2023high} provides finite-time high-probability guarantees, with rate $\widetilde{\mathcal{O}}(t^{-1/4})$. Another important work is \cite{gorbunov2023breaking}, where the authors show that clipped \sgd using mini-batch and a median-of-means estimator achieves optimal high-probability rates for convex and strongly convex costs in the presence of heavy-tailed noise with potentially unbounded moments. Our work differs in that we consider non-convex costs and a general nonlinear framework. {Finally, it is worth mentioning that the idea of symmetrization used in our work (see the contributions and main results ahead) has a long history in probability theory, where it is often used to establish some desirable properties, such as concentration around the median, etc. For a more detailed treatment, the reader is referred to \cite{feller1971introduction,gut-probability}.}    

\paragraph{Contributions} In this work, we provide tight finite-time high-probability convergence rates of a general \nsgd-based framework, which extends beyond symmetric noise. To do so, we first revisit the high-probability guarantees of \nsgd in the presence of noise with a PDF symmetric and positive around the origin, originally considered in \cite{armacki2023high}, relaxing the noise assumptions, by facilitating \emph{state-dependent} noise and simultaneously providing \emph{sharper analysis} and \emph{improved rates}. {We compare the guarantees in Table \ref{tab:comp}, showing that \nsgd achieves the oracle complexity $\mathcal{O}(\epsilon^{-4})$, matching the \emph{optimal oracle complexity} of linear \sgd in the presence of light-tailed noise}, strictly better than the $\mathcal{O}(\epsilon^{-8})$ oracle complexity established in \cite{armacki2023high}. Next, to guarantee convergence for noises beyond symmetric ones, we propose two novel stochastic gradient estimators, based on the idea of \emph{noise symmetrization}. The first estimator, dubbed Symmetrized Gradient Estimator (\sge), requires access to a single noiseless gradient at an arbitrary reference point at the start of training and an additional two stochastic gradients in each iteration. {Together with the nonlinear framework, we get the novel \nsge method, which achieves the same oracle complexity $\mathcal{O}(\epsilon^{-4})$, for general noise satisfying a mild technical condition.} Finally, if access to a noiseless gradient is not available, we propose a Mini-batch Symmetrized Gradient Estimator (\msge), which, combined with the nonlinear framework, leads to the novel \nmsge method and overall complexity $\mathcal{O}\big(\epsilon^{-\frac{6p-4}{p-1}}\big)$, in the presence of unbiased noise with bounded $p$-th moment. We can again see in Table \ref{tab:comp} that both \nsgd and \nsge achieve strictly better complexity than \cite{nguyen2023improved,hubler2024normalization,kornilov2025sign}, whenever $p < 2$, i.e., for \emph{any heavy-tailed noise}, while \nmsge achieves an oracle complexity comparable to the said methods. Our results are based on a black-box treatment of the nonlinearity, allowing us to establish unified guarantees for a broad range of popular nonlinearities, such as sign, normalization, clipping and their smooth counterparts.

\paragraph{Technical challenges and novelty} In order to facilitate our results, several challenges needed to be resolved. First, in order to provide improved high-probability convergence rate of \nsgd under symmetric noise, we establish tighter control of the moment-generating function (MGF), by introducing an offset term to cancel out the effect of the noise, similar to the idea used in \cite{liu2023high}. However, due to our black-box treatment of the nonlinearity, this leads to an additional challenge of lower-bounding the difference of two related quantities of different orders, which we achieve via a careful case-based analysis.\footnote{In particular, we lower-bound $\min\{\|\nabla F(\bxk)\|,\|\nabla F(\bxk)\|^2\} - \alpha_k\|\nabla F(\bxk)\|^2$, where $\alpha_k > 0$ is the step-size in iteration $k$, see Section \ref{sec:insight} for details.} Next, to facilitate state-dependent symmetric noise, we establish novel results on the properties of the ``denoised'' nonlinearity (see Subsection \ref{subsec:setup} ahead for the definition) for smooth component-wise nonlinearities in Lemmas \ref{lm:key-unified-state} and \ref{lm:polyak-tsypkin}.\footnote{The ``denoised'' nonlinearity is defined as $\bPhi(\bx) \triangleq \E_{\bz}[\bPsi(\bx + \bz)]$, where $\bPsi: \R^d \mapsto \R^d$ is the original nonlinear mapping, while $\bz \in \R^d$ is the stochastic noise vector, see Section \ref{sec:main} for details.} Finally, inspired by the strong performance of \nsgd in the presence of symmetric noise, we propose two novel stochastic gradient estimators to handle non-symmetric noise, based on the idea of noise symmetrization. To that end, we formally show that the \sge fully symmetrizes the noise and develop a novel analysis to handle the effect of the non-symmetric noise component in the \msge estimator, see Section \ref{sec:main} and Appendix for details. 

\paragraph{Paper organization} The rest of the paper is organized as follows. Section \ref{sec:methods} outlines the proposed methods, Section \ref{sec:main} presents the main results, Section \ref{sec:insight} provides some insights into our improved results and detailed comparison with existing works, with Section \ref{sec:conclusion} concluding the paper. Appendix contains proofs and results omitted from the main body. The remainder of this section introduces the notation used throughout the paper. 

\paragraph{Notation} Positive integers, real numbers and $d$-dimensional vectors are denoted by $\N$, $\R$ and $\R^d$, respectively. We use $\R_+$ to denote non-negative real numbers, i.e., $\R_+ = [0,\infty)$. For $a \in \N$, the set of integers up to and including $a$ is denoted by $[a] = \{1,\ldots,a \}$. Regular and bold symbols denote scalars and vectors, i.e., $x \in \R$ and $\bx \in \R^d$, with $I \in \R^{d \times d}$ denoting the identity matrix. The standard Euclidean inner product, induced vector and matrix norms are denoted by  $\langle \cdot,\cdot\rangle$,  $\|\cdot\|$ and $\tnorm{\cdot}$, respectively, while $\cdot^\top$ denotes the transposition operator. Notation $\mathcal{O}(\cdot)$ stands for the standard ``big O'', i.e., for two non-negative sequences $\{a_t\}_{t\in\N}$, $\{b_t\}_{t\in\N}$, we have $a_t = \mathcal{O}(b_t)$ if there exist $C > 0$ and $t_0 \in \N$, such that $a_t \leq Cb_t$, for all $t \geq t_0$, while $\widetilde{\mathcal{O}}(\cdot)$ indicates existence of factors poly-logarithmic in $t$ or confidence parameter $\delta$.

\section{Proposed Methods}\label{sec:methods}

In this section we outline the methods considered in the paper and provides some intuition behind them. Subsection \ref{subsec:nsgd} outlines the \nsgd family of methods, used when the noise has a symmetric PDF. Subsection \ref{subsec:nsge} introduces the novel estimator \sge, which is used to deal with noise beyond symmetric ones and leads to the \nsge family of methods. Finally, Subsection \ref{subsec:nmsge} introduces a novel mini-batch estimator, dubbed \msge, and the resulting \nmsge family of methods. The remainder of this section outlines a framework for learning with nonlinear methods using a generic (stochastic) gradient estimator, which is then specialized in the following subsections.

To solve the problem \eqref{eq:problem}, we assume access to a Stochastic First-order Oracle (\sfo) model, e.g., \cite{nemirovski1983problem,Arjevani2023}, which, when queried with input $\bx \in \R^d$, returns the gradient of the loss $f$ evaluated at $\bx$ and a random sample $\xi \in \Xi$, i.e., the \sfo outputs $\nabla f(\bx;\xi)$. Given any deterministic initialization $\bx^1 \in \R^d$, let $\bxt \in \R^d$ be our model in iteration $t \in \N$, which, along with potentially other points of interest, we use to query the \sfo. Let $\bgt \in \R^d$ be a generic stochastic gradient estimator constructed using the output of the oracle. We then have the following update rule to produce the next model
\begin{equation}\label{eq:nonlin-generic}
    \bxtp = \bxt - \alpha_t\bPsi(\bgt),
\end{equation} where $\alpha_t > 0$ is the step-size in iteration $t$, while $\bPsi: \R^d \mapsto \R^d$ is a general nonlinear map. The procedure is summarized in Algorithm \ref{alg:nonlin-online}. We next specialize the generic estimator $\bgt$, to get our \nsgd, \nsge and \nmsge methods.

\begin{algorithm}[tb]
\caption{Learning with Nonlinear Methods}
\label{alg:nonlin-online}
\begin{algorithmic}[1]
   \REQUIRE{Choice of nonlinearity $\bPsi: \R^d \mapsto \R^d$, model initialization $\bx^{1} \in \R^{d}$, step-size schedule $\{\alpha_t\}_{t \in \N}$;}
   \FOR{t = 1,2,\ldots}
        \STATE Query the \sfo and construct a gradient estimator $\bgt$;  
        \STATE Update $\bxtp \leftarrow \bxt - \alpha_t\mathbf{\Psi}\left(\bgt\right)$;
    \ENDFOR
\end{algorithmic}
\end{algorithm}

\subsection{\nsgd}\label{subsec:nsgd}

If the noise has a symmetric PDF, we query the \sfo using $\bxt$ and chose $\bgt$ to be the standard \sgd estimator, i.e., $\bgt \triangleq \nabla f(\bxt;\xit)$. Using the nonlinear framework from \eqref{eq:nonlin-generic}, we then get \nsgd, given by
\begin{equation}\label{eq:nsgd}
	\bx^{t+1} = \bxt - \alpha_t\bPsi(\nabla f(\bxt;\xit)).
\end{equation} The \nsgd framework has been well-studied, both for individual nonlinearities, like clipping, sign and normalization, e.g., \cite{zhang2019gradient,bernstein2018signsgd,cutkosky20normalized_SGD}, and as a unified framework, e.g., \cite{jakovetic2023nonlinear,armacki2023high,armacki2024_ldp+mse}.  

\subsection{\nsge}\label{subsec:nsge}

If the noise is not necessarily symmetric and at the start of training we have access to a noiseless gradient of $F$ at any reference point $\by \in \R^d$, i.e., $\nabla F(\by)$, we construct our novel \emph{Symmetrized Gradient Estimator} (\sge) as follows. In each iteration $t \in \N$, we query the \sfo using $\bxt$ and $\by$, to get $\nabla f(\bxt;\xi^t_1)$ and $\nabla f(\by;\xi^t_2)$, where $\xi_1^t,\xi_2^t \in \Xi$ are two independent, identically distributed (IID) samples. The estimator \sge is then given by $\bgt \triangleq \nabla f(\bxt;\xi_1^t) - \nabla f(\by;\xit_2) + \nabla F(\by)$. Combined with the nonlinear framework from \eqref{eq:nonlin-generic}, we get the method \nsge, given by the update rule
\begin{equation}\label{eq:nsge}
    \bxtp = \bxt - \alpha_t\bPsi\big(\nabla f(\bxt;\xi_1^t) - \nabla f(\by;\xit_2) + \nabla F(\by)\big).
\end{equation} Compared to \sgd, which requires access to a single stochastic gradient per iteration, \sge requires access to two stochastic gradients in each iteration. We next provide some intuition behind \sge. Let $\bz^t_1 \triangleq \nabla f(\bxt;\xi^t_1) - \nabla F(\bxt)$ and $\bz^t_2 \triangleq \nabla f(\by;\xi^t_2) - \nabla F(\by)$ denote the stochastic gradient noise. We can then represent \sge as follows
\begin{equation*}
    \bgt =  \nabla F(\bxt) + \bz_1^t - (\nabla F(\by) + \bz_2^t) + \nabla F(\by) = \nabla F(\bxt) + \wbzt,
\end{equation*} where $\wbzt \triangleq \bz_1^t - \bz_2^t$ is the ``symmetrized'' noise. As we show in Section \ref{sec:main}, the resulting noise vectors $\{\wbzt\}_{t \in \N}$ have some desirable properties, under mild conditions on the original noise. 

\subsection{\nmsge}\label{subsec:nmsge}

In case a noiseless gradient is unavailable at the start of training, we can estimate it by using a mini-batch of size $B_t$ in each iteration, to get the \emph{Mini-batch Symmetrized Estimator} (\msge), given by $\bgt \triangleq \nabla f(\bxt;\xi_1^t) - \nabla f(\by;\xi_2^t) + \frac{1}{B_t}\sum_{j = 1}^{B_t}\nabla f(\by;\xi^t_{j+2})$. As such, \msge requires $B_t + 2$ calls to the \sfo per iteration, one using the current model $\bxt$ and $B_t + 1$ calls using an arbitrary point $\by$ as input. Using the nonlinear framework from \eqref{eq:nonlin-generic}, we get the method \nmsge, given by the update rule
\begin{equation}\label{eq:nmsge}
    {\bxtp = \bxt - \alpha_t\bPsi\Big(\nabla f(\bxt;\xi_1^t) - \nabla f(\by;\xi_2^t) + \frac{1}{B_t}\sum_{j = 1}^{B_t}\nabla f(\by;\xi^t_{j+2})\Big)}.
\end{equation} The idea behind \msge is similar to that of \sge, i.e., to ``symmetrize'' the resulting noise, however, \msge has an additional, potentially non-symmetric noise component, stemming from the use of the mini-batch. This can be seen by again defining the stochastic gradient noise as $\bz^t_1 \triangleq \nabla f(\bxt;\xi^t_1) - \nabla F(\bxt)$ and $\bztj \triangleq \nabla f(\by;\xi^t_j) - \nabla F(\by)$, for $j = 2,\ldots,B_t+2$. We can then rewrite the \msge update rule as follows
\begin{equation*}
    \bgt =  \nabla F(\bxt) + \bz_1^t - (\nabla F(\by) + \bz_2^t) + \frac{1}{B_t}\sum_{j = 3}^{B_t+2}(\nabla F(\by) + \bztj) = \nabla F(\bxt) + \wbzt + \obzt,
\end{equation*} where $\wbzt \triangleq \bz_1^t - \bz_2^t$ is the ``symmetrized'', while $\obzt \triangleq \frac{1}{B_t}\sum_{j = 3}^{B_t+2}\bztj$ is the potentially non-symmetric noise component.

\section{Main Results}\label{sec:main}

In this section we present our main results. Subsection \ref{subsec:prelim} provides preliminaries, Subsection \ref{subsec:setup} sets up the analysis, Subsection \ref{subsec:symmetric} presents results for symmetric, while Subsection \ref{subsec:non-symmetric} presents results for potentially non-symmetric noise.

\subsection{Preliminaries}\label{subsec:prelim}

In this section we outline the assumptions used in the paper. We start with the assumption on the population cost.

\begin{assumption}\label{asmpt:L-smooth}
    The population cost $F$ is bounded from below and has $L$-Lipschitz continuous gradients, i.e., $\inf_{\bx \in \R^d}F(\bx) > -\infty$ and $\|\nabla F(\bx) - \nabla F(\by) \| \leq L\|\bx - \by\|$, for all $\bx,\by \in \R^d$.
\end{assumption}

\begin{remark}
    Assumption \ref{asmpt:L-smooth} is standard in smooth non-convex optimization, e.g., \cite{bertsekas-gradient,ghadimi2013stochastic,cevher-almost_sure,nguyen2023improved,armacki2024_ldp+mse}. It can be shown that $L$-Lipschitz continuous gradients imply the $L$-smoothness inequality, namely that
    \begin{equation*}
        F(\by) \leq F(\bx) + \langle \nabla F(\bx), \by - \bx \rangle + \frac{L}{2}\|\bx - \by\|^2,
    \end{equation*} for any $\bx,\by \in \R^d$, which is an important inequality used in our analysis. For a proof of the implication, see, e.g., \cite{bertsekas2003convex,nesterov-lectures_on_cvxopt,lan2020first}.
\end{remark}

 Denote by $F^\star \triangleq \inf_{\bx \in \R^d}F(\bx)$. Recall that stochastic gradient noise in iteration $t$ is denoted by $\bzt = \nabla f(\bxt;\xit) - \nabla F(\bxt)$. Depending on the gradient estimator and choice of nonlinearity, we will make use of one of the following noise assumptions.

{
\begin{assumption}\label{asmpt:noise-state}
    For any $t \in \N$, the noise vector $\bz^t$ depends on the past only through the current state $\bxt$. Moreover, for any $t \in \N$, $\bzt$ depends on $\bxt$ via the PDF, i.e., the noise $\bzt$ has a PDF $P_{\bx}: \R^d \mapsto \R_+$, where $P_{\bx}(\cdot) = P(\cdot \: \vert \: \bxt = \bx)$ is conditioned on the realization $\bx$ of $\bxt$. For any $\bx \in \R^d$, $P_{\bx}$ is symmetric and uniformly positive around the origin, i.e., $P_{\bx}(-\bz) = P_{\bx}(\bz)$, for all $\bz \in \R^d$, and $\inf_{\bx \in \R^d}P_{\bx}(\bz) > 0$, for all $\|\bz\| \leq E_0$ and some $E_0 > 0$.
\end{assumption}
}

\begin{remark}
    Assumption \ref{asmpt:noise-state} allows for state-dependent symmetric noise. As discussed in the introduction, noise symmetry has been widely observed in scenarios such as training deep learning models using large batch sizes, e.g., \cite{bernstein2018signsgd,bernstein2018signsgd_iclr,chen2020understanding,barsbey-heavy_tails_and_compressibility,pmlr-v238-battash24a,armacki2023high}, while works such as \cite{simsekli2019tail,pmlr-v108-peluchetti20b,heavy-tail-phenomena,barsbey-heavy_tails_and_compressibility} theoretically demonstrate that symmetric heavy-tailed noise is an appropriate noise model in many practical settings.
\end{remark}

\begin{remark}
    Uniform positivity around the origin is a mild condition on the behaviour of the PDF, satisfied by many distributions, e.g., Gaussian, Cauchy and Student's t-distribution. In general, uniform positivity can be relaxed, by requiring the PDF to be uniformly positive on the set of every possible realization of the algorithm, i.e., $\inf_{\bx \in \mathcal{X}}P_{\bx}(\bz) > 0$, for all $\|\bz\| \leq E_0$, where $\mathcal{X} = \{\bx \in \R^d: \exists \: (t ,\omega) \in \N \times \Omega \text{ such that } \bxt(\omega) = \bx\}$.
\end{remark}

\begin{remark}
    As can be seen in the following subsections, the radius $E_0 > 0$ of positivity of the PDF does not directly impact the convergence bounds, but is rather a technical requirement to ensure the PDF is sufficiently well-behaved. As such, $E_0$ can be arbitrarily small, without affecting the convergence directly.
\end{remark}

To handle noise beyond symmetric ones, we use two assumption. The first one, stated next, is used when a noiseless gradient is available at the start of training and we can construct the estimator \sge.

\begin{assumption}\label{asmpt:non-sym}
    The noise vectors $\{\bz^t_1 \}_{t \in \N}$ and $\{\bz^t_2 \}_{t \in \N}$ are IID, with PDF $P: \R^d \mapsto \R_+$, such that, for some $E_0 > 0$ and all $\|\bx\| \leq E_0$, the set $S_{\bx} \triangleq \{\by \in \R^d: P(\by) > 0 \text{ and } P(\by - \bx) > 0\}$ is of positive Lebesgue measure.
\end{assumption}

\begin{remark}
    Assumption \ref{asmpt:non-sym} relaxes symmetry, at the cost of IID noise. The IID condition is often used when analyzing stochastic algorithms e.g., \cite{raginsky2017non,liu2020improved,armacki2024_ldp+mse,armacki2023high} and is required for our symmetrization-based analysis. The identically distributed requirement can be relaxed to allow the noise in iteration $t$ to have a PDF $P_t$, with each $\{P_t\}_{t\in\N}$ satisfying Assumption \ref{asmpt:non-sym}. 
\end{remark}
    
\begin{remark}
    Positivity with respect to the Lebesgue measure of $S_{\bx}$ is a mild condition and can be ensured, e.g., if the PDF $P$ is strictly positive in a slightly larger neighbourhood around the origin, see Appendix \ref{app:on_noise} for a formal result. In general, it is easy to see that this condition is satisfied by many noise distributions encountered in practice, such as exponential, gamma, Cauchy, Weibull, Pareto and Gumbel, as it does not require the PDF to be positive for all $\bx$ sufficiently close to the origin, but rather for the set $S_{\bx}$ to take up a non-zero Lebesgue measure. 
\end{remark}

Assumptions \ref{asmpt:noise-state} and \ref{asmpt:non-sym} make no moment requirements, allowing for noise with very heavy tails, such as Cauchy or Student's t-distribution. Finally, when a noiseless gradient is unavailable, we use the estimator \msge. Recalling the discussion in Section \ref{sec:methods}, a mini-batch of size $B_t + 2$ is required to construct \msge in iteration $t$. Moreover, in addition to a symmetrized noise component, \msge has a non-symmetric noise component. To deal with these obstacles, we impose the following assumption.

\begin{assumption}\label{asmpt:non-sym-p-moment}
    In each iteration $t \in \N$, the noise vectors $\{\bztj \}_{j = 1}^{B_t+2}$ are IID, with PDF $P: \R^d \mapsto \R$, such that, for some $E_0 > 0$ and all $\|\bx\| \leq E_0$, the set $S_{\bx} = \{\by \in \R^d: P(\by) > 0 \text{ and } P(\by - \bx) > 0\}$ is of positive Lebesgue measure. Moreover, the noise is unbiased, with bounded moment of order $p \in (1,2]$, i.e., $\E[\bztj] = \mathbf{0}$ and $\E[\|\bztj\|^p] \leq \sigma^p$, for some $\sigma > 0$.
\end{assumption}

\begin{remark}
    In addition to conditions from Assumption \ref{asmpt:non-sym}, Assumption \ref{asmpt:non-sym-p-moment} requires the noise to have bounded $p$-th moment, for some $p \in (1,2]$. Bounded $p$-th moment is a widely used condition when analyzing heavy-tailed noise, e.g., \cite{sadiev2023highprobability,nguyen2023improved,hubler2024normalization,kornilov2025sign}. 
\end{remark}

{
\begin{remark}
    We note that all our noise assumptions require the existence of a PDF, ruling out discrete noise distributions and scenarios like finite-sum costs where the noise stems from choosing a sub-sample of data points. However, we consider a general setting of stochastic costs (including finite-sum structure as a special case) and a \sfo which returns noisy versions of the true gradient, facilitating general, continuous noises like Gaussian or Cauchy noise.  
\end{remark}
}

Note that Assumptions \ref{asmpt:non-sym} and \ref{asmpt:non-sym-p-moment} do not rule out symmetric noise and as such, both \nsge and \nmsge can be used to handle non-symmetric, as well as symmetric noise. To handle the heavy-tailed noise while facilitating a unified black-box analysis, we consider a general nonlinear framework and make use of one of the following assumptions, for different noise regimes and estimators. The first assumption, stated next, was previously used in \cite{jakovetic2023nonlinear,armacki2024_ldp+mse,armacki2023high}.

\begin{assumption}\label{asmpt:nonlin}
The nonlinear map $\bPsi: \mbb R^d \mapsto \mbb R^d$ is either component-wise, i.e., of the form $\bPsi(\bx) = \lbr \psi(x_1), \dots, \psi(x_d) \rbr^\top$, or joint, i.e., of the form $\bPsi(\bx) = \bx\varphi(\|\bx\|)$, where the mappings $\psi,\:\varphi: \R \mapsto \R$ satisfy
\begin{enumerate}
    \item $\psi,\varphi$ are continuous, except for at most finitely many points, with $\psi$ piece-wise differentiable and $|\psi^\prime(x)| \leq c_1$, for some $c_1 > 0$, while $a \mapsto a\varphi(a)$ is non-decreasing on $a \in (0,\infty)$.
    \item $\psi$ is monotonically non-decreasing and odd, while $\varphi$ is non-increasing on $(0,\infty)$.
    \item $\psi$ is either discontinuous at zero, or strictly increasing on $(-c_2,c_2)$, for some $c_2 > 0$, with $\varphi(a) > 0$, for any $a > 0$.
    \item $\psi$ and $\bx\varphi(\|\bx\|)$ are uniformly bounded, i.e., for some $C_1,\: C_2 > 0$, and all $x \in \R$, $\bx \in \R^d$, we have $|\psi(x)| \leq C_1$ and $\|\bx\varphi(\|\bx\|)\|\leq C_2$.
\end{enumerate}
\end{assumption}

Assumption \ref{asmpt:nonlin} is very general, subsuming many popular nonlinearities, like sign, normalization, component-wise and joint clipping, see Appendix \ref{app:on_nonlin} for a formal result. The next assumption imposes further structure on component-wise nonlinearities.

\begin{assumption}\label{asmpt:nonlin-state}
    If $\bPsi: \mbb R^d \mapsto \mbb R^d$ is joint, i.e., of the form $\bPsi(\bx) = \bx\varphi(\|\bx\|)$, then the conditions of Assumption \ref{asmpt:nonlin} hold. If $\bPsi$ is component-wise, i.e., of the form $\bPsi(\bx) = \lbr \psi(x_1), \dots, \psi(x_d) \rbr^\top$, then in addition to conditions from Assumption \ref{asmpt:nonlin}, $\psi$ is twice continuously differentiable, with uniformly bounded first and second derivatives, i.e., $|\psi^\prime(x)|\leq K_1$ and $|\psi^{\prime\prime}(x)|\leq K_2$, for some $K_1,K_2 > 0$.
\end{assumption}

Assumption \ref{asmpt:nonlin-state} requires component-wise nonlinearities to be smooth and is used to facilitate state-dependent noise. In addition to joint nonlinearities like normalization and joint clipping, Assumption \ref{asmpt:nonlin-state} is satisfied by component-wise nonlinearities like smooth sign and smooth component-wise clipping, see Appendix \ref{app:on_nonlin} for a formal result. Finally, to handle the non-symmetric noise component arising in the estimator \msge, we use the following assumption.   

\begin{assumption}\label{asmpt:nonlin-nmsge}
    The nonlinear map $\bPsi: \mbb R^d \mapsto \mbb R^d$ is either component-wise, i.e., of the form $\bPsi(\bx) = \lbr \psi(x_1), \dots, \psi(x_d) \rbr^\top$, or joint, i.e., of the form $\bPsi(\bx) = \bx\varphi(\|\bx\|)$, where the mappings $\psi,\:\varphi: \R \mapsto \R$ satisfy
    \begin{enumerate}
        \item $\psi, \:\varphi$ are respectively twice and once continuously differentiable and $a \mapsto a\varphi(a)$ is non-decreasing on $(0,\infty)$.
        \item $\psi$ is monotonically non-decreasing and odd, while $\varphi$ is non-increasing and strictly positive on $(0,\infty)$.
        \item $\psi$ and $\bx\varphi(\|\bx\|)$ are uniformly bounded, with uniformly bounded first derivatives and $\psi$ also has a uniformly bounded second derivative, i.e., for some $C_i, K_j > 0$, $i \in [2]$, $j \in [3]$ and all $x \in \R$, $\bx \in \R^d$, we have $|\psi(x)| \leq C_1$, $|\psi^\prime(x)| \leq K_1$, $|\psi^{\prime\prime}(x)| \leq K_2$, $\|\bx\varphi(\|\bx\|)\|\leq C_2$ and $\tnorm{\varphi^\prime(\|\bx\|)\frac{\bx\bx^\top}{\|\bx\|} + \varphi(\|\bx\|)I}\leq K_3$.
    \end{enumerate}
\end{assumption}

Assumption \ref{asmpt:nonlin-nmsge} imposes smoothness on both component-wise and joint nonlinearities and is satisfied by smooth sign, smooth normalization, as well as smooth component-wise and joint clipping, see Appendix \ref{app:on_nonlin} for a formal result. Note that Assumption \ref{asmpt:nonlin} is the most general one, as it subsumes Assumption \ref{asmpt:nonlin-state}, which itself subsumes Assumption \ref{asmpt:nonlin-nmsge}. In spite of this, Assumption \ref{asmpt:nonlin-nmsge} is satisfied by a wide range of nonlinearities, as essentially any bounded nonlinearity can be approximated by a smooth counterpart. In Figure \ref{fig:smoothed-nonlin} we provide visualization of some popular nonlinearities satisfying Assumption \ref{asmpt:nonlin} and their smooth counterparts satisfying Assumption \ref{asmpt:nonlin-nmsge}. In particular, we visualize the following nonlinearities
\begin{enumerate}
    \item \emph{Sign}, {e.g., \cite{bernstein2018signsgd}}: $\psi(x) = \text{sign}(x)$,
    \item \emph{Component-wise clipping}, {e.g., \cite{zhang2020adaptive}}: $\psi(x) = \begin{cases}
        M\text{sign}(x), & |x| > M \\
        x, & |x| \leq M
    \end{cases}$, with $M = 3.5$,
    \item \emph{Normalization}, {e.g., \cite{cutkosky20normalized_SGD}}: $\bPsi(\bx) = \begin{cases}
        \frac{\bx}{\|\bx\|}, & \bx \neq \mathbf{0} \\
        \mathbf{0}, & \bx = \mathbf{0}
    \end{cases}$,
\end{enumerate} and their following smooth counterparts
\begin{enumerate}
    \item \emph{Smooth sign}: $\psi(x) = \tanh(x/k)$, with $k = 0.1$,\footnote{While we are not aware of $\tanh$ being used in the context of nonlinear \sgd methods, it is a widely used activation function in the context of neural networks, e.g., \cite{bishop2006}.}
    \item \emph{Smooth component-wise clipping}: $\psi(x) = \begin{cases}
        M\text{sign}(x), & |x| > M \\
        \frac{5}{8}\left(3x - \frac{2x^3}{M^2} + \frac{3x^5}{5M^4}\right), & |x| \leq M
    \end{cases}$, with $M = 3.5$,
    \item \emph{Smooth normalization}, {e.g., \cite{yu2023smoothed}}: $\bPsi(\bx) = \frac{\bx}{\sqrt{\|\bx\|^2 + \epsilon}}$, with $\epsilon = 0.1$.
\end{enumerate} We can see in Figure \ref{fig:smoothed-nonlin} that the smooth nonlinearities are good approximations of their non-smooth counterparts, preserving crucial properties, like oddity and boundedness, while ensuring that the resulting mapping is smooth. 

\begin{figure*}[!ht]
\centering
\begin{tabular}{lll}
\includegraphics[width=0.4\linewidth]{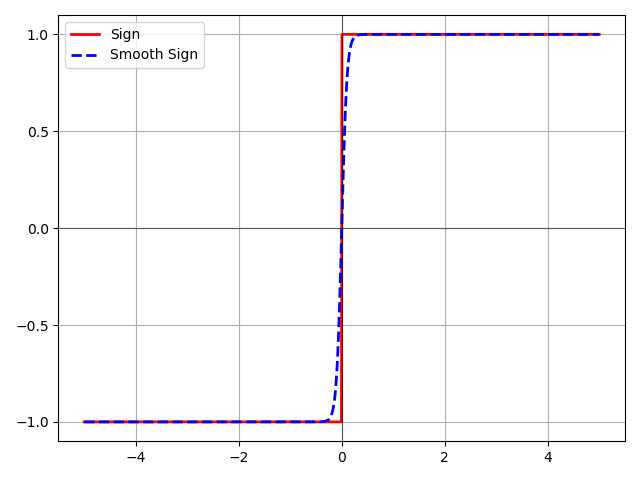}
&
\includegraphics[width=0.4\linewidth]{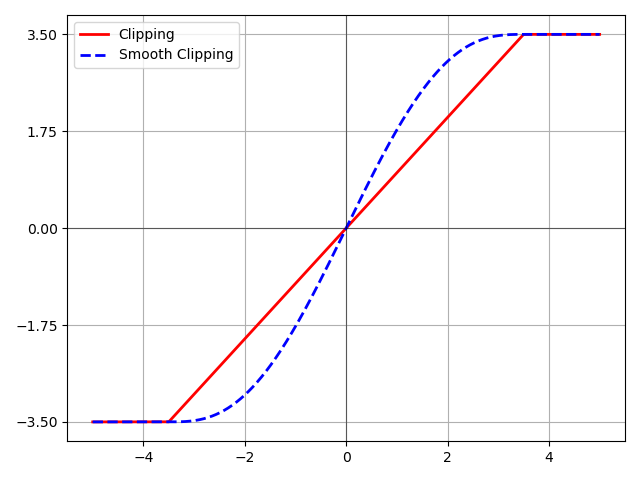}
\\
\includegraphics[width=0.4\linewidth]{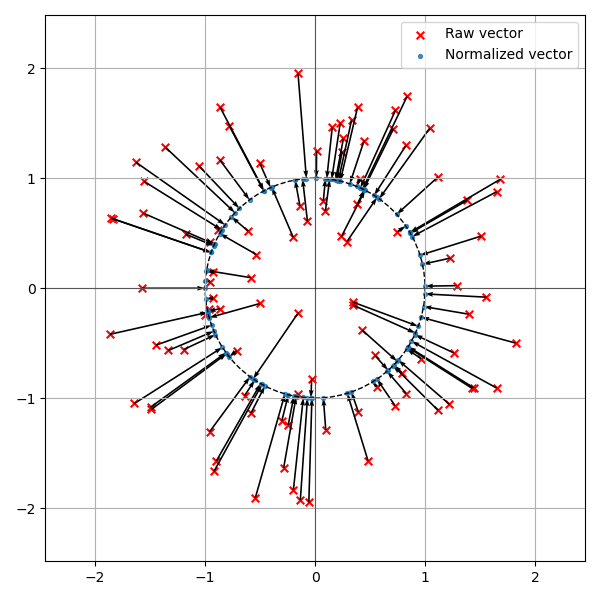}
&
\includegraphics[width=0.4\linewidth]{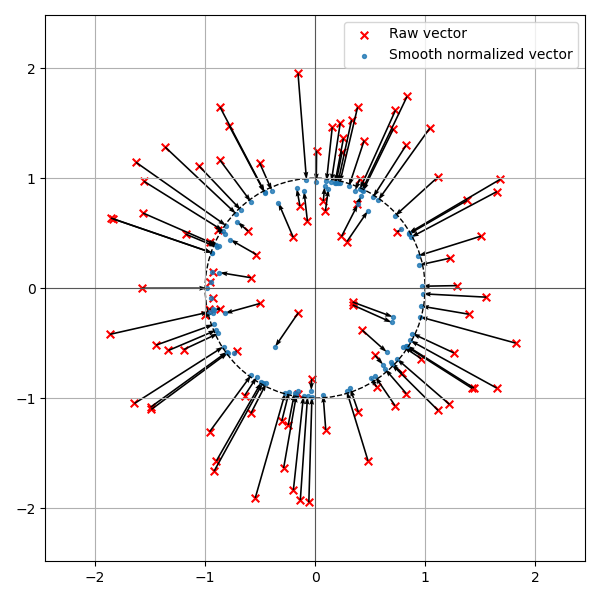}
\end{tabular}
\caption{Non-smooth component-wise nonlinearities and their smoothed counterparts. Top row: sign, clipping and their smooth counterparts. Bottom row: vectors sampled from a ball of radius 2 in $\R^2$ and their normalized and smooth normalized versions.}
\label{fig:smoothed-nonlin}
\end{figure*}

\subsection{Analysis Setup}\label{subsec:setup}

Recall that $F^\star = \inf_{\bx \in \R^d}F(\bx) > -\infty$ and {$\bx^1 \in \R^d$ is a deterministically selected model initialization}, and denote by $\Delta \triangleq F(\bx^1) - F^\star$ and $ \widetilde{\nabla} \triangleq \|\nabla F(\bx^1)\|$ the optimality and stationarity gap of the initial model, respectively. Next, note that Assumptions \ref{asmpt:nonlin}-\ref{asmpt:nonlin-nmsge} all imply a uniform bound on the nonlinearity and let $C > 0$ denote the bound, i.e., $\|\bPsi(\bx)\| \leq C$, for all $\bx \in \R^d$. It can be readily seen that $C = C_1\sqrt{d}$ for component-wise and $C = C_2$ for joint nonlinearities. Similarly, we will use $K > 0$ to denote the uniform bound on the first derivative of the nonlinearity when Assumption \ref{asmpt:nonlin-nmsge} is in place, with $K = K_1$ for component-wise and $K = K_3$ for joint nonlinearities.

In order to facilitate the analysis, we define an important concept, referred to as the ``denoised'' nonlinearity, i.e., the mapping $\bPhi: \R^d \mapsto \R^d$, given by $\bPhi(\by) \triangleq \E_{\bz}[\bPsi(\by + \bz)]$, where $\bz \in \R^d$ is the stochastic noise. Consider again the update rule \eqref{eq:nonlin-generic}, using a generic estimator $\bgt$ and define the stochastic noise as $\bzt \triangleq \bgt - \nabla F(\bxt)$. We can then rewrite the update \eqref{eq:nonlin-generic} as
\begin{equation*}
    \bxtp = \bxt - \alpha_t\bPsi(\bgt) = \bxt - \alpha_t\bPhi(\nabla F(\bxt)) - \alpha_t\bet,
\end{equation*} where $\bet \triangleq \bPsi(\bgt) - \bPhi(\nabla F(\bxt))$ is the ``effective'' noise. By the definition of $\bPhi$ and the fact that the mapping $\bPsi$ is uniformly bounded, we can expect that the effective noise $\bet$ is well-behaved (to be specified precisely in Appendix \ref{app:proofs}), regardless of the original noise. The work \cite{armacki2023high} shows that, for the \nsgd method (i.e., using $\bgt = \nabla f(\bxt;\xi^t)$), if the noise is IID with PDF symmetric and positive around the origin, then the denoised nonlinearity $\bPhi$ satisfies a monotonicity property with respect to the original input. The formal result is stated next.

\begin{lemma}[Lemma 3.2 in \cite{armacki2023high}]\label{lm:key-iid}
    Let Assumptions {\ref{asmpt:noise-state} and \ref{asmpt:nonlin} hold, with the stochastic noise vectors $\bzt = \nabla f(\bxt;\xi^t) - \nabla F(\bxt)$ being mutually IID and independent from the state for all $t \in \N$, i.e., $P_{\bx} \equiv P$, then the denoised nonlinearity satisfies the following inequality, for any $\bx \in \R^d$
    \begin{equation*}
        \langle \bPhi(\bx), \bx\rangle \geq \min\{\gamma_1\|\bx\|,\gamma_2\|\bx\|^2\},
    \end{equation*}} where $\gamma_1,\gamma_2 > 0$ are constants that depend on the noise, choice of nonlinearity and other problem related parameters.
\end{lemma}

Lemma \ref{lm:key-iid} provides an important property of the denoised nonlinearity, which then facilitates the rest of the analysis in \cite{armacki2023high,armacki2024_ldp+mse}. However, due to different noise assumptions and estimators used in our work, Lemma \ref{lm:key-iid} does not directly apply. In what follows, one of the main challenges is to derive counterparts of Lemma \ref{lm:key-iid} for different estimators and noise assumptions. To close out this section, we provide an example of constants $\gamma_1,\gamma_2$ for a specific noise and nonlinearity.

\begin{example}\label{ex:example-1}
    Consider the noise with PDF $P(\bz) = \prod_{i =1}^d\rho(z_i)$, where $\rho(z) = \frac{\alpha-1}{2(|z|+1)^{\alpha}}$, for some $\alpha > 2$, {which is an instance of noise with power-law tail decay (see Appendix \ref{app:power-law})} and only has finite moments of order strictly less than $\alpha - 1$. Consider the sign nonlinearity, i.e., $\bPsi(\bx) = \begin{bmatrix}\psi(x_1),\ldots,\psi(x_d) \end{bmatrix}^\top$, where $\psi(x) = \text{sign}(x)$. It can then be shown that $\gamma_1 = \frac{\alpha-1}{2\alpha\sqrt{d}}$ and $\gamma_2 = \frac{\alpha-1}{2d}$, see Appendix \ref{app:example} for details. 
\end{example}

\subsection{Guarantees for Symmetric Noise}\label{subsec:symmetric}

The initial insights of our work stem from analyzing the behaviour of \nsgd in the presence of symmetric noise satisfying Assumption \ref{asmpt:noise-state}. For any state $\bx \in \R^d$, we define the mapping $\bPhi_{\bx}: \R^d \mapsto \R^d$, given by $\bPhi_{\bx}(\by) \triangleq \E_{\bx} [\bPsi (\by + \bz)] = \int \bPsi(\by+\bz) P_{\bx}(\bz) d\bz$,\footnote{If $\bPsi$ is a component-wise nonlinearity, then $\bPhi_{\bx}$ is a vector with components $\phi_{i,\bx}(x_i) = \E_{i,\bx}[\psi(x_i + z_i)]$, where $\E_{i,\bx}$ is the marginal expectation with respect to the $i$-th noise component, $i \in [d]$.} where the expectation is taken with respect to the random vector $\bz$ which may depend on state $\bx$ and $\E_{\bx}[\cdot] \triangleq \E[\cdot\:\vert\: \bx]$. Note that unlike in \cite{armacki2023high}, where the mapping $\bPhi$ is independent of the state, now the mapping $\bPhi_{\bx}$ is itself a function of the current state. As such, using Lemma \ref{lm:key-iid} directly may result in $\gamma_1,\gamma_2$, being functions of the state and there is no way to guarantee that $\inf_{\bx \in \R^d}\gamma_1(\bx),\inf_{\bx \in \R^d}\gamma_2(\bx) > 0$. To that end, we utilize the additional smoothness of component-wise nonlinearities from Assumption \ref{asmpt:nonlin-state} and provide a novel result characterizing the properties of the component-wise denoised  nonlinearity, see Appendix \ref{app:lemma-key} for details. Armed with that result, we are able to establish the analog of Lemma \ref{lm:key-iid}, stated next.

\begin{lemma}\label{lm:key-unified-state}
    Let Assumptions \ref{asmpt:noise-state} and \ref{asmpt:nonlin-state} hold. Then, for any $\bx,\by \in \R^d$, we have $\langle \bPhi_{\bx}(\by),\by\rangle \geq \min\left\{\beta_1\|\by\|,\beta_2\|\by\|^2 \right\}$, where $\beta_1,\beta_2 > 0$, are constants that only depend on noise, choice of nonlinearity and other problem related parameters.
\end{lemma}

Note that constants $\beta_1,\beta_2$ are independent of the state $\bx$ and uniformly positive and let $\beta \triangleq \min\{\beta_1,\beta_2\}$. Next, we provide an example of constants $\beta_1,\beta_2$ for a specific noise and nonlinearity.

\begin{example}\label{ex:example-2}
    Consider the same noise from Example \ref{ex:example-1} and the smooth sign nonlinearity, given by $\bPsi(\bx) = \begin{bmatrix}\psi(x_1),\ldots,\psi(x_d) \end{bmatrix}^\top$, where $\psi(x) = \tanh(x)$. It can then be shown that $\beta_1 = \frac{3\sqrt{3}(\alpha-1)^2J_\alpha^2}{16\sqrt{d}}$ and $\beta_2 = \frac{(\alpha-1)J_\alpha}{2d}$, where $J_\alpha = \int_0^{\infty}\frac{4e^{2x}}{(e^{2x}+1)^2(x+1)^\alpha}dx$, see Appendix \ref{app:example} for details.
\end{example}

We are now ready to present our main result for symmetric noise.

\begin{theorem}\label{thm:n-sgd}
    Let Assumptions \ref{asmpt:L-smooth}, \ref{asmpt:noise-state} and \ref{asmpt:nonlin-state} hold and let $\{ \bxt\}_{t \in \N}$ be the sequence of models generated by \nsgd, with step-size schedule $\alpha_t = \frac{a}{(t+1)^{\eta}}$, for any $\eta \in [1/2,1)$. If the step-size parameter satisfies $a \leq \min\left\{ 1,\frac{\beta_2}{8C^2}, \frac{(1-\eta)\beta_1}{8C^2(LC+(1-\eta)\widetilde{\nabla})}\right\}$, then for any $t \in \N$ and any $\delta \in (0,1)$, with probability at least $1 - \delta$, we have 
    \begin{equation*}
        \frac{1}{t}\sum_{k = 1}^t\min\{\|\nabla F(\bxk)\|,\|\nabla F(\bxk)\|^2\} \leq \frac{4(\log(\nicefrac{1}{\delta}) + \Delta)}{a\beta t^{1-\eta}} + \frac{2aLC^2}{\beta t^{1-\eta}}\sum_{k = 1}^t(k+1)^{-2\eta}.
    \end{equation*}
\end{theorem}

Theorem \ref{thm:n-sgd} establishes a high-probability convergence rate $\mathcal{O}\left(t^{\eta-1}\big(\log(\nicefrac{1}{\delta}) + \sum_{k = 1}^t(k+1)^{-2\eta}\big)\right)$, for any choice of $\eta \in [1/2,1)$. Note that the sum in the second term can be upper bounded by a constant whenever $\eta > 1/2$, while only growing as $\log(t)$ for $\eta = 1/2$, which is a product of the use of a time-varying step-size and can be removed if a fixed step-size is used. Choosing $\eta = 1/2$ leads to the rate $\widetilde{\mathcal{O}}\left(t^{-1/2}\right)$, which is \emph{strictly better} than the $\widetilde{\mathcal{O}}\left(t^{-1/4}\right)$ rate obtained by \cite{armacki2023high}, for the same \nsgd method under stronger noise assumptions.\footnote{The work \cite{armacki2023high} requires IID noise, with PDF symmetric and positive around the origin, while using the more general Assumption \ref{asmpt:nonlin} on the nonlinearity. We note that the improved results of Theorem \ref{thm:n-sgd} continue to hold in the setting from \cite{armacki2023high}, i.e., Assumption \ref{asmpt:nonlin} and IID noise.} 

{
Next, observe that the bound in Theorem \ref{thm:n-sgd} contains the constant $\beta^{-1}$, which is instance-dependent, in the sense that it depends on the specific noise, choice of nonlinearity and (potentially) other problem related parameters. This can be used to compare the suitability of a nonlinearity for a specific problem setting (i.e., noise and cost) and determine the preferred nonlinearity (i.e., the one leading to a larger value of $\beta$) in the given problem setting. For example, for the noise in Example \ref{ex:example-1} and component-wise and joint clipping, with thresholds $m < 1$ and $M$, respectively, it can be shown that $\beta_{cc} = \frac{1-(m+1)^{-\alpha}}{2d}$ and $\beta_{jc} = \left[\frac{\alpha-1}{2}\right]^d\min\{1/2,M/2\}$, hence it follows that $\beta_{jc} \ll \beta_{cc}$ for $\alpha \in (2,3)$ (i.e., heavy-tailed) and $d$ sufficiently, indicating that component-clipping is preferred to joint clipping in this regime. For details, please see Appendix \ref{app:bounds}. 
}

{Further, the step-size schedule in Theorem \ref{thm:n-sgd} requires knowledge of the gradient norm of $F$ evaluated at the initial model $\bx^1$, in the form of $\widetilde{\nabla}$. This can be relaxed to requiring \emph{any upper bound} on $\widetilde{\nabla}$, i.e., any $M > 0$ such that $\widetilde{\nabla} \leq M$. One such upper bound is $\widetilde{\nabla} \leq \sqrt{2L\Delta}$, which is a direct consequence of smoothness, e.g., \cite{nesterov-lectures_on_cvxopt}. Compared to $\widetilde{\nabla}$, which requires knowledge of the full gradient of $F$, the upper bound $2L\Delta$ requires knowledge of: (i) the smoothness constant $L$, which is either known or easy to estimate, e.g., via backtracking \cite[Section 6]{nesterov-coordinate} (ii) a lower bound on $F$, which is also typically known, or easy to estimate, e.g., just by knowing the closed-form of the loss\footnote{For example, knowing that the loss $\ell$ is non-negative readily implies that we can take $F^\star = 0$.} and (iii) the value of $F$ at the model initialization, which can be obtained via an inexpensive operation, e.g., a forward pass in a neural network or querying a zeroth-order oracle.}

Finally, Theorem \ref{thm:n-sgd} can be used to derive a bound on the more standard metric $\min_{k \in [t]}\|\nabla F(\bxk)\|^2$, presented next. 

\begin{corollary}\label{cor:standard-metric}
    Let conditions of Theorem \ref{thm:n-sgd} hold. Then, for any $\delta \in (0,1)$, {there exists a problem related constant $R_{\delta} > 0$, such that, for any $t \geq R_{\delta}^{\frac{1}{1-\eta}}$}, with probability at least $1 - \delta$, we have 
    \begin{equation*}
        \min_{k \in [t]}\|\nabla F(\bxk)\|^2 = \mathcal{O}\left(t^{\eta-1}\Big(\log(\nicefrac{1}{\delta}) + \sum_{k = 1}^t(k+1)^{-2\eta}\Big) \right).
    \end{equation*}
\end{corollary}

By choosing $\eta = 1/2$, we immediately have the following result.

\begin{corollary}\label{cor:optimal}
    Let conditions of Theorem \ref{thm:n-sgd} hold, with $\eta = 1/2$. Then, for any $\delta \in (0,1)$, {there exists a problem related constant $R_\delta > 0$, such that, for any $t \geq R_\delta^2$}, with probability at least $1 - \delta$, we have 
    \begin{equation*}
        \min_{k \in [t]}\|\nabla F(\bxk)\|^2 = \mathcal{O}\left(\frac{\log(\nicefrac{t}{\delta})}{\sqrt{t}} \right).
    \end{equation*}
\end{corollary}

Ignoring the logarithmic factor, Corollary \ref{cor:optimal} implies that \nsgd requires $\widetilde{\mathcal{O}}(\epsilon^{-4})$ oracle calls to reach an $\epsilon$-stationary point, i.e., a point such that $\|\nabla F(\bx)\|\leq \epsilon$, which is much sharper than the $\mathcal{O}(\epsilon^{-8})$ oracle calls implied by \cite{armacki2023high}. {Moreover, this result matches the in-expectation lower bound for first-order stochastic methods established in \cite{Arjevani2023}, for noise with bounded variance, which is a special case of the more general noise with unbounded moments used in our work.}\footnote{While we consider noise with more general moment conditions, we impose additional structure, like the existence of a symmetric PDF, positive in a neighbourhood of zero. As such, it would be of interest to establish a lower bound under the exact noise conditions considered in our work, which is beyond the scope of the current paper.}

\subsection{Guarantees Beyond Symmetric Noise}\label{subsec:non-symmetric}

If the noise is not necessarily symmetric and we have access to a noiseless gradient at the start of training, we can use the \nsge method. Recalling that \sge is given by $\bgt = \nabla F(\bxt) + \wbzt$, where $\wbzt = \bz^t_1 - \bz^t_2$, we then have the following important result on the noise $\{\wbzt\}_{t \in \N}$.

\begin{lemma}\label{lm:sge-symmetrize}
    If Assumption \ref{asmpt:non-sym} holds, then the noise $\{\wbzt\}_{t \in \N}$ is IID, with PDF $\wP(\bz) \triangleq \int_{\by \in \R^d}P(\by)P(\by - \bz)d\by$. Moreover, the PDF $\wP$ is symmetric and positive around the origin, i.e., {$\wP(-\bz) = \wP(\bz)$}, for all $\bz \in \R^d$, and $\wP(\bz) > 0$, for all $\|\bz\| \leq E_0$.
\end{lemma}

Lemma \ref{lm:sge-symmetrize} states that estimator \sge indeed symmetrizes the noise, even if the original noise is not symmetric. We can then define $\bPhi(\nabla F(\bxt)) = \E_{\wbzt}[\bPsi(\nabla F(\bxt) + \wbzt)]$ and use Lemma \ref{lm:sge-symmetrize}, to conclude that Lemma \ref{lm:key-iid} holds when using estimator \sge under Assumptions \ref{asmpt:non-sym} and \ref{asmpt:nonlin}. As such, we then have the following result.

\begin{theorem}\label{thm:n-sge}
    Let Assumptions \ref{asmpt:L-smooth}, \ref{asmpt:non-sym} and \ref{asmpt:nonlin} hold and let $\{ \bxt\}_{t \in \N}$ be the sequence of models generated by \nsge, with step-size schedule $\alpha_t = \frac{a}{(t+1)^{\eta}}$, for any $\eta \in [1/2,1)$. If the step-size parameter satisfies $a \leq \min\left\{1, \frac{\gamma_2}{8C^2}, \frac{(1-\eta)\gamma_1}{8C^2(LC + (1-\eta)\widetilde{\nabla})}\right\}$, then for any $t \in \N$ and any $\delta \in (0,1)$, with probability at least $1 - \delta$, we have 
    \begin{equation*}
        \frac{1}{t}\sum_{k = 1}^t\min\{\|\nabla F(\bxk)\|,\|\nabla F(\bxk)\|^2\} = \mathcal{O}\left(t^{\eta-1}\Big(\log(\nicefrac{1}{\delta}) + \sum_{k = 1}^t(k+1)^{-2\eta}\Big) \right).
    \end{equation*}
\end{theorem}

The full dependence on problem constants, omitted for brevity, is the same as in Theorem \ref{thm:n-sgd}, with $\beta$ replaced by $\gamma \triangleq \min\{\gamma_1,\gamma_2\}$. Results of Theorem \ref{thm:n-sge} are strong - if a noiseless gradient at an arbitrary reference point is available at the start of training, then \nsge matches the rate $\widetilde{\mathcal{O}}(t^{-1/2})$ of \nsgd, for any heavy-tailed noise satisfying Assumption \ref{asmpt:non-sym}. Theorem \ref{thm:n-sge} can be used to derive counterparts of Corollaries \ref{cor:standard-metric}, \ref{cor:optimal}, implying the same $\mathcal{O}(\epsilon^{-4})$ complexity for reaching an $\epsilon$-stationary point, extended to noises beyond symmetric ones. {Finally, we note that, while it is sometimes possible to obtain the full gradient of $F$ at a single point (e.g., if the cost $F$ has finite-sum structure), the primary significance of the \nsge method is of theoretical nature, in the sense that \sge serves to introduce the idea of symmetrization and underline the strong benefits this approach brings.}

On the other hand, if we do not have access to a noiseless gradient at the start of training, we can use the \nmsge method. Recall that the estimator \msge is given by $\bgt = \nabla F(\bxt) + \wbzt + \obzt$, where $\wbzt = \bz^t_1 - \bz^t_2$ is the symmetrized, while $\obzt = \frac{1}{B_t}\sum_{j = 3}^{B_t+2}\bztj$ is the potentially non-symmetric noise component. While Lemma \ref{lm:sge-symmetrize} holds true for $\wbzt$, we need to handle the potentially non-symmetric noise component $\obzt$. To that end, we define $\bPhi(\nabla F(\bxt)) \triangleq E_{\wbzt}[\bPsi(\nabla F(\bxt) + \wbzt)]$ and $\obPhi(\nabla F(\bxt)) \triangleq E_{\wbzt,\obzt}[\bPsi(\nabla F(\bxt) + \wbzt + \obzt)]$. Here, $\bPhi$ is an ``idealized'' denoised nonlinearity, assuming that the noise is fully symmetric, while $\obPhi$ is the denoised nonlinearity with respect to the entire noise. Using these two quantities, we can rewrite the \nmsge update rule \eqref{eq:nmsge} as follows 
\begin{equation*}
    \bxtp = \bxt - \alpha_t\bPhi(\nabla F(\bxt)) - \alpha_t\bet - \alpha_t\brt,
\end{equation*} where $\bet \triangleq \bPsi(\nabla F(\bxt) + \wbzt + \obzt) - \obPhi(\nabla F(\bxt))$ is the effective noise, while $\brt \triangleq \obPhi(\nabla F(\bxt)) - \bPhi(\nabla F(\bxt))$ represents the difference between the ideal and true denoised nonlinearity. While we can use Lemma \ref{lm:key-iid} and again expect the effective noise $\bet$ to be well-behaved, we additionally need to control the quantity $\brt$. To that end, we have the following result.

\begin{lemma}\label{lm:gap-control}
    Let Assumptions \ref{asmpt:non-sym-p-moment} and \ref{asmpt:nonlin-nmsge} hold. Then, for any $t \in \N$, we have $\|\brt\| \leq 2\sqrt{2d}\sigma K B_t^{\frac{1-p}{p}}$, almost surely.
\end{lemma}

Lemma \ref{lm:gap-control} shows that we can control the gap between the ideal and true denoised nonlinearity in terms of the mini-batch size. Using this result, we then have the following guarantee on the convergence of \nmsge. 

\begin{theorem}\label{thm:n-msge}
    Let Assumptions \ref{asmpt:L-smooth}, \ref{asmpt:non-sym-p-moment} and \ref{asmpt:nonlin-nmsge} hold and let $\{ \bxt\}_{t \in \N}$ be the sequence of models generated by \nmsge, with step-size schedule $\alpha_t = \frac{a}{(t+1)^{\eta}}$, for any $\eta \in [1/2,1)$ and batch size $B_t = t^{\frac{p}{2(p-1)}}$. If the step-size parameter satisfies $a \leq \min\left\{1, \frac{\gamma_2}{1 + 8C_1^2}, \frac{(1-\eta)\gamma_1}{(1+8C_1^2)(LC_1 + (1-\eta)\widetilde{\nabla})}\right\}$, then for any $t \in \N$ and any $\delta \in (0,1)$, with probability at least $1 - \delta$, we have 
    \begin{equation*}
        \frac{1}{t}\sum_{k = 1}^t\min\{\|\nabla F(\bxk)\|,\|\nabla F(\bxk)\|^2\} = \mathcal{O}\left(t^{\eta-1}\Big(\log(\nicefrac{1}{\delta}) + \log(t) + \sum_{k = 1}^t(k+1)^{-2\eta} \Big)\right).    
    \end{equation*}
\end{theorem}

Theorem \ref{thm:n-msge} provides convergence guarantees for \nmsge. Compared to Theorems \ref{thm:n-sgd}, \ref{thm:n-sge}, the expression on the right-hand side of Theorem \ref{thm:n-msge} contains an additive $\log(t)$ term, which stems from the non-symmetric noise component and the bound in Lemma \ref{lm:gap-control}. Noting that the optimal choice of $\eta$ is given by $\eta = 1/2$, we immediately get the following corollary.

\begin{corollary}\label{cor:standard-p-th}
    Let conditions of Theorem \ref{thm:n-msge} hold, with $\eta = 1/2$. Then, for any $\delta \in (0,1)$, {there exists a problem related constant $R_\delta > 0$, such that, for any $t \geq R^2_\delta$}, with probability at least $1 - \delta$, we have 
    \begin{equation*}
        \min_{k \in [t]}\|\nabla F(\bxk)\|^2 = \mathcal{O}\left(\frac{\log(\nicefrac{t}{\delta})}{\sqrt{t}} \right).
    \end{equation*}
\end{corollary}

Corollary \ref{cor:standard-p-th} shows that \nmsge is guaranteed to achieve the same convergence rate as \nsgd and \nmsge. However, this comes at the expense of using an increasing batch size, whereas both \nsgd and \nmsge use a constant batch size. Combining the batch size requirement from Theorem \ref{thm:n-msge} and Corollary \ref{cor:standard-p-th}, we immediately get the following result on the overall oracle complexity incurred by the \nmsge method.

\begin{corollary}
    Let conditions of Theorem \ref{thm:n-msge} hold, with $\eta = 1/2$. Then, the oracle complexity of \nmsge for reaching an $\epsilon$-stationary point is of the order $\mathcal{O}\left(\epsilon^{-\frac{6p-4}{p-1}} \right)$, for any $\epsilon > 0$.
\end{corollary}

\section{Insights and Discussion}\label{sec:insight}

In this section we provide some intuition behind our improved results, as well as a discussion on the impact of our results and comparison with state-of-the-art (SOTA) results under heavy-tailed noise.

\paragraph{Proof sketch of Theorem \ref{thm:n-sgd}} To provide some intuition behind our approach and insight into our improved results, we give a proof sketch of Theorem \ref{thm:n-sgd}. To that end, let $\bPsi^{t} = \boldsymbol{\Psi}(\nabla f(\bxt,\xit))$ and $\bPhit = \E_{\bxt}\left[\bPsit \right] = \E\left[\bPsit \: \vert \: \bxt \right]$, with $\bet = \bPsit - \bPhit$ denoting the \emph{effective noise}, which allows us to move the stochasticity ``outside'' of the nonlinearity. Using $L$-smoothness of $F$ and update \eqref{eq:nsgd}, we get
\begin{align*}
    F(\bx^{t+1}) &\leq F(\bxt) - \alpha_t\langle \nabla F(\bxt), \bPsit \rangle + \frac{\alpha_t^2L}{2}\|\bPsit\|^2 \\
    &\leq F(\bxt) - \alpha_t\langle \nabla F(\bxt), \bPhit \rangle - \alpha_t\langle \nabla F(\bxt), \bet \rangle + \frac{\alpha_t^2LC^2}{2},
\end{align*} where in the second inequality we added and subtracted $\bPhit$ and used the uniform boundedness of $\bPsi$. Using Lemma \ref{lm:key-unified-state} to bound the first inner product, we get
\begin{align*}
    F(\bx^{t+1}) &\leq F(\bxt) - \alpha_t\min\{\beta_1\|\nabla F(\bxt)\|,\beta_2\|\nabla F(\bxt)\|^2\} - \alpha_t\langle \nabla F(\bxt), \bet \rangle + \frac{\alpha_t^2LC^2}{2},
\end{align*} Define $G_t = \min\left\{\beta_1\|\nabla F(\bxt)\|,\beta_2\|\nabla F(\bxt)\|^2 \right\}$, sum the first $t$ terms and rearrange, to get
\begin{equation}\label{eq:isolating-noise}
    \sum_{k = 1}^t\alpha_k G_k - \Delta - \frac{LC^2}{2}\sum_{k = 1}^t\alpha_k^2 \leq  -\sum_{k = 1}^t\alpha_k\langle \nabla F(\bxk), \bek \rangle.
\end{equation} Noting that the effective noise $\bek = \bPsi^k - \bPhi^k$ is unbiased and a difference of two bounded quantities (via Assumption \ref{asmpt:nonlin}), we can readily upper bound the moment-generating function (MGF) of the right-hand side (RHS) of \eqref{eq:isolating-noise} using sub-Gaussian inequalities, which is the approach used in \cite{armacki2023high}. However, this approach leads to the sub-optimal rate $\widetilde{\mathcal{O}}(t^{-1/4})$, for the following reason. Informally, applying the sub-Gaussian property on the MGF of the RHS of \eqref{eq:isolating-noise} iteratively, we get
\begin{equation*}
    \E\left[\exp\left(-\sum_{k = 1}^t\alpha_k\langle \nabla F(\bxk), \bek \rangle \right)\right] \leq \exp\left(4C^2\sum_{k = 1}^t\alpha_k^2\|\nabla F(\bxk)\|^2 \right).
\end{equation*} To bound the RHS, the authors in \cite{armacki2023high} use the update rule \eqref{eq:nsgd} and the fact that $\bPsi$ is uniformly bounded, to show that $\|\nabla F(\bxk)\|^2 = \mathcal{O}(k^{2(1-\eta)})$. This crude upper bound results in the constraint on the step-size parameter $\eta > 2/3$, resulting in sub-optimal rates. To circumvent this issue, we introduce an offset term, by subtracting $4C^2\sum_{k = 1}^t\alpha_k^2\|\nabla F(\bxk)\|^2$ from both sides in \eqref{eq:isolating-noise} and considering the MGF of $Z_t = \sum_{k = 1}^t\alpha_k (G_k - 4\alpha_kC^2\|\nabla F(\bxk)\|^2) - F(\bx^1) + F^\star - \frac{LC^2}{2}\sum_{k = 1}^t\alpha_k^2$, similarly to what was done in \cite{liu2023high}. While this resolves the issue of bounding the MGF, it poses another challenge: $Z_t$ now consists of terms with different orders, namely $G_k - 4\alpha_kC^2\|\nabla F(\bxk)\|^2$, where $G_k = \min\{\beta_1\|\nabla F(\bxk)\|,\beta_2\|\nabla F(\bxk)\|^2 \}$, which stems from the use of the nonlinearity.\footnote{This is very different from \cite{liu2023high}, where the authors consider linear \sgd in the presence of light-tailed noise and the introduction of the offsetting term results in difference of terms of the same order, namely $(1 - \alpha_k)\|\nabla F(\bxk)\|^2$, which can be obviously lower-bounded in terms of $\|\nabla F(\bxk)\|^2$, as $\alpha_k < 1$.} The rest of the proof is dedicated to a careful case-based analysis, showing that the difference of these competing terms can be controlled via a meticulously chosen step-size policy.

\paragraph{Discussion and comparison with SOTA} The results established in our work are strong - for symmetric noise and \nsgd, our improved analysis yields a $\mathcal{O}(\epsilon^{-4})$ oracle complexity for reaching $\epsilon$-stationary points, while for general, potentially non-symmetric noise, given access to a noiseless gradient at the start of training, our novel \nsge method attains the same complexity. {These results match the lower-bound oracle complexity for stochastic first-order methods established in \cite{Arjevani2023} for noises with bounded variance, while we consider the more general case of noises with unbounded moments and some additional structure (symmetric PDF, positive around zero). As such, our high-probability guarantees for a general class of nonlinear first-order methods under heavy-tailed noise match the oracle complexity in high-probability achieved by linear \sgd under noise with light-tails.} Compared to works assuming bounded $p$-th moment \cite{nguyen2023improved,hubler2024normalization,kornilov2025sign}, who require $\mathcal{O}\left(\epsilon^{-\frac{3p-2}{p-1}}\right)$ oracle calls, depending on $p \in (1,2]$ and vanishing as $p \rightarrow 1$, the complexity of \nsgd and \nsge is constant and strictly better whenever $p < 2$, i.e., for any \emph{heavy-tailed noise}. Additionally, our rates require no moment assumptions and are established for a general nonlinear framework, while the rates in \cite{nguyen2023improved,hubler2024normalization,kornilov2025sign} are established for specific nonlinearities. Compared to \cite{armacki2023high}, we provide a sharper rate, while significantly relaxing the noise assumptions, allowing for symmetric state-dependent, as well as non-symmetric noise. To handle the non-symmetric noise, we develop two novel methods, \nsge and \nmsge, showing both achieve strong guarantees.

While the oracle complexity $\mathcal{O}\left(\epsilon^{-\frac{6p-4}{p-1}}\right)$ of \nmsge is sub-optimal compared to \cite{nguyen2023improved,hubler2024normalization,kornilov2025sign}, the sub-optimality stems from the metric considered. In particular, note that the mini-batch size required by \nmsge in terms of the number of iterations, i.e., $B_t = t^{\frac{p}{2(p-1)}}$, exactly matches those required in both \cite{hubler2024normalization,kornilov2025sign}, who use mini-batch estimators for normalized and sign \sgd, respectively. Moreover, both \cite{hubler2024normalization,kornilov2025sign} match the convergence rate $\widetilde{\mathcal{O}}(t^{-1/2})$, achieved by \nmsge. However, the guarantees in \cite{hubler2024normalization,kornilov2025sign} are established in terms of $\frac{1}{t}\sum_{k = 1}^t\|\nabla F(\bxk)\|$ and $\frac{1}{t}\sum_{k = 1}^t\|\nabla F(\bxk)\|_1$, respectively, while our guarantees are established for the metric $\frac{1}{t}\sum_{k = 1}^t\min\{\|\nabla F(\bxk)\|,\|\nabla F(\bxk)\|^2\}$, implying guarantees on the \emph{squared norm}. This minor discrepancy leads to the iteration complexity of $t = \mathcal{O}(\epsilon^{-4})$ of \nmsge, whereas \cite{hubler2024normalization,kornilov2025sign} require $t = \mathcal{O}(\epsilon^{-2})$ iterations. If the guarantees for \nmsge were established in terms of the norm of the gradient, leading to $t = \mathcal{O}(\epsilon^{-2})$ iteration complexity, the total complexity of \nmsge would be $t\sum_{k = 1}^tB_k = \mathcal{O}\left(\epsilon^{-\frac{3p-2}{p-1}}\right)$, matching that from \cite{hubler2024normalization,kornilov2025sign}. However, we remark that the results of Lemma \ref{lm:key-iid} are provided for a general, black-box nonlinear framework, while the results in \cite{hubler2024normalization,kornilov2025sign} rely on the closed-form expression of the nonlinearity used therein.

\section{Conclusion}\label{sec:conclusion}

We provide tight high-probability guarantees for non-convex stochastic optimization in the presence of heavy-tailed noise. In particular, for state-dependent noise with symmetric PDF and the \nsgd family of methods, we establish the $\mathcal{O}(\epsilon^{-4})$ oracle complexity for reaching an $\epsilon$-stationary point, matching the performance of linear \sgd in the presence of light-tailed noise. Next, given access to a noiseless gradient at the start of training, we propose a novel family of methods dubbed \nsge, which also achieves the $\mathcal{O}(\epsilon^{-4})$ complexity for general, possibly non-symmetric noise. Finally, if we do not have access to a noiseless gradient, we propose a family of nonlinear methods dubbed \nmsge, which, for unbiased noises with bounded moment of order $p \in (1,2)$, guarantees the $\mathcal{O}(\epsilon^{-\frac{6p-4}{p-1}})$ oracle complexity. Compared to existing works which assume bounded $p$-th moments, our methods \nsgd and \nsge achieve strictly better oracle complexity whenever $p < 2$, i.e., \emph{for any heavy-tailed noise} and simultaneously relax the moment requirement and provide guarantees for a broader nonlinear framework. While \nmsge achieves a sub-optimal oracle complexity, it extends convergence guarantees to a much broader nonlinear framework compared to existing works. Ideas for future work include providing an estimator with optimal oracle complexity under bounded $p$-th moment, as well as establishing a high-probability oracle  complexity lower bound for reaching an $\epsilon$-stationary point in smooth non-convex optimization in the presence of noise with potentially unbounded moments and symmetric PDF, positive around zero, as in, e.g., \cite{nemirovski1983problem,Arjevani2023,pmlr-v247-carmon24a}.       

\bibliographystyle{ieeetr}
\bibliography{bibliography}

@article{jakovetic2023nonlinear,
    author = {Jakoveti\'{c}, Du\v{s}an and Bajovi\'{c}, Dragana and Sahu, Anit Kumar and Kar, Soummya and Milo\v{s}evi\'{c}, Nemanja and Stamenkovi\'{c}, Du\v{s}an},
    title = {Nonlinear Gradient Mappings and Stochastic Optimization: A General Framework with Applications to Heavy-Tail Noise},
    journal = {SIAM Journal on Optimization},
    volume = {33},
    number = {2},
    pages = {394-423},
    year = {2023},
    doi = {10.1137/21M145896X},
    URL = {https://doi.org/10.1137/21M145896X},
    eprint = {https://doi.org/10.1137/21M145896X}
}

@article{zhang2020improved,
  title={Improved analysis of clipping algorithms for non-convex optimization},
  author={Zhang, Bohang and Jin, Jikai and Fang, Cong and Wang, Liwei},
  journal={Advances in Neural Information Processing Systems},
  volume={33},
  pages={15511--15521},
  year={2020}
}

@article{yu2023smoothed,
      title={Smoothed Gradient Clipping and Error Feedback for Distributed Optimization under Heavy-Tailed Noise}, 
      author={Shuhua Yu and Du\v{s}an Jakoveti\'{c} and Soummya Kar},
      year={2023},
      journal={arXiv},
      eprint={2310.16920},
      archivePrefix={arXiv},
      primaryClass={math.OC}
}

@article{chen2020understanding,
  title={Understanding gradient clipping in private sgd: A geometric perspective},
  author={Chen, Xiangyi and Wu, Steven Z and Hong, Mingyi},
  journal={Advances in Neural Information Processing Systems},
  volume={33},
  pages={13773--13782},
  year={2020}
}

@inproceedings{sadiev2023highprobability,
      title={High-Probability Bounds for Stochastic Optimization and Variational Inequalities: the Case of Unbounded Variance}, 
      author={Abdurakhmon Sadiev and Marina Danilova and Eduard Gorbunov and Samuel Horváth and Gauthier Gidel and Pavel Dvurechensky and Alexander Gasnikov and Peter Richtárik},
      booktitle={International Conference on Machine Learning},
    pages={29563--29648},
  year={2023},
  organization={PMLR}
}

@inproceedings{armacki2023high,
  title = 	 {{High-probability Convergence Bounds for Online Nonlinear Stochastic Gradient Descent under Heavy-tailed Noise}},
  author =       {Armacki, Aleksandar and Yu, Shuhua and Sharma, Pranay and Joshi, Gauri and Bajovi\'{c}, Dragana and Jakoveti\'{c}, Du\v{s}an and Kar, Soummya},
  booktitle = 	 {Proceedings of The 28th International Conference on Artificial Intelligence and Statistics},
  pages = 	 {1774--1782},
  year = 	 {2025},
  volume = 	 {258},
  series = 	 {Proceedings of Machine Learning Research},
  publisher =    {PMLR},
  url = 	 {https://proceedings.mlr.press/v258/armacki25a.html}
}

@article{zhang2020adaptive,
  title={Why are adaptive methods good for attention models?},
  author={Zhang, Jingzhao and Karimireddy, Sai Praneeth and Veit, Andreas and Kim, Seungyeon and Reddi, Sashank and Kumar, Sanjiv and Sra, Suvrit},
  journal={Advances in Neural Information Processing Systems},
  volume={33},
  pages={15383--15393},
  year={2020}
}

@inproceedings{simsekli2019tail,
  title={A tail-index analysis of stochastic gradient noise in deep neural networks},
  author={Simsekli, Umut and Sagun, Levent and Gurbuzbalaban, Mert},
  booktitle={International Conference on Machine Learning},
  pages={5827--5837},
  year={2019},
  organization={PMLR}
}

@article{robbins1951stochastic,
  title={A stochastic approximation method},
  author={Robbins, Herbert and Monro, Sutton},
  journal={The annals of mathematical statistics},
  pages={400--407},
  year={1951},
  publisher={JSTOR}
}

@inproceedings{cutkosky20normalized_SGD,
  title={Momentum improves normalized sgd},
  author={Cutkosky, Ashok and Mehta, Harsh},
  booktitle={International conference on machine learning},
  pages={2260--2268},
  year={2020},
  organization={PMLR}
}

@article{ghadimi2013stochastic,
  title={Stochastic first-and zeroth-order methods for nonconvex stochastic programming},
  author={Ghadimi, Saeed and Lan, Guanghui},
  journal={SIAM Journal on Optimization},
  volume={23},
  number={4},
  pages={2341--2368},
  year={2013},
  publisher={SIAM}
}

@article{cutkosky2021high,
  title={High-probability bounds for non-convex stochastic optimization with heavy tails},
  author={Cutkosky, Ashok and Mehta, Harsh},
  journal={Advances in Neural Information Processing Systems},
  volume={34},
  pages={4883--4895},
  year={2021}
}

@inproceedings{bernstein2018signsgd_iclr,
    title={sign{SGD} with Majority Vote is Communication Efficient and Fault Tolerant},
    author={Jeremy Bernstein and Jiawei Zhao and Kamyar Azizzadenesheli and Anima Anandkumar},
    booktitle={International Conference on Learning Representations},
    year={2019},
    url={https://openreview.net/forum?id=BJxhijAcY7}
}

@inproceedings{zhang2022clip_FL_icml,
  title={Understanding clipping for federated learning: Convergence and client-level differential privacy},
  author={Zhang, Xinwei and Chen, Xiangyi and Hong, Mingyi and Wu, Zhiwei Steven and Yi, Jinfeng},
  booktitle={International Conference on Machine Learning, ICML 2022},
  year={2022}
}

@inproceedings{nguyen2023improved,
 author = {Nguyen, Ta Duy and Nguyen, Thien H and Ene, Alina and Nguyen, Huy},
 booktitle = {Advances in Neural Information Processing Systems},
 editor = {A. Oh and T. Neumann and A. Globerson and K. Saenko and M. Hardt and S. Levine},
 pages = {24191--24222},
 publisher = {Curran Associates, Inc.},
 title = {Improved Convergence in High Probability of Clipped Gradient Methods with Heavy Tailed Noise},
 url = {https://proceedings.neurips.cc/paper_files/paper/2023/file/4c454d34f3a4c8d6b4ca85a918e5d7ba-Paper-Conference.pdf},
 volume = {36},
 year = {2023}
}

@InProceedings{liu2023breaking,
  title = 	 {Breaking the Lower Bound with (Little) Structure: Acceleration in Non-Convex Stochastic Optimization with Heavy-Tailed Noise},
  author =       {Liu, Zijian and Zhang, Jiawei and Zhou, Zhengyuan},
  booktitle = 	 {Proceedings of Thirty Sixth Conference on Learning Theory},
  pages = 	 {2266--2290},
  year = 	 {2023},
  editor = 	 {Neu, Gergely and Rosasco, Lorenzo},
  volume = 	 {195},
  series = 	 {Proceedings of Machine Learning Research},
  month = 	 {12--15 Jul},
  publisher =    {PMLR},
  url = 	 {https://proceedings.mlr.press/v195/liu23c.html}
}

@ARTICLE{shuhua-clipping,
  author={Yu, Shuhua and Kar, Soummya},
  journal={IEEE Transactions on Signal Processing}, 
  title={Secure Distributed Optimization Under Gradient Attacks}, 
  year={2023},
  volume={71},
  number={},
  pages={1802-1816},
  doi={10.1109/TSP.2023.3277211}
}

@inproceedings{bernstein2018signsgd,
  title={sign{SGD}: Compressed optimisation for non-convex problems},
  author={Bernstein, Jeremy and Wang, Yu-Xiang and Azizzadenesheli, Kamyar and Anandkumar, Animashree},
  booktitle={International Conference on Machine Learning},
  pages={560--569},
  year={2018},
  organization={PMLR}
}

@inproceedings{liu2023high,
  title={High probability convergence of stochastic gradient methods},
  author={Liu, Zijian and Nguyen, Ta Duy and Nguyen, Thien Hang and Ene, Alina and Nguyen, Huy},
  booktitle={International Conference on Machine Learning},
  pages={21884--21914},
  year={2023},
  organization={PMLR}
}

@inproceedings{zhang2019gradient,
  title={Why Gradient Clipping Accelerates Training: A Theoretical Justification for Adaptivity},
  author={Zhang, Jingzhao and He, Tianxing and Sra, Suvrit and Jadbabaie, Ali},
  booktitle={International Conference on Learning Representations},
  year={2019}
}

@book{vershynin_2018,
    place={Cambridge}, 
    series={Cambridge Series in Statistical and Probabilistic Mathematics}, 
    title={High-Dimensional Probability: An Introduction with Applications in Data Science}, 
    DOI={10.1017/9781108231596}, 
    publisher={Cambridge University Press}, 
    author={Vershynin, Roman}, 
    year={2018}, 
    collection={Cambridge Series in Statistical and Probabilistic Mathematics}
}

@article{polyak-adaptive-estimation,
author = {Polyak, Boris and Tsypkin, Ya.Z.},
year = {1979},
month = {01},
pages = {},
title = {Adaptive estimation algorithms: Convergence, optimality, stability},
volume = {1979},
journal = {Automation and Remote Control}
}

@article{csimcsekli2019heavy,
  title={On the heavy-tailed theory of stochastic gradient descent for deep neural networks},
  author={Simsekli, Umut and Gurbuzbalaban, Mert and Nguyen, Thanh Huy and Richard, Gael and Sagun, Levent},
  journal={arXiv preprint arXiv:1912.00018},
  year={2019}
}

@book{nesterov-lectures_on_cvxopt,
author = {Nesterov, Yurii},
title = {Lectures on Convex Optimization},
year = {2018},
isbn = {3319915770},
publisher = {Springer Publishing Company, Incorporated},
edition = {2nd}
}

@article{gorbunov2023breaking,
  title={Breaking the Heavy-Tailed Noise Barrier in Stochastic Optimization Problems},
  author={Puchkin, Nikita and Gorbunov, Eduard and Kutuzov, Nikolay and Gasnikov, Alexander},
  journal={arXiv preprint arXiv:2311.04161},
  year={2023}
}

@book{bertsekas2003convex,
  title={Convex analysis and optimization},
  author={Bertsekas, Dimitri and Nedic, Angelia and Ozdaglar, Asuman},
  volume={1},
  year={2003},
  publisher={Athena Scientific}
}

@book{heavy-tail-book, 
    place={Cambridge}, 
    series={Cambridge Series in Statistical and Probabilistic Mathematics}, 
    title={The Fundamentals of Heavy Tails: Properties, Emergence, and Estimation}, 
    publisher={Cambridge University Press}, 
    author={Nair, Jayakrishnan and Wierman, Adam and Zwart, Bert}, 
    year={2022}, 
    collection={Cambridge Series in Statistical and Probabilistic Mathematics}
}

@InProceedings{heavy-tail-phenomena,
  title = 	 {The Heavy-Tail Phenomenon in SGD},
  author =       {Gurbuzbalaban, Mert and Simsekli, Umut and Zhu, Lingjiong},
  booktitle = 	 {Proceedings of the 38th International Conference on Machine Learning},
  pages = 	 {3964--3975},
  year = 	 {2021},
  volume = 	 {139},
  series = 	 {Proceedings of Machine Learning Research},
  publisher =    {PMLR},
  url = 	 {https://proceedings.mlr.press/v139/gurbuzbalaban21a.html}
}

@article{bertsekas-gradient,
    author = {Bertsekas, Dimitri P. and Tsitsiklis, John N.},
    title = {Gradient Convergence in Gradient methods with Errors},
    journal = {SIAM Journal on Optimization},
    volume = {10},
    number = {3},
    pages = {627-642},
    year = {2000},
    doi = {10.1137/S1052623497331063},
    URL = {https://doi.org/10.1137/S1052623497331063},
    eprint = {https://doi.org/10.1137/S1052623497331063}
}

@inproceedings{cevher-almost_sure,
 author = {Mertikopoulos, Panayotis and Hallak, Nadav and Kavis, Ali and Cevher, Volkan},
 booktitle = {Advances in Neural Information Processing Systems},
 editor = {H. Larochelle and M. Ranzato and R. Hadsell and M.F. Balcan and H. Lin},
 pages = {1117--1128},
 publisher = {Curran Associates, Inc.},
 title = {On the Almost Sure Convergence of Stochastic Gradient Descent in Non-Convex Problems},
 url = {https://proceedings.neurips.cc/paper_files/paper/2020/file/0cb5ebb1b34ec343dfe135db691e4a85-Paper.pdf},
 volume = {33},
 year = {2020}
}

@InProceedings{pmlr-v238-battash24a,
  title = 	 {Revisiting the Noise Model of Stochastic Gradient Descent},
  author =       {Battash, Barak and Wolf, Lior and Lindenbaum, Ofir},
  booktitle = 	 {Proceedings of The 27th International Conference on Artificial Intelligence and Statistics},
  pages = 	 {4780--4788},
  year = 	 {2024},
  volume = 	 {238},
  series = 	 {Proceedings of Machine Learning Research},
  publisher =    {PMLR},
  url = 	 {https://proceedings.mlr.press/v238/battash24a.html}
}

@inproceedings{barsbey-heavy_tails_and_compressibility,
 author = {Barsbey, Melih and Sefidgaran, Milad and Erdogdu, Murat A and Richard, Ga\"{e}l and Simsekli, Umut},
 booktitle = {Advances in Neural Information Processing Systems},
 editor = {M. Ranzato and A. Beygelzimer and Y. Dauphin and P.S. Liang and J. Wortman Vaughan},
 pages = {29364--29378},
 publisher = {Curran Associates, Inc.},
 title = {Heavy Tails in SGD and Compressibility of Overparametrized Neural Networks},
 url = {https://proceedings.neurips.cc/paper_files/paper/2021/file/f5c3dd7514bf620a1b85450d2ae374b1-Paper.pdf},
 volume = {34},
 year = {2021}
}

@InProceedings{pmlr-v108-peluchetti20b,
  title = 	 {Stable behaviour of infinitely wide deep neural networks},
  author =       {Peluchetti, Stefano and Favaro, Stefano and Fortini, Sandra},
  booktitle = 	 {Proceedings of the Twenty Third International Conference on Artificial Intelligence and Statistics},
  pages = 	 {1137--1146},
  year = 	 {2020},
  editor = 	 {Chiappa, Silvia and Calandra, Roberto},
  volume = 	 {108},
  series = 	 {Proceedings of Machine Learning Research},
  month = 	 {26--28 Aug},
  publisher =    {PMLR},
  url = 	 {https://proceedings.mlr.press/v108/peluchetti20b.html}
}

@book{lan2020first,
  title={First-order and stochastic optimization methods for machine learning},
  author={Lan, Guanghui},
  volume={1},
  publisher={Springer},
  isbn = {978-3-030-39567-4},
  doi = {10.1007/978-3-030-39568-1},
  year = {2020}
}

@article{armacki2024_ldp+mse,
  author = {Armacki, Aleksandar and Yu, Shuhua and Bajovi\'{c}, Dragana and Jakoveti\'{c}, Du\v{s}an and Kar, Soummya},
  title = {{Large Deviation Upper Bounds and Improved MSE Rates of Nonlinear SGD: Heavy-Tailed Noise and Power of Symmetry}},
  journal = {SIAM Journal on Optimization},
  volume = {36},
  number = {1},
  pages = {32-59},
  year = {2026},
  doi = {10.1137/24M1704154},
  URL = {https://doi.org/10.1137/24M1704154},
  eprint = {https://doi.org/10.1137/24M1704154}
}

@article{polyak1984criterial,
  title={Criterial algorithms of stochastic optimization},
  author={Polyak, Boris Teodorovich and Tsypkin, Yakov Zalmanovich},
  journal={Avtomatika i Telemekhanika},
  number={6},
  pages={95--104},
  year={1984},
  publisher={Russian Academy of Sciences, Branch of Power Industry, Machine Building~…}
}

@article{liu2020improved,
  title={An improved analysis of stochastic gradient descent with momentum},
  author={Liu, Yanli and Gao, Yuan and Yin, Wotao},
  journal={Advances in Neural Information Processing Systems},
  volume={33},
  pages={18261--18271},
  year={2020}
}

@InProceedings{sign-momentum,
  title = 	 {Momentum Ensures Convergence of {SIGNSGD} under Weaker Assumptions},
  author =       {Sun, Tao and Wang, Qingsong and Li, Dongsheng and Wang, Bao},
  booktitle = 	 {Proceedings of the 40th International Conference on Machine Learning},
  pages = 	 {33077--33099},
  year = 	 {2023},
  editor = 	 {Krause, Andreas and Brunskill, Emma and Cho, Kyunghyun and Engelhardt, Barbara and Sabato, Sivan and Scarlett, Jonathan},
  volume = 	 {202},
  series = 	 {Proceedings of Machine Learning Research},
  month = 	 {23--29 Jul},
  publisher =    {PMLR},
  url = 	 {https://proceedings.mlr.press/v202/sun23l.html}
}

@inproceedings{hubler2024normalization,
      title={From Gradient Clipping to Normalization for Heavy Tailed SGD}, 
      author={Florian Hübler and Ilyas Fatkhullin and Niao He},
      year={2025},
      booktitle={The 28th International Conference on Artificial Intelligence and Statistics},
      volume={258}  
}

@article{cutkosky2019momentum,
  title={Momentum-based variance reduction in non-convex sgd},
  author={Cutkosky, Ashok and Orabona, Francesco},
  journal={Advances in neural information processing systems},
  volume={32},
  year={2019}
}

@Article{Arjevani2023,
    author={Arjevani, Yossi
    and Carmon, Yair
    and Duchi, John C.
    and Foster, Dylan J.
    and Srebro, Nathan
    and Woodworth, Blake},
    title={Lower bounds for non-convex stochastic optimization},
    journal={Mathematical Programming},
    year={2023},
    month={May},
    day={01},
    volume={199},
    number={1},
    pages={165-214},
    issn={1436-4646},
    doi={10.1007/s10107-022-01822-7},
    url={https://doi.org/10.1007/s10107-022-01822-7}
}

@InProceedings{pmlr-v28-sutskever13,
  title = 	 {On the importance of initialization and momentum in deep learning},
  author = 	 {Sutskever, Ilya and Martens, James and Dahl, George and Hinton, Geoffrey},
  booktitle = 	 {Proceedings of the 30th International Conference on Machine Learning},
  pages = 	 {1139--1147},
  year = 	 {2013},
  editor = 	 {Dasgupta, Sanjoy and McAllester, David},
  number =       {3},
  series = 	 {Proceedings of Machine Learning Research},
  address = 	 {Atlanta, Georgia, USA},
  month = 	 {17--19 Jun},
  publisher =    {PMLR},
  url = 	 {https://proceedings.mlr.press/v28/sutskever13.html}
}

@article{yu2024distributed,
  title={Distributed Sign Momentum with Local Steps for Training Transformers},
  author={Yu, Shuhua and Zhou, Ding and Xie, Cong and Xu, An and Zhang, Zhi and Liu, Xin and Kar, Soummya},
  journal={arXiv preprint arXiv:2411.17866},
  year={2024}
}

@book{gut-probability,
    author = {Gut, Allan},
    year = {2014},
    month = {11},
    pages = {602},
    title = {Probability: a graduate course},
    isbn = {978-1-4899-9755-5},
    edition = {2},
    doi = {https://doi.org/10.1007/978-1-4614-4708-5},
    publisher={Springer New York, NY}
}

@inproceedings{raginsky2017non,
  title={Non-convex learning via stochastic gradient langevin dynamics: a nonasymptotic analysis},
  author={Raginsky, Maxim and Rakhlin, Alexander and Telgarsky, Matus},
  booktitle={Conference on Learning Theory},
  pages={1674--1703},
  year={2017},
  organization={PMLR}
}

@InProceedings{pmlr-v247-carmon24a,
  title = 	 {The Price of Adaptivity in Stochastic Convex Optimization},
  author =       {Carmon, Yair and Hinder, Oliver},
  booktitle = 	 {Proceedings of Thirty Seventh Conference on Learning Theory},
  pages = 	 {772--774},
  year = 	 {2024},
  volume = 	 {247},
  series = 	 {Proceedings of Machine Learning Research},
  publisher =    {PMLR},
  url = 	 {https://proceedings.mlr.press/v247/carmon24a.html}
}

@article{kornilov2025sign,
      title={Sign Operator for Coping with Heavy-Tailed Noise: High Probability Convergence Bounds with Extensions to Distributed Optimization and Comparison Oracle}, 
      author={Nikita Kornilov and Philip Zmushko and Andrei Semenov and Alexander Gasnikov and Alexander Beznosikov},
      year={2025},
      journal={arXiv preprint arXiv:2502.07923} 
}

@article{liu2025nonconvex,
  title={Nonconvex Stochastic Optimization under Heavy-Tailed Noises: Optimal Convergence without Gradient Clipping},
  author={Liu, Zijian and Zhou, Zhengyuan},
  journal={ICLR},
  year={2025}
}

@book{nevel1976stochastic,
  title={Stochastic Approximation and Recursive Estimation},
  author={Nevel'son, M.B. and Has'minskii, R.Z.},
  isbn={9780821809068},
  lccn={76048298},
  series={Translations of Mathematical Monographs},
  volume={47},
  year={1976},
  publisher={American Mathematical Society}
}

@ARTICLE{kar_estimaton,
  author={Kar, Soummya and Moura, José M. F. and Ramanan, Kavita},
  journal={IEEE Transactions on Information Theory}, 
  title={Distributed Parameter Estimation in Sensor Networks: Nonlinear Observation Models and Imperfect Communication}, 
  year={2012},
  volume={58},
  number={6},
  pages={3575-3605},
  doi={10.1109/TIT.2012.2191450}}

@article{al2001adaptive,
  title={Adaptive filters with error nonlinearities: Mean-square analysis and optimum design},
  author={Al-Naffouri, Tareq Y and Sayed, Ali H},
  journal={EURASIP journal on advances in signal processing},
  volume={2001},
  number={1},
  pages={192--205},
  year={2001},
  publisher={Hindawi Limited London, UK, United Kingdom}
}

@ARTICLE{transient-nonlin-filters,
  author={Al-Naffouri, T.Y. and Sayed, A.H.},
  journal={IEEE Transactions on Signal Processing}, 
  title={Transient analysis of adaptive filters with error nonlinearities}, 
  year={2003},
  volume={51},
  number={3},
  pages={653-663},
  doi={10.1109/TSP.2002.808108}}

@book{nemirovski1983problem,
  title={Problem Complexity and Method Efficiency in Optimization},
  author={Nemirovski{\u\i}, A.S. and Yudin, D.B.},
  isbn={9780471103455},
  lccn={82011065},
  series={A Wiley-Interscience publication},
  year={1983},
  publisher={Wiley}
}

@book{feller1971introduction,
  title={An Introduction to Probability Theory and Its Applications, Volume 2},
  author={Feller, William},
  isbn={0471257095},
  series={Wiley Series in Probability and Statistics},
  year={1971},
  publisher={John Wiley \& Sons, Inc.}
}

@article{nesterov-coordinate,
author = {Nesterov, Yu.},
title = {Efficiency of Coordinate Descent Methods on Huge-Scale Optimization Problems},
journal = {SIAM Journal on Optimization},
volume = {22},
number = {2},
pages = {341-362},
year = {2012},
doi = {10.1137/100802001},
URL = {https://doi.org/10.1137/100802001},
eprint = {https://doi.org/10.1137/100802001}
}

@book{bishop2006,
  author = {Bishop, Christopher M.},
  title = {Pattern Recognition and Machine Learning},
  isbn = {978-0-387-31073-2},
  publisher = {Springer New York, NY},
  year = {2006}
}

\appendix

\section{Introduction}

The appendix contains results omitted from the main body. Appendix \ref{app:on_noise} shows that the positive Lebesgue measure of $S_{\bx}$ from Assumption \ref{asmpt:non-sym} is guaranteed under a simple condition, Appendix \ref{app:on_nonlin} provides results demonstrating the generality of Assumptions \ref{asmpt:nonlin}-\ref{asmpt:nonlin-nmsge}, Appendix \ref{app:lemma-key} states and proves the full versions of Lemma \ref{lm:key-unified-state} and an intermediary result, Appendix \ref{app:proofs} contains the proofs omitted from the main body, Appendix \ref{app:example} provides further details on computing the constants $\gamma_1,\gamma_2$ and $\beta_1,\beta_2$ for specific instances of noise and nonlinearity, {Appendix \ref{app:bounds} shows how our bounds can be used to inform the choice of nonlinearity best suited to the problem at hand, while Appendix \ref{app:power-law} provides details on the noise from Example \ref{ex:example-1}.}

\section{On Assumption \ref{asmpt:non-sym}}\label{app:on_noise}

In this section we show that the condition on the positive Lebesgue measure of the set $S_{\bx} = \{ \by \in \R^d: \: P(\by) > 0 \text{ and } P(\by-\bx) > 0\}$, for any $\bx \in \R^d$ such that $\|\bx\| \leq E_0$, defined in Assumption \ref{asmpt:non-sym} is satisfied if the PDF $P$ is positive in a slightly larger neighbourhood around the origin. The formal result is stated next.

\begin{lemma}\label{lm:pos-Leb-measure}
    If there exist $0 < E_0 < E_1$, such that $P(\bx) > 0$ for all $\bx \in \R^d$ that satisfy $\|\bx\| \leq E_1$, then the set $S_{\bx} = \{ \by \in \R^d: \: P(\by) > 0 \text{ and } P(\by-\bx) > 0\}$ has positive Lebesgue measure for all $\bx \in \R^d$ that satisfy $\|\bx\| \leq E_0$.
\end{lemma}

\begin{proof}
    Let $\bx \in \R^d$ be such that $\|\bx\| \leq E_0$, let $E_2 \triangleq E_1 - E_0 > 0$ and define $\mathcal{B} \triangleq \{\by \in \R^d: \: \|\by\| \leq E_2\}$. Clearly, $\mathcal{B}$ is of positive Lebesgue measure. Next, for any $\by \in \mathcal{B}$, we have $P(\by) > 0$, since $E_2 < E_1$, as well as $\|\by - \bx\| \leq \|\by\| + \|\bx\| \leq E_2 + E_0 = E_1$, therefore $P(\by - \bx) > 0$. We can therefore conclude that $\mathcal{B} \subseteq S_{\bx}$, hence $S_{\bx}$ is itself of positive Lebesgue measure.
\end{proof}

\section{On Assumptions \ref{asmpt:nonlin}-\ref{asmpt:nonlin-nmsge}}\label{app:on_nonlin}

In this section we show the generality of Assumptions \ref{asmpt:nonlin}-\ref{asmpt:nonlin-nmsge}, by formally proving they are satisfied by a broad range of nonlinearities. Notice that Assumptions \ref{asmpt:nonlin}-\ref{asmpt:nonlin-nmsge} have a hierarchical structure, in that Assumption \ref{asmpt:nonlin} subsumes Assumption \ref{asmpt:nonlin-state}, which itself subsumes Assumption \ref{asmpt:nonlin-nmsge}. As such, we start with a formal result on the nonlinearities subsumed by the most ``stringent'' one, namely Assumption \ref{asmpt:nonlin-nmsge}.

\begin{lemma}\label{lm:smooth-nonlinearities}
    Assumption \ref{asmpt:nonlin-nmsge} is satisfied by the following nonlinear mappings:
    \begin{enumerate}
        \item Smooth sign: $[\bPsi(\bx)]_i = \psi(x_i)$, for all $i \in [d]$, where $\psi(x) \triangleq \tanh(x/k)$ and $k > 0$ is a user-specified constant.

        \item Smooth component-wise clipping: $[\bPsi(\bx)]_i = \psi(x_i)$, for all $i \in [d]$, where 
        \begin{equation*}
            \psi(x) \triangleq \begin{cases}
            m\cdot\textnormal{sign}(x), & |x| > m \\
            \frac{5}{8}\left(3x - \frac{2x^3}{m^2} + \frac{3x^5}{5m^4}\right), & |x| \leq m
        \end{cases},
        \end{equation*} and $m > 0$ is a user-specified constant.

        \item Smooth normalization: $\bPsi(\bx) \triangleq \frac{\bx}{\sqrt{\|\bx\|^2 + \epsilon}}$, where $\epsilon > 0$ is a user-specified constant.

        \item Smooth joint clipping: 
        \begin{equation*}
            \bPsi(\bx) \triangleq \begin{cases}
            M\frac{\bx}{\|\bx\|}, & \|\bx\| > M \\
            \frac{3}{2}\left(1 - \frac{\|\bx\|^2}{3M^2}\right)\bx, & \|\bx\| \leq M
        \end{cases},
        \end{equation*} where $M > 0$ is a user-specified constant.
    \end{enumerate}
\end{lemma}

\begin{proof}
    To prove our claim, we verify that each of the nonlinearities satisfies the conditions of Assumption \ref{asmpt:nonlin-nmsge}.

    \begin{enumerate}
        \item For ease of exposition, let $k = 1$, so that $\psi(x) = \tanh(x)$, noting that the general case of $k > 0$ follows similar arguments. To prove the claim, we will use the following well-known identities: $\tanh(x) = \frac{e^{x} - e^{-x}}{e^x + e^{-x}}$ and $\frac{d}{dx} [\tanh(x)] = 1 - \tanh^2(x)$. By definition, the function $\psi(x) = \tanh(x)$ is continuous. Moreover, using the formula for the derivative, it can readily be verified that $\psi^\prime(x) = 1 - \psi^2(x)$ and $\psi^{\prime\prime}(x) = -2\psi(x)\psi^\prime(x) = -2\psi(x)(1 - \psi^2(x))$, implying $\psi(x)$ is twice continuously differentiable and satisfies the first condition. Next, from the identity $\psi(x) = \frac{e^x - e^{-x}}{e^x + e^{-x}}$, it can be readily seen that $\psi$ is odd and uniformly bounded, i.e., that $\psi(-x) = -\psi(x)$ and $\psi(x) \in (-1,1)$. Using this fact and noting that $\psi^{\prime}(x) = 1 - \psi^2(x) \geq 0$, it follows that $\psi$ is odd and monotonically non-decreasing, satisfying the second property. Finally, from the previous discussion we know that $|\psi(x)| \leq 1$ and $|\psi^\prime(x)| = |1 - \psi^2(x)| \leq 1$, while it is easy to see that $|\psi^{\prime\prime}(x)| \leq 2|\psi(x)||(1 - \psi^2(x))| \leq 2$, verifying the third and final property.

        \item From the definition of smooth component-wise clipping, it can be readily seen that it is twice continuously differentiable on $|x| > m$ and $|x| < m$, with
        \begin{equation*}
            \psi^\prime(x) = \begin{cases}
            0, & |x| > m \\
            \frac{15}{8}\left(1 - \frac{2x^2}{m^2} + \frac{x^4}{m^4} \right), & |x| \leq m
        \end{cases},
        \end{equation*} and
        \begin{equation*}
            \psi^{\prime\prime}(x) = \begin{cases}
            0, & |x| > m \\
            \frac{15}{2}\left(\frac{x^3}{m^4} - \frac{x}{m^2} \right), & |x| \leq m
        \end{cases},
        \end{equation*} therefore it suffices to check if $\psi$ and its first and second derivatives are continuous at $x = \pm m$. This is clearly the case, as it can be readily verified that $\psi(\pm m) = \pm m$ and $\psi^\prime(\pm m) = \psi^{\prime\prime}(\pm m) = 0$, confirming the first property. Next, $\psi$ is clearly odd, as it is a polynomial of odd degrees on $|x| \leq m$ and proportional to the sign of the input on $|x|>m$. Moreover, noticing that the second derivative of $\psi$ is non-positive and non-negative on $x \in [0,m]$ and $x \in [-m,0]$, respectively, it readily follows that $\psi^\prime(x) \geq \psi^\prime(m) = 0$ on $x \in [0,m]$ and $\psi^\prime(x) \geq \psi^\prime(-m) = 0$ on $x \in [-m,0]$, implying that $\psi$ is non-decreasing, satisfying the second property. Finally, it is easy to see that $|\psi(x)| \leq m$, $|\psi^\prime(x)| \leq \frac{15}{8}$ and $|\psi^{\prime\prime}(x)| \leq \frac{5\sqrt{3}}{3m}$, proving the third property.

        \item Note that smooth normalization can be represented as $\bPsi(\bx) = \bx\varphi(\|\bx\|)$, where $\varphi(x) \triangleq \frac{1}{\sqrt{x^2 +\epsilon}}$, for any user-specified $\epsilon > 0$. Clearly, $\varphi$ is continuously differentiable, with the first derivative given by $\varphi^\prime(x) = -\frac{x}{(x^2+\epsilon)^{3/2}}$. Moreover, the mapping $h(x) = x\varphi(x)$ has a strictly positive first derivative, given by $h^\prime(x) = \frac{\epsilon}{(x^2+\epsilon)^{3/2}}$, satisfying the first property. Next, it is clear from the definitions of $\varphi$ and $\varphi^\prime$ that $\varphi$ is non-increasing and strictly positive on $(0,\infty)$, satisfying the second property. Finally, it can be easily seen that $\|\bx\varphi(\|\bx\|)\| = \frac{\|\bx\|}{\sqrt{\|\bx\|^2 + \epsilon}} \leq 1$ and that
        \begin{align*}
            \tnorm{\varphi^\prime(\|\bx\|)\frac{\bx\bx^\top}{\|\bx\|} + \varphi(\|\bx\|)I } &= \tnorm{\frac{\|\bx\|^2I - \bx\bx^\top}{(\|\bx\|^2 + \epsilon)^{3/2}} + \frac{\epsilon I}{(\|\bx\|^2 + \epsilon)^{3/2}}} \\
            &\leq \frac{\|\bx\|^2}{(\|\bx\|^2+\epsilon)^{3/2}} + \frac{\epsilon}{(\|\bx\|^2 + \epsilon)^{3/2}} \leq \frac{1}{\sqrt{\epsilon}},
        \end{align*} showing that smooth normalization satisfies all the required properties.
        
        \item Note that smooth clipping can be represented as $\bPsi(\bx) = \bx\varphi(\|\bx\|)$, where 
        \begin{equation*}
            \varphi(x) \triangleq \begin{cases}
            \frac{M}{|x|}, & |x| > M \\
            \frac{3}{2}\left(1 - \frac{x^2}{3M^2}\right), & |x| \leq M
        \end{cases},
        \end{equation*} for any user-specified $M > 0$. Clearly, $\varphi$ is continuously differentiable, with 
        \begin{equation*}
            \varphi^\prime(x) = \begin{cases}
            -\frac{M}{x|x|}, & |x| > M \\
            -\frac{x}{M^2}, & |x| \leq M
        \end{cases}.
        \end{equation*} Moreover, the first derivative of the mapping $h(x) = x\varphi(x)$ is given by 
        \begin{equation*}
            h^\prime(x) = \begin{cases}
            0, & |x| > M \\
            \frac{3}{2}\left(1 - \frac{x^2}{M^2}\right), & |x| \leq M
        \end{cases},    
        \end{equation*} which is clearly non-negative, satisfying the first property. Next, we can easily see from the definitions of $\varphi$ and $\varphi^\prime$ that $\varphi$ is strictly positive and non-increasing on $(0,\infty)$, satisfying the second property. Finally, it can be seen that $\|\bx\varphi(\|\bx\|)\| \leq M$ and that 
        \begin{align*}
            \tnorm{\varphi^\prime(\|\bx\|)\frac{\bx\bx^\top}{\|\bx\|} + \varphi(\|\bx\|)I } &= \begin{cases}
                \frac{M}{\|\bx\|^3}\tnorm{\|\bx\|^2I - \bx\bx^\top}, & \|\bx\| > M \\
                \tnorm{\frac{3}{2}\left(1 - \frac{\|\bx\|^2}{M^2}\right)I + \frac{\|\bx\|^2I - \bx\bx^\top}{M^2}}, & \|\bx\| \leq M
            \end{cases} \\ 
            &\leq \begin{cases}
                \frac{M}{\|\bx\|}, & \|\bx\| > M \\
                \frac{3}{2}\left(1 - \frac{\|\bx\|^2}{M^2}\right) + \frac{\|\bx\|^2}{M^2}, & \|\bx\| \leq M
            \end{cases} \leq \frac{3}{2},
        \end{align*} completing the proof.
    \end{enumerate}
\end{proof}

Next, note that by design, Assumption \ref{asmpt:nonlin-state} is satisfied by all nonlinearities satisfying Assumption \ref{asmpt:nonlin-nmsge}, as well as the joint ones satisfying Assumption \ref{asmpt:nonlin}. Therefore, we proceed with showing a result on which nonlinearities are subsumed by Assumption \ref{asmpt:nonlin}, the most general one. 

\begin{lemma}\label{lm:general-nonlinearities}
    Assumption \ref{asmpt:nonlin} is satisfied by the following nonlinear mappings:
    \begin{enumerate}
        \item Sign: $[\bPsi(\bx)]_i = \textnormal{sign}(x_i), \: i \in [d]$.

        \item Normalization: $\bPsi(\bx) = \bx / \|\bx\|$, for $\bx \neq \mathbf{0}$ and $\bPsi(\mathbf{0}) = \mathbf{0}$.
        
        \item Component-wise clipping: $[\bPsi(\bx)]_i = x_i$, for $|x_i| \leq m$, and $[\bPsi(\bx)]_i = m\cdot\text{sign}(x_i)$, for $|x_i| > m$, $i \in [d]$, and any user-specified $m>0$.
        
        \item Joint clipping: $\bPsi(\bx) = \min\{1,\nicefrac{M}{\|\bx\|}\}\bx$, for $\bx \neq \mathbf{0}$ and $\bPsi(\mathbf{0}) = \mathbf{0}$, for any user-specified $M>0$.

        \item Component-wise quantization: for each $i \in [d]$, let $[\bPsi(\bx)]_i = r_j$, for $x_i \in (q_j,q_{j+1}]$, with $j = 0,\ldots,J-1$ and $-\infty = q_0 < q_1 <\ldots < q_J = +\infty$, where $r_j$, $q_j$ are chosen such that each component of $\bPsi$ is odd, and we have $\max_{j \in \{0,\ldots,J-1\}}|r_j| \leq R$, for any user-specified $R > 0$.
    \end{enumerate}
\end{lemma}

\begin{proof}
    We prove the claim by verifying each of the nonlinearities satisfy the corresponding properties in Assumption \ref{asmpt:nonlin}. 
    \begin{enumerate}
        \item First note that sign nonlinearity is a component-wise nonlinearity, where $\psi(x) \triangleq \textnormal{sign}(x)$. It then suffices to verify that $\psi$ satisfies the properties required by Assumption \ref{asmpt:nonlin}. By definition, $\psi$ is continuous almost everywhere, as its only discontinuity point is $x = 0$. Also, by definition, it is monotonically non-decreasing, odd and piece-wise differentiable, with $\psi^\prime(x) = 0$, for $x \neq 0$. The third property is also satisfied, as $\psi$ is discontinuous at zero. Finally, $\psi(x) \in [-1,1]$, implying uniform boundedness. 

        \item Normalization is a joint nonlinearity, which can be written as $\bPsi(\bx) = \bx\varphi(\|\bx\|)$, where $\varphi(x) \triangleq 1/|x|$, for any $x \in \R$ and $\varphi(0) = 0$. Clearly, $\varphi$ is continuous almost everywhere, as its only discontinuity point is $x = 0$, with the mapping $h(x) = x\varphi(x) = 1$ for $x > 0$, therefore being non-decreasing on $(0,\infty)$. Similarly, by definition, $\varphi(x) = 1/|x|$ is non-increasing and strictly positive for any $x > 0$. Finally, the mapping $\bPsi(\bx) = \bx\varphi(\|\bx\|) = \bx / \|\bx\|$ is clearly uniformly bounded by $1$. 

        \item For component-wise clipping, we have 
        \begin{equation*}
            \psi(x) \triangleq \begin{cases}
            m\cdot\textnormal{sign}(x),& |x| > m \\
            x, & |x| \leq m
        \end{cases},
        \end{equation*} for any user-specified constant $m > 0$. Clearly, the map $\psi$ is odd, continuous and differentiable everywhere, with $|\psi^\prime(x)| \leq \max\{1,m\}$. Since $\psi(x) = x$, for all $x \in (-m,m)$, it readily follows that $\psi$ is strictly increasing on $(-m,m)$. Finally, by definition, we have $|\psi(x)| \leq m$, for any $x \in \R$. 

        \item For joint clipping, we have $\bPsi(\bx) = \bx\varphi(\|\bx\|)$, where $\varphi(x) \triangleq \min\{1,M/|x|\}$, for any $x \in \R$ and $\varphi(0) = 0$, where $M > 0$ is any user-specified constant. Clearly, $\varphi$ is continuous almost everywhere on its domain, as its only discontinuity point is $x = 0$. Now, consider the map $h(x) = x\varphi(x)$ and let $0 < x < y$. If $x < y < M$, we have $h(x) = x < y = h(y)$. If $x \leq M < y$, we then have $h(x) = x \leq M = h(y)$. Finally, if $M < x < y$, then $h(x) = M = h(y)$, therefore, $h(x) = x\varphi(x)$ is non-decreasing on $(0,\infty)$. By the same approach, it can be readily verified that $\varphi$ is non-increasing, as well as strictly positive on $(0,\infty)$. Finally, $\|\bx\varphi(\|\bx\|)\| = \min\{\|\bx\|,M\} \leq M$. 

        \item For ease of exposition, let $J = 2$, so that $q_0 = -\infty$, $q_1 = 0$ and $q_2 = \infty$, with the same argument readily extendable to any even integer $J > 2$ and equidistant intervals. Let $r_0 = -R$ and $r_1 = R$, for any $R > 0$. We then have 
        \begin{equation*}
            \psi(x) \triangleq \begin{cases}
            -R, & x \leq 0 \\
            R, & x > 0
        \end{cases}.
        \end{equation*} Clearly, $\psi$ is continuous almost anywhere, as its only discontinuity point is $x = 0$. By definition, $\psi$ is odd, monotonically non-decreasing and piece-wise differentiable, with $\psi^\prime(x) = 0$, for $x \neq 0$. As discussed, it is discontinuous at $x = 0$, satisfying point 3 in Assumption \ref{asmpt:nonlin}. Finally, $|\psi(x)| \leq R$, therefore the claim follows.         
    \end{enumerate}
\end{proof}

\section{On Lemma \ref{lm:key-unified-state}}\label{app:lemma-key}

In this section we prove Lemma \ref{lm:key-unified-state}, showing the full dependence of constants $\beta_1,\beta_2$ on noise, nonlinearity and other problem related constants. We start by restating Lemma \ref{lm:key-iid}, showing the full dependence of constants $\gamma_1,\gamma_2$ on noise, choice of nonlinearity and problem related parameters.

\begin{customthm}{1}
    Let Assumptions {\ref{asmpt:noise-state} and \ref{asmpt:nonlin} hold, with the stochastic noise vectors $\bzt = \nabla f(\bxt;\xi^t) - \nabla F(\bxt)$ being mutually IID and independent from the state for all $t \in \N$, i.e., $P_{\bx} \equiv P$, then the denoised nonlinearity satisfies the following inequality, for any $\bx \in \R^d$
    \begin{equation*}
        \langle \bPhi(\bx), \bx\rangle \geq \min\{\gamma_1\|\bx\|,\gamma_2\|\bx\|^2\},
    \end{equation*}} where $\gamma_1,\gamma_2 > 0$ are constants that depend on the noise, choice of nonlinearity and other problem related parameters. In particular, if the nonlinearity $\bPsi$ is component-wise, we have $\gamma_1 = \nicefrac{\phi^\prime(0)\xi}{2\sqrt{d}}$ and $\gamma_2 = \nicefrac{\phi^\prime(0)}{2d}$, where $\xi > 0$ is a constant that depends only on the noise and choice of nonlinearity, while $\phi^\prime(0) = \min_{i \in [d]}\phi_{i}^\prime(0) > 0$. If $\bPsi$ is a joint nonlinearity, then $\gamma_1 = p_0\varphi(1)/2$ and $\gamma_2 = p_0\varphi(1)$, where $p_0 = P(\mathbf{0})$.
\end{customthm}

Lemma \ref{lm:key-iid} establishes some important properties of the denoised nonlinearity $\bPhi$ for symmetric noise, however, it does not allow for state-dependent noise. As discussed in the main body, naively applying Lemma \ref{lm:key-iid} for state-dependent noise can lead to constants $\gamma_1,\gamma_2$ being state-dependent and we are unable to guarantee their uniform positivity. To facilitate state-dependent noise and ensure that the resulting problem related constants are independent of state and uniformly bounded away from zero, we provide a novel result on the behaviour of smooth component-wise nonlinearities. The lemma stated next utilizes a similar approach to a technical result from \cite{polyak-adaptive-estimation}, established under state-independent IID noise and Assumption \ref{asmpt:nonlin}. Our proofs utilize the additional smoothness in Assumption \ref{asmpt:nonlin-state} to show stronger properties of the denoised nonlinearity $\bPhi_{\bx}$ .

\begin{lemma}\label{lm:polyak-tsypkin}
    Let Assumptions \ref{asmpt:noise-state} and \ref{asmpt:nonlin-state} hold, with the nonlinearity $\bPsi: \R^d \mapsto \R^d$ being component wise, i.e., of the form $\bPsi(\by) = \begin{bmatrix} \psi(y_1),\ldots,\psi(y_d)\end{bmatrix}^\top$. Then, for any $\bx,\by \in \R^d$, the function $\bPhi_{\bx}: \R^d \mapsto \R^d$ is of the form $\bPhi_{\bx}(\by) = \begin{bmatrix} \phi_{1,\bx}(y_1),\ldots,\phi_{d,\bx}(y_d) \end{bmatrix}^\top$, where $\phi_{i,\bx}(y_i) = \E_{i,\bx}\left[\psi(y_i + z_i)\right]$ is the marginal expectation of the $i$-th noise component conditioned on the state $\bx$, with the following properties: 
    \begin{enumerate}
        \item $\phi_{i,\bx}$ is non-decreasing and odd, with $\phi_{i,\bx}(0) = 0$;

        \item $\phi_{i,\bx}$ is twice continuously differentiable, with uniformly bounded second derivatives, i.e., $|\phi^{\prime\prime}_{i,\bx}(y)| \leq K_2$, for all $\bx \in \R^d$, $y \in \R$. Additionally, $\phi_i^\prime(0) = \inf_{\bx \in \R^d}\phi_{i,\bx}^\prime(0) > 0$.
    \end{enumerate} 
\end{lemma}

Noise symmetry plays an important role in establishing the properties of the denoised nonlinearity in Lemma \ref{lm:polyak-tsypkin}, which can be seen in the proof ahead. This result is crucial to proving Lemma \ref{lm:key-unified-state}, which in turns facilitates the rest of the analysis. {Note that the quantity $\phi_i^\prime(0)$ in Lemma \ref{lm:polyak-tsypkin} is different to the one appearing in the statement of Lemma \ref{lm:key-iid} above, in the sense that Lemma \ref{lm:key-iid} requires the noise to be IID and independent of state, hence $\phi_i^\prime(0)$ in Lemma \ref{lm:key-iid} represents the first derivative of the denoised nonlinearity, taken with respect to the marginal PDF of the $i$-th noise component, without the infimum with respect to the state.} We next state an important intermediary result, used for joint joint nonlinearities.

\begin{lemma}[Lemma 6.1 in \cite{jakovetic2023nonlinear}]\label{lm:jakovetic-joint}
    If any of Assumptions \ref{asmpt:nonlin}-\ref{asmpt:nonlin-nmsge} holds, with the nonlinearity $\bPsi: \R^d \mapsto \R^d$ being joint, i.e., of the form $\bPsi(\bx) = \bx\varphi(\|\bx\|)$, then for any $\by, \bz \in \R^d$ such that $\|\bz\| > \|\by\|$
    \begin{equation*}
        \left|\varphi(\|\by + \bz\|) - \varphi(\|\by - \bz\|)\right| \leq \frac{\|\by\|}{\|\bz\|}\left[\varphi(\|\by + \bz\|) + \varphi(\|\by - \bz\|) \right].
    \end{equation*} 
\end{lemma}

Denote by $\phi^\prime(0) \triangleq \min_{i \in [d]}\phi_i^\prime(0)$ and $p_0 \triangleq \inf_{\bx \in \R^d}P_{\bx}(\mathbf{0})$. Next, we provide the extended version of Lemma \ref{lm:key-unified-state}, with the full dependence of constants $\beta_1,\beta_2$ on noise, choice of nonlinearity and other problem parameters.  

\begin{customthm}{2}
    Let Assumptions \ref{asmpt:noise-state} and \ref{asmpt:nonlin-state} hold. Then, for any $\bx,\by \in \R^d$, we have $\langle \bPhi_{\bx}(\by),\by\rangle \geq \min\left\{\beta_1\|\by\|,\beta_2\|\by\|^2 \right\}$, where $\beta_1,\beta_2 > 0$, are constants that only depend on noise, choice of nonlinearity and other problem related parameters. In particular, if the nonlinearity $\bPsi$ is component-wise, we have $\beta_1 = \nicefrac{(\phi^\prime(0))^2}{4K_2\sqrt{d}}$ and $\beta_2 = \nicefrac{\phi^\prime(0)}{2d}$. If $\bPsi$ is a joint nonlinearity, then $\beta_1 = p_0\varphi(1)/2$ and $\beta_2 = p_0\varphi(1)$, where $p_0 = \inf_{\bx \in \R^d}P_{\bx}(\mathbf{0})$.
\end{customthm}

Note the difference between constants $\gamma_1,\gamma_2$ in Lemma \ref{lm:key-iid} and $\beta_1,\beta_2$ in Lemma \ref{lm:key-unified-state}. In particular, while constants $\gamma_2 $ and $\beta_2$ exactly match, the constants $\gamma_1$ and $\beta_1$ differ in terms $\xi > 0$ in $\gamma_1$ and $\phi^\prime(0)/2K_2$ in $\beta_1$. This difference stems from the Taylor polynomial used in the proofs, with the first-order Taylor polynomial sufficing for Lemma \ref{lm:key-iid} and state-independent noise, while Lemma \ref{lm:key-unified-state} requires the second-order Taylor polynomial to circumvent the potential dependence of the constant $\xi$ on the state. We next prove Lemmas \ref{lm:key-unified-state} and \ref{lm:polyak-tsypkin}, starting with the intermediary Lemma \ref{lm:polyak-tsypkin}.

\begin{proof}[Proof of Lemma \ref{lm:polyak-tsypkin}]
    Fix any $\bx \in \R^d$ and $i \in [d]$. 
    \begin{enumerate}[leftmargin=*]
        \item For any $y \in \R$, we know that $\phi_{i,\bx}(y) = \int^{\infty}_{-\infty} \psi(y + z)P_{i,\bx}(z)dz$, where $P_{i,\bx}$ is the marginal distribution of the $i$-th noise component, conditioned on $\bx$. We then have
        \begin{align*}
            \phi_{i,\bx}&(-y) = \int^{\infty}_{-\infty} \psi(-y + z)P_{i,\bx}(z)dz = \int^{\infty}_{-\infty} -\psi(y - z)P_{i,\bx}(z)dz \\ &= \int^{-\infty}_{\infty} \psi(y + s)P_{i,\bx}(-s)ds = -\int^{\infty}_{-\infty} \psi(y + s)P_{i,\bx}(s)ds = -\phi_{i,\bx}(y),
        \end{align*} where in the second equality we used the oddity of $\psi$, in the third we introduced the substitution $s = -z$, while in the fourth we used the fact that $P_{i,\bx}(s)$ is symmetric\footnote{Follows from the fact that the joint distribution $P_{\bx}(\bz)$ is symmetric, hence the marginal with respect to each component is as well.} and $\int_a^bf(x)dx = -\int^a_bf(x)dx$, demonstrating that $\phi_{i,\bx}$ is odd. Next, for any $y \in \R$, we have 
        \begin{align}
            \phi_{i,\bx}(y) &= \int_{-\infty}^\infty\psi(y+z)P_{i,\bx}(z)dz = \int_{0}^\infty\psi(y+z)P_{i,\bx}(z)dz + \int_{-\infty}^0\psi(y+z)P_{i,\bx}(z)dz \nonumber \\ &= \int_{0}^\infty\psi(y+z)P_{i,\bx}(z)dz + \int_{0}^{\infty}\psi(y-s)P_{i,\bx}(s)ds, \label{eq:integral-split}
        \end{align} where in the third equality we used the substitution $s = -z$ and the facts that $P_{i,\bx}$ is symmetric and $\int_a^b-f(x)dx = \int_b^a f(x)dx$. For any $a,\: b \in \R$ such that $a \leq b$, we have
        \begin{align*}
            \phi_{i,\bx}(b) &- \phi_{i,\bx}(a) = \int_{0}^\infty\big[\psi(b+z) - \psi(a+z)\big]P_{i,\bx}(z)dz \\ &+ \int_{0}^{\infty}\big[\psi(b-z) - \psi(a-z)\big]P_{i,\bx}(s)ds \geq 0, 
        \end{align*} where we used \eqref{eq:integral-split} in the equality, while the inequality follows from the facts that $\psi$ is non-decreasing, showing that $\phi_{i,\bx}$ is non-decreasing. To verify the third property, notice that
        \begin{align*}
            \phi_{i,\bx}(0) &= \int_{-\infty}^\infty \psi(z)P_{i,\bx}(z)dz = \int_{-\infty}^0 \psi(z)P_{i,\bx}(z)dz + \int_{0}^\infty \psi(z)P_{i,\bx}(z)dz \\ &= \int_{\infty}^0 \psi(-s)P_{i,\bx}(-s)(-ds) + \int_{0}^\infty \psi(z)P_{i,\bx}(z)dz = 0,
        \end{align*} where in the third equality we introduced the substitution $s = -z$, while the last equality uses the fact that $\psi$ is odd, $P_{i,\bx}$ is symmetric and $\int_a^bf(x)dx = -\int_b^af(x)dx$.

        \item Let $\{a_n\}_{n \in \N}$ be a sequence such that $\lim_{n \rightarrow \infty}a_n = 0$ and for any $y \in \R$, consider
        \begin{equation}\label{eq:setup-lebesgue}
            \frac{\phi_{i,\bx}(y + a_n) - \phi_{i,\bx}(y)}{a_n} = \int_{-\infty}^\infty a_n^{-1}\big[\psi(y + z + a_n) - \psi(y + z) \big]P_{i,\bx}(z)dz. 
        \end{equation} Noting that $\psi$ is twice continuously differentiable, it follows that 
        \begin{equation*}
            \lim_{n \rightarrow \infty}a_n^{-1}\big[\psi(y + z + a_n) - \psi(y + z) \big] = \psi^\prime(y + z),
        \end{equation*} implying that the sequence $a_n^{-1}\big[\psi(y + z + a_n) - \psi(y + z) \big]$ can be bounded (point-wise in $y$, $z$) by some integrable function $g$, (the integrability stemming from the fact that both $\psi$ and $\psi^\prime$ are uniformly bounded). Therefore, taking the limit as $n \rightarrow \infty$ in \eqref{eq:setup-lebesgue}, we get 
        \begin{equation}\label{eq:derivative-form}
            \phi_{i,\bx}^\prime(y) = \lim_{n \rightarrow \infty}\int_{-\infty}^{\infty}a_n^{-1}\big[\psi(y + z + a_n) - \psi(y + z) \big]P_{i,\bx}(z)dz = \int_{-\infty}^{\infty}\psi^\prime(y + z)P_{i,\bx}(z)dz,
        \end{equation} where in the second equality we used Lebesgue's dominated convergence theorem to exchange the order of the limit and the integral. Moreover, using the fact that $\psi$ is strictly increasing on some interval $(-c_3,c_3)$ and that $\inf_{\bx \in \R^d}P_{\bx}(\by) > 0$, for all $\|\by\| \leq E_0$ and defining $c_4 = \min\{c_3,E_0\}$, we then get
        \begin{align*}
            \phi_i^\prime(0) &= \inf_{\bx \in \R^d}\int_{-\infty}^\infty\psi^\prime(z)P_{i,\bx}(z)dz \stackrel{(a)}{\geq} \int_{-\infty}^\infty\psi^\prime(z)\inf_{\bx \in \R^d}P_{i,\bx}(z)dz \\ &\geq \int_{-c_4}^{c_4}\psi^\prime(z)\inf_{\bx \in \R^d}P_{i,\bx}(z)dz \stackrel{(b)}{>} 0,
        \end{align*} where in $(a)$ we used the fact that $\min_{y}\E_x[f(x,y)] \geq \E_{x}[\min_y f(x,y)]$ for non-negative function $f$, while in $(b)$ we use the fact that $\inf_{\bx \in \R^d}P_{i,\bx}(z),\: \calN^\prime_1(z)> 0$, for all $z \in (-c_4,c_4)$, completing the first part of the claim.\footnote{Note that uniform positivity of the marginal $P_{i,\bx}$ in a neighbourhood of zero is implied from the uniform positivity of the joint PDF $P_{\bx}$.} Next, using the same arguments as above, it can be shown that the second derivative of $\phi_{i,\bx}$ is given by
        \begin{equation}\label{eq:second-derivative}
            \phi^{\prime\prime}_{i,\bx}(y) = \int_{-\infty}^\infty\psi^{\prime\prime}(y+z)P_{i,\bx}(z)dz.
        \end{equation} Since $|\psi^{\prime\prime}(y)| \leq K_2$, with $P_{i,\bx}$ being a PDF, using \eqref{eq:second-derivative}, we get
        \begin{equation*}
            |\phi^{\prime\prime}_{i,\bx}(y)| \leq \int_{-\infty}^\infty|\psi^{\prime\prime}(y+z)|P_{i,\bx}(z)dz \leq K_2\int_{-\infty}^\infty P_{i,\bx}(z)dz = K_2,
        \end{equation*} completing the proof.
    \end{enumerate}
\end{proof}

Next, we prove Lemma \ref{lm:key-unified-state}. 

\begin{proof}[Proof of Lemma \ref{lm:key-unified-state}]
    First, let $\bPsi$ be a component-wise nonlinearity. Using Lemma \ref{lm:polyak-tsypkin} and Taylor's expansion, it follows that, for any $y \in \R$, $\bx \in \R^d$ and $i \in [d]$, we have 
    \begin{equation}\label{eq:Taylor-sym-state}
        \phi_{i,\bx}(y) = \phi_{i,\bx}(0) + \phi_{i,\bx}^\prime(0)y + \phi^{\prime\prime}_{i,\bx}(\widetilde{y})y^2 \geq \phi^\prime(0)y - K_2y^2 = (\phi^\prime(0) - K_2y)y,
    \end{equation} where $\widetilde{y} \in (0,y)$, while in the inequality we used the facts that $\phi_{i,\bx}(0) = 0$, $\phi^\prime_{i,\bx}(0) \geq \phi^\prime(0) > 0$ and $|\phi^{\prime\prime}_{i,\bx}(y)| \leq K_2$. From \eqref{eq:Taylor-sym-state}, it readily follows that, for any $0 \leq y \leq \frac{\phi^\prime(0)}{2K_2}$, we have $\phi_{i,\bx}(y) \geq \frac{\phi^\prime(0)y}{2}$. Similarly, since $\phi_{i,\bx}$ is non-decreasing, for any $y > \frac{\phi^\prime(0)}{2K_2}$, we have $\phi_{i,\bx}(y) \geq \phi_{i,\bx}(\frac{\phi^\prime(0)}{2K_2}) \geq \frac{(\phi^\prime(0))^2}{4K_2}$. Combining both cases, it follows that, for any $y \geq 0$, $\bx \in \R^d$ and $i \in [d]$
    \begin{equation}\label{eq:lower-b-sym-state}
        \phi_{i,\bx}(y) \geq \min\left\{\frac{(\phi^\prime(0))^2}{4K_2},\frac{\phi^\prime(0)y}{2}\right\}.
    \end{equation} Since $\phi_i$ is odd (recall Lemma \ref{lm:polyak-tsypkin}), we have, for any $y \in \R$
    \begin{equation*}
        y\phi_{i,\bx}(y) = |y|\phi_{i,\bx}(|y|) \geq \min\left\{\frac{(\phi^\prime(0))^2|y|}{4K_2},\frac{\phi^\prime(0)y^2}{2}\right\},
    \end{equation*} where the last inequality follows from \eqref{eq:lower-b-sym-state}. Putting everything together, we get, for any $\bx,\by \in \R^d$
    \begin{align*}
        \langle \by, \bPhi_{\bx}(\by) \rangle &= \sum_{i = 1}^d y_i\phi_{i,\bx}(y_i)  = \sum_{i = 1}^d |y_i|\phi_{i,\bx}(|y_i|) \geq \max_{i \in [d]}|y_i|\phi_{i,\bx}(|y_i|) \\
        &\geq \max_{i \in [d]}\min\left\{\frac{\phi^\prime(0)y^2}{2} ,\frac{(\phi^\prime(0))^2|y|}{4K_2}\right\} = \min\left\{\frac{\phi^\prime(0)\|\by\|_{\infty}^2}{2} ,\frac{(\phi^\prime(0))^2\|\by\|_{\infty}}{4K_2}\right\} \\
        &\geq \min\left\{\frac{\phi^\prime(0)\|\by\|^2}{2d} ,\frac{(\phi^\prime(0))^2\|\by\|}{4K_2\sqrt{d}}\right\} ,
    \end{align*} where the last inequality follows from $\|\by\|_{\infty} \geq \|\by\| / \sqrt{d}$. Therefore, the claim holds for component-wise nonlinearities and state-dependent noise with $\beta_1 = \frac{(\phi^\prime(0))^2}{4K_2\sqrt{d}}$ and $\beta_2 = \frac{\phi^\prime(0)}{2d}$.
    
    For the case when $\bPsi(\by) = \by\varphi(\|\by\|)$ is a joint nonlinearity, the proof follows similar steps as in \cite{jakovetic2023nonlinear,armacki2023high}, replacing the PDF $P$ with $P_{\bx}$. For completeness, we provide the full proof. Fix an arbitrary $\by \in \R^d \setminus \{\mathbf{0}\}$. By the definition of $\bPsi$, we have
    \begin{align*}
        \langle \bPhi_{\bx}(\by),\by \rangle = \int_{\bz \in \R^d}\underbrace{(\by + \bz)^\top \by \varphi(\|\by + \bz\|)}_{\triangleq M(\by,\bz)}P_{\bx}(\bz)d\bz = {\int_{\bz \in \R^d}M(\by,\bz)P_{\bx}(\bz)d\bz}. 
    \end{align*} Next, by symmetry of $P_{\bx}$, it readily follows that $\langle \bPhi_{\bx}(\by),\by \rangle = \int_{J_1(\by)}M_2(\by,\bz)P_{\bx}(\bz)d\bz,$ where $J_1(\by) \triangleq \{\bz \in \R^d: \langle \bz,\by\rangle \geq 0 \}$ and 
    \begin{equation*}
        M_2(\by,\bz) \triangleq (\|\by\|^2 + \langle \bz,\by\rangle)\varphi(\|\by + \bz\|) + (\|\by\|^2 - \langle \bz,\by\rangle)\varphi(\|\by - \bz\|).
    \end{equation*} Consider the set $J_2(\by) \triangleq \left\{\bz \in \R^d: \frac{\langle \bz,\by\rangle}{\|\bz\|\|\by\|} \in [0,0.5] \right\} \cup \{\mathbf{0}\}$. Clearly $J_2(\by) \subset J_1(\by)$. Note that on $J_1(\by)$ we have $\|\by + \bz \| \geq \|\by - \bz\|$, which, together with the fact that $\varphi$ is non-increasing, implies
    \begin{equation}\label{eq:identity}
        \varphi(\|\by - \bz\|) - \varphi(\|\by + \bz\|) = \left|\varphi(\|\by - \bz\|) - \varphi(\|\by + \bz\|) \right|, 
    \end{equation} for any $\bz \in J_1(\by)$. For any $\bz \in J_2(\by)$ such that $\|\bz\| > \|\by\|$, we then have
    \begin{align*}
        M_2&(\by,\bz) = \|\by\|^2[\varphi(\|\by - \bz\|) + \varphi(\|\by + \bz\|)] - \langle \bz,\by\rangle[\varphi(\|\by - \bz\|) - \varphi(\|\by + \bz\|)] \\ &\stackrel{(a)}{=} \|\by\|^2[\varphi(\|\by - \bz\|) + \varphi(\|\by + \bz\|)] - \langle \bz,\by\rangle\left|\varphi(\|\by - \bz\|) - \varphi(\|\by + \bz\|)\right| \\ &\stackrel{(b)}{\geq} \|\by\|^2[\varphi(\|\by - \bz\|) + \varphi(\|\by + \bz\|)] - \langle \bz,\by\rangle\nicefrac{\|\by\|}{\|\bz\|}[\varphi(\|\by - \bz\|) + \varphi(\|\by + \bz\|)] \\ &\stackrel{(c)}{\geq} 0.5\|\by\|^2[\varphi(\|\by - \bz\|) + \varphi(\|\by + \bz\|)],
    \end{align*} where $(a)$ follows from \eqref{eq:identity}, $(b)$ follows from Lemma \ref{lm:jakovetic-joint}, while $(c)$ follows from the definition of $J_2(\by)$. Next, consider any $\bz \in J_2(\by)$, such that $0 < \|\bz\| \leq \|\by\|$. We have
    \begin{align*}
        M_2&(\by,\bz) = \|\by\|^2[\varphi(\|\by - \bz\|) + \varphi(\|\by + \bz\|)] - \langle \bz,\by\rangle[\varphi(\|\by - \bz\|) - \varphi(\|\by + \bz\|)] \\ &\stackrel{(a)}{=} \|\by\|^2[\varphi(\|\by - \bz\|) + \varphi(\|\by + \bz\|)] - \langle \bz,\by\rangle\left|\varphi(\|\by - \bz\|) - \varphi(\|\by + \bz\|)\right| \\ &\stackrel{(b)}{\geq} \|\by\|^2[\varphi(\|\by - \bz\|) + \varphi(\|\by + \bz\|)] - 0.5\|\by\|^2\left|\varphi(\|\by - \bz\|) - \varphi(\|\by + \bz\|)\right| \\ &\stackrel{(c)}{\geq} 0.5\|\by\|^2[\varphi(\|\by - \bz\|) + \varphi(\|\by + \bz\|)],
    \end{align*} where $(a)$ again follows from \eqref{eq:identity}, $(b)$ follows from the definition of $J_2(\by)$ and the fact that $0 < \|\bz\| \leq \|\by\|$, while $(c)$ follows from the fact that $\varphi$ is non-negative and 
    \begin{equation*}
        \left|\varphi(\|\by - \bz\|) - \varphi(\|\by + \bz\|)\right| \leq \varphi(\|\by - \bz\|) + \varphi(\|\by + \bz\|).
    \end{equation*} Finally, if $\bz = \mathbf{0}$, we have $M_2(\by,\mathbf{0}) = 2\|\by\|^2\varphi(\|\by\|) > 0.5\|\by\|^2[\varphi(\|\by + \mathbf{0}\|) + \varphi(\|\by - \mathbf{0}\|)]$. Therefore, for any $\bz \in J_2(\by)$, we have 
    \begin{equation*}
        M_2(\by,\bz) \geq 0.5\|\by\|^2[\varphi(\|\by - \bz\|) + \varphi(\|\by + \bz\|)] \geq \|\by\|^2\varphi(\|\by\| + \|\bz\|),    
    \end{equation*} where the second inequality follows from the fact that $\varphi$ is non-increasing and $\|\by \pm \bz\| \leq \|\by\| + \|\bz\|$. Note that following a similar argument as above, it can be shown that $M_2(\by,\bz) \geq 0$, for any $\by \in \R^d$ and $\bz \in J_1(\by)$. Combining everything, it readily follows that
    \begin{align}\label{eq:semi-done}
        \langle \bPhi_{\bx}(\by),\by \rangle \geq \int_{J_2(\by)}M_2(\by,\bz)P_{\bx}(\bz)d\bz \geq \|\by\|^2 \int_{J_2(\by)}\varphi(\|\by\| + \|\bz\|)P_{\bx}(\bz)d\bz,
    \end{align} where the first inequality follows from the fact that $M_2(\by,\bz) \geq 0$ on $J_1(\by)$. Define $C_0 \triangleq \min\left\{E_0,0.5 \right\}$ and consider the set $J_3(\by) \subset J_2(\by)$, defined as 
    \begin{equation*}
        J_3(\by) \triangleq \left\{\bz \in \R^d: \frac{\langle \bz,\by\rangle}{\|\bz\|\|\by\|} \in [0,0.5], \: \|\bz\| \leq C_0 \right\} \cup \{\mathbf{0}\}.
    \end{equation*} Since $a\varphi(a)$ is non-decreasing, it follows that $\varphi(a) \geq \varphi(1)\min\left\{a^{-1},1 \right\}$, for any $a > 0$. For any $\bz \in J_3(\by)$, it then holds that $\varphi(\|\bz\| + \|\by\|) \geq \varphi(1)\min\left\{1/(\|\by\|+C_0),1 \right\}$. Plugging in \eqref{eq:semi-done}, we then have
    \begin{align}
        \langle \bPhi_{\bx}(\by),\by \rangle &\geq \|\by\|^2 \int_{J_3(\by)}\varphi(\|\by\| + \|\bz\|)P_{\bx}(\bz)d\bz \geq \|\by\|^2\varphi(1)\min\left\{(\|\by\| + C_0)^{-1},1 \right\}\hspace{-0.45em}\int_{J_3(\by)}\hspace{-1.5em}P_{\bx}(\bz)d\bz \nonumber \\ &\geq \|\by\|^2\varphi(1)\min\left\{(\|\by\| + C_0)^{-1},1 \right\}P_\bx(\mathbf{0}) \geq \|\by\|^2\varphi(1)\min\left\{(\|\by\| + C_0)^{-1},1 \right\}p_0 \label{eq:intermed}.
    \end{align} If $\|\by\| \leq C_0$, it follows that $\|\by\| + C_0 \leq 2C_0$, therefore 
    \begin{equation*}
        \min\left\{1/(\|\by\|+C_0),1\right\} \geq \min\left\{1/(2C_0),1 \right\} \triangleq \kappa.
    \end{equation*} If $\|\by\| \geq C_0$, it follows that $\|\by\| + C_0 \leq 2\|\by\|$, therefore 
    \begin{equation*}
        \min\left\{1/(\|\by\|+C_0),1 \right\} \geq \min\left\{1/(2\|\by\|),1 \right\} \geq \min\left\{1/(2\|\by\|),\kappa \right\}.
    \end{equation*} Combining, we get $\langle \bPhi_{\bx}(\by),\by \rangle \geq p_0\varphi(1)\min\left\{\|\by\|/2,\kappa\|\by\|^2 \right\}$. Consider the constant $\kappa = \min\left\{\nicefrac{1}{(2C_0)},1 \right\}$. If $E_0 \geq 0.5$, it follows that $C_0 = 0.5$ and therefore $\kappa = 1$. On the other hand, if $E_0 < 0.5$, it follows that $C_0 = E_0$ and therefore $\kappa = \min\left\{\nicefrac{1}{(2E_0)},1 \right\} = 1$, as $2E_0 < 1$.
\end{proof}

\section{Missing proofs}\label{app:proofs}

In this section we provide proofs omitted from the main body. Subsection \ref{subapp:proof-thm1} contains proofs of results relating to \nsgd method, Subsection \ref{subapp:proof-thm2} provides proofs of results relating to the \nsge method, while Subsection \ref{subapp:proof-thm3} provides proofs of results relating to the \nmsge method.

In the remainder of this section, we provide a result on the effective noise $\bet$. To that end, let $\bgt \in \R^d$ be a generic gradient estimator, let $\bzt = \bgt - \nabla F(\bxt)$ denote the stochastic noise and define $\bPhi(\nabla F(\bxt)) \triangleq \E[\bPsi(\bgt) \: \vert \: \mathcal{F}_t]$, where $\mathcal{F}_t = \sigma(\{\bx^1,\bx^2,\ldots,\bxt\})$ is the natural filtration, with $\mathcal{F}_1 = \sigma(\{\Omega,\emptyset\})$, due to $\bx^1 \in \R^d$ being a deterministic quantity in our analysis. Defining $\bPhi$ via the natural filtration covers both state-dependent and state-independent noise, i.e., the entire range of our noise assumptions, in relations to the state. We then have the following general result.

\begin{lemma}\label{lm:sub-Gauss}
    Let $\bet = \bPsi(\bgt) - \bPhi(\nabla F(\bxt))$ denote the effective noise vector, where the original stochastic noise $\bzt$ is either independent of state, or depends on the history via the relation specified in Assumption \ref{asmpt:noise-state}. If the nonlinearity $\bPsi$ is uniformly bounded, i.e., $\|\bPsi(\bx)\| \leq C$, for some $C > 0$ and all $\bx \in \R^d$, then the effective noise satisfies the following properties.
    \begin{enumerate}
        \item $\E[\bet \vert \mathcal{F}_t] = \mathbf{0}$ and $\|\bet\| \leq 2C$.

        \item $\E[\exp(\langle \bet, \by) \rangle \: \vert \: \mathcal{F}_t] \leq \exp(4C^2\|\by\|^2)$, for any $\mathcal{F}_t$-measurable vector $\by \in \R^d$.
    \end{enumerate}
\end{lemma}

Lemma \ref{lm:sub-Gauss} is a generalization of a similar result from \cite{armacki2023high}, which was proved for state-independent, IID noise. Importantly, the second property implies that the effective noise is \emph{light-tailed}, regardless of the original noise. {For completeness, we now provide the proof for the state-dependent case considered in our work. 

\begin{proof}[Proof of Lemma \ref{lm:sub-Gauss}]
    To prove the first property, recall that the effective noise vector is defined as $\bet = \bPhit - \bPsit$, where $\bPhit = \mbe_\bzt \left[\bPsi (\nabla f(\bxt)+\bzt) \: \vert \: \mathcal{F}_t \right] = \mbe \left[\bPsit \: \vert \: \mathcal{F}_t \right]$ is the denoised version of $\bPsit$. By definition, it then follows that
    \begin{equation*}
        \E\left[\bet \vert \: \mathcal{F}_t\right] = \E\left[\bPhit - \bPsit \vert \: \mathcal{F}_t\right] = \bPhit - \E\left[ \bPsit \vert \: \mathcal{F}_t \right] = 0.
    \end{equation*} Moreover, since the nonlinearity is uniformly bounded and using the triangle and Jensen's inequality, we have 
    \begin{equation*}
        \|\bet \| = \|\bPhit - \bPsit\| \leq \|\bPhit\| + \|\bPsit\| \leq \E\|\bPsit\| +  C \leq 2C,
    \end{equation*} which proves the first claim. To prove the second claim, we apply the inequality $e^x \leq x + e^{x^2}$ (which holds for any $x \in \R$, see, e.g., \cite{vershynin_2018}) to $e^{\langle \bet, \by \rangle}$, to get
    \begin{equation*}
        \E\left[e^{\langle \bet,\by \rangle} \vert \: \mathcal{F}_t\right] \leq \E\left[\langle \bet, \by \rangle + e^{\langle \bet,\by \rangle^2} \vert \: \mathcal{F}_t\right] \leq \E\left[e^{\|\bet\|^2\|\by \|^2} \vert \: \mathcal{F}_t\right] \leq \exp(4C^2\|\by\|^2), 
    \end{equation*} where the second inequality follows from the fact that $\by$ is $\mathcal{F}_t$-measurable, $\bet$ is conditionally zero-mean and Cauchy-Schwartz, while the third follows from the fact that $\bet$ is uniformly bounded by $2C$, proven in the first part. 
\end{proof}
}

\subsubsection{Proofs for \nsgd}\label{subapp:proof-thm1}

In this section we provide proofs of Theorem \ref{thm:n-sgd} and Corollary \ref{cor:standard-metric}. We start with Theorem \ref{thm:n-sgd}.

\begin{proof}[Proof of Theorem \ref{thm:n-sgd}]
    Using $L$-smoothness of $F$ and the update rule \eqref{eq:nsgd}, we get
    \begin{align*}
        F(\bx^{t+1}) &\leq F(\bxt) - \alpha_t\langle \nabla F(\bxt), \bPsit \rangle + \frac{\alpha_t^2L}{2}\|\bPsit\|^2 \\
        &\leq F(\bxt) - \alpha_t\langle \nabla F(\bxt), \bPhit \rangle - \alpha_t\langle \nabla F(\bxt), \bet \rangle + \frac{\alpha_t^2LC^2}{2},
    \end{align*} where in the second inequality we added and subtracted $\bPhit$ and used the uniform boundedness of $\bPsi$ from Assumption \ref{asmpt:nonlin-state}. Using Lemma \ref{lm:key-unified-state} to bound the first inner product term, we get
    \begin{align*}
        F(\bx^{t+1}) &\leq F(\bxt) - \alpha_t\min\{\beta_1\|\nabla F(\bxt)\|,\beta_2\|\nabla F(\bxt)\|^2\} - \alpha_t\langle \nabla F(\bxt), \bet \rangle + \frac{\alpha_t^2LC^2}{2},
    \end{align*} Define $G_t = \min\left\{\beta_1\|\nabla F(\bxt)\|,\beta_2\|\nabla F(\bxt)\|^2 \right\}$, sum the first $t$ terms and rearrange, to get
    \begin{equation}\label{eq:isolating-noise-pf}
        \sum_{k = 1}^t\alpha_k G_k - \Delta - \frac{LC^2}{2}\sum_{k = 1}^t\alpha_k^2 \leq  -\sum_{k = 1}^t\alpha_k\langle \nabla F(\bxk), \bek \rangle.
    \end{equation} Recalling Lemma \ref{lm:sub-Gauss}, each of the effective noise terms on the RHS of \eqref{eq:isolating-noise-pf} is sub-Gaussian. To offset their effect, we use an idea similar to the one in \cite{liu2023high}, by subtracting $4C^2\sum_{k = 1}^t\alpha_k^2\|\nabla F(\bxk)\|^2$ from both sides in \eqref{eq:isolating-noise-pf} and considering the MGF of $Z_t = \sum_{k = 1}^t\alpha_k (G_k - 4\alpha_kC^2\|\nabla F(\bxk)\|^2) - \Delta - \frac{LC^2}{2}\sum_{k = 1}^t\alpha_k^2$. Using the shorthand $\E_t[\cdot] = \E[\cdot \vert \mathcal{F}_t]$, we get
    \begin{align*}
        &\E[\exp(Z_t)] \stackrel{(a)}{\leq} \E\left[\exp\left( -\sum_{k = 1}^t\alpha_k\left(\langle \nabla F(\bxk), \bek \rangle + 4\alpha_kC^2\|\nabla F(\bxk)\|^2 \right) \right) \right] \\
        &= \E\left[\exp\left( -\sum_{k = 1}^{t-1}\alpha_k\langle \nabla F(\bxk), \bek \rangle - \sum_{k = 1}^{t}4\alpha_k^2C^2\|\nabla F(\bxk)\|^2 \right)\E_t\left[\exp\left( -\alpha_t\langle \nabla F(\bxt), \bet \rangle \right) \right] \right] \\
        &\stackrel{(b)}{\leq} \E\left[\exp\left( -\sum_{k = 1}^{t-1}\alpha_k\langle \nabla F(\bxk), \bek \rangle -\sum_{k = 1}^{t} 4\alpha_k^2C^2\|\nabla F(\bxk)\|^2 + 4\alpha_t^2C^2\|\nabla F(\bxt)\|^2 \right) \right] \\ 
        &= \E\left[\exp\left( -\sum_{k = 1}^{t-1}\alpha_k\left(\langle \nabla F(\bxk), \bek \rangle + 4\alpha_kC^2\|\nabla F(\bxk)\|^2 \right)\right) \right],
    \end{align*} where $(a)$ follows from the definition of $Z_t$ and \eqref{eq:isolating-noise-pf}, while $(b)$ follows by applying Lemma \ref{lm:sub-Gauss}. Repeating the same arguments recursively, it can be readily seen that $\E[\exp(Z_t)] \leq 1$. Combined with the exponential Markov inequality, for any $\epsilon > 0$, we get 
    \begin{equation*}
        \mathbb{P}(Z_t > \epsilon) \leq \exp\left(-\epsilon\right)\E\left[\exp\left(Z_t\right)\right] \leq \exp(-\epsilon),
    \end{equation*} or equivalently, using the definition of $Z_t$, we get, for any $\delta \in (0,1)$ 
    \begin{equation}\label{eq:tail-bdd}
        \mathbb{P}\left(\sum_{k = 1}^t\alpha_k (G_k - 4\alpha_kC^2\|\nabla F(\bxk)\|^2) \leq \log(\nicefrac{1}{\delta}) + \Delta + \frac{LC^2}{2}\sum_{k = 1}^t\alpha_k^2\right) \geq 1 - \delta.
    \end{equation} To complete the proof, we need to show that the quantity on the left-hand side of the above expression, $G_k - 4\alpha_kC^2\|\nabla F(\bxk)\|^2$, is non-negative for each $k \in [t]$. Recalling that $G_k = \min\{\beta_1\|\nabla F(\bxk)\|,\beta_2\|\nabla F(\bxk)\|^2\}$, we can see that $G_k$ and $\|\nabla F(\bxk)\|^2$ are of different orders. To bound their difference, we consider two cases. First, let $\|\nabla F(\bxk)\| \leq \beta_1/\beta_2$. By the definition of $G_k$, we then have $G_k = \beta_2\|\nabla F(\bxk)\|^2$, hence
    \begin{equation}\label{eq:bdd-G_k-pt1}
        G_k - 4\alpha_kC^2\|\nabla F(\bxk)\|^2 = (1 - 4\alpha_kC^2/\beta_2)G_k = \left(1 - \frac{4aC^2}{\beta_2(k+1)^{\eta}} \right)G_k \geq \frac{G_k}{2}, 
    \end{equation} where the last inequality follows since $k \geq 1$ and by choosing $a \leq \frac{\beta_2}{8C^2}$. Alternatively, if $\|\nabla F(\bxk)\| > \beta_1/\beta_2$, then $G_k = \beta_1\|\nabla F(\bxk)\|$, hence 
    \begin{equation}\label{eq:bdd-G_k-pt2-step1}
        G_k - 4\alpha_kC^2\|\nabla F(\bxk)\|^2 = (1 - 4\alpha_kC^2\|\nabla F(\bxk)\|/\beta_1)G_k.
    \end{equation}
    Now, consider $\|\nabla F(\bxk)\|$. Recalling that $\widetilde{\nabla} = \|\nabla F(\bx^1)\|$, we can then upper bound the quantity $\|\nabla F(\bxk)\|$ as follows
    \begin{align}
        \|\nabla F(\bxk)\| &\stackrel{(a)}{\leq} \|\nabla F(\bxk) - \nabla F(\bx^1)\| + \widetilde{\nabla} \stackrel{(b)}{\leq} L\|\bxk - \bx^1\| + \widetilde{\nabla} \nonumber \\ &\stackrel{(c)}{=} L\|\bx^{k-1} - \alpha_{k-1}\bPsi^{k-1} - \bx^1\| + \widetilde{\nabla} \stackrel{(d)}{\leq} L(\|\bx^{k-1} - \bx^1\| + \alpha_{k-1}C) + \widetilde{\nabla} \nonumber \\ 
        &\leq \ldots \leq LC\sum_{s = 1}^{k-1}\alpha_s + \widetilde{\nabla} \stackrel{(e)}{\leq} aLC\frac{k^{1-\eta}}{1-\eta} + \widetilde{\nabla} \stackrel{(f)}{\leq} \bigg(\frac{LC}{1-\eta} + \widetilde{\nabla}\bigg)k^{1-\eta} , \label{eq:grad-bound}
    \end{align} where in $(a)$ we added and subtracted $\nabla F(\bx^1)$ and used the triangle inequality, $(b)$ follows from $L$-smoothness of $F$, in $(c)$ we used the update rule \eqref{eq:nsgd}, in $(d)$ we used the triangle inequality and uniform boundedness of $\bPsi$, $(e)$ follows from the definition of the step-size and using the lower Darboux sum to bound $\sum_{s = 1}^{k-1}\alpha_s$, while in $(f)$ we used the fact that $k \geq 1$ and we choose $a \leq 1$. Plugging \eqref{eq:grad-bound} into \eqref{eq:bdd-G_k-pt2-step1}, we get
    \begin{align}
        G_k - 4\alpha_kC^2\|\nabla F(\bxk)\|^2 &\geq \left(1 - 4aC^2\bigg(\frac{LC}{1-\eta} + \widetilde{\nabla}\bigg)\frac{k^{1-2\eta}}{\beta_1}\right)G_k \nonumber \\ &\geq \left(1 - \frac{4aC^2(LC +(1-\eta)\widetilde{\nabla})}{(1-\eta)\beta_1}\right)G_k \geq \frac{G_k}{2} \label{eq:bdd-G_k-pt2},
    \end{align} where the first inequality follows from \eqref{eq:grad-bound}, the second follows from the fact that $\eta \geq 1/2$ and $k \geq 1$, hence $k^{1-2\eta} \leq 1$, while the third follows by choosing $a \leq \frac{(1-\eta)\beta_1}{8C^2(LC+(1-\eta)\widetilde{\nabla})}$. Combining \eqref{eq:bdd-G_k-pt1} and \eqref{eq:bdd-G_k-pt2}, it follows that $G_k - 4\alpha_kC^2\|\nabla F(\bxk)\|^2 \geq \frac{G_k}{2}$, for every $k \geq 1$. Therefore, using \eqref{eq:tail-bdd}, we get that, for any $\delta \in (0,1)$, with probability at least $1 - \delta$
    \begin{align*}
        \sum_{k = 1}^t\alpha_k G_k \leq 2\big(\log(\nicefrac{1}{\delta}) + \Delta\big) + LC^2\sum_{k = 1}^t\alpha_k^2.
    \end{align*} Since the step-size $\alpha_k = \frac{a}{(k+1)^{\eta}}$ is decreasing, it follows that $\alpha_t\sum_{k = 1}^tG_k \leq \sum_{k = 1}^t\alpha_kG_k$. Using this fact and dividing both sides of the above equation by $\alpha_tt$ gives
    \begin{equation*}
        \frac{1}{t}\sum_{k = 1}^t G_k \leq \frac{4\big(\log(\nicefrac{1}{\delta}) + \Delta\big) + 2LC^2\sum_{k = 1}^t\alpha_k^2}{at^{1-\eta}}.
    \end{equation*} Define $\beta = \min\{\beta_1,\beta_2\}$ and notice that $\min\{\|\nabla F(\bxk)\|,\|\nabla F(\bxk)\|^2\} \leq G_k/\beta$, for every $k \geq 1$. The proof is then completed by dividing both sides of the above equation by $\beta$ and using the said inequality.
\end{proof}

Next, we prove Corollary \ref{cor:standard-metric}.

\begin{proof}[Proof of Corollary \ref{cor:standard-metric}]
    For ease of exposition, consider first a deterministic bound of the form
    \begin{equation*}
        \frac{1}{t}\sum_{k = 1}^t\min\{\|\nabla F(\bxk)\|,\|\nabla F(\bxk)\|^2\} \leq Rt^{-\kappa},
    \end{equation*} for some constants $\kappa > 0$, $R > 0$ and any $t \in \N$. Define the sets $U = \{k \in [t]: \: \|\nabla F(\bxk)\| \leq 1\}$ and $U^c = [t] \setminus U$. It then follows that 
    \begin{equation*}
        \frac{1}{t}\sum_{k \in U}\|\nabla F(\bxk)\|^2 \leq Rt^{-\kappa} \hspace{0.5em} \text{and}\hspace{0.5em} \frac{1}{t}\sum_{k \in U^c}\|\nabla F(\bxk)\| \leq Rt^{-\kappa}.
    \end{equation*} Now, consider the values $\{z_k\}_{k \in [t]}$, where $z_k = \|\nabla F(\bxk)\|$, for $k \in U$, otherwise $z_k = 0$. Using Jensen's inequality, we then have
    \begin{equation*}
        \frac{1}{t}\sum_{k \in U}\|\nabla F(\bxk)\| = \frac{1}{t}\sum_{k = 1}^tz_k \leq \sqrt{\frac{1}{t}\sum_{k = 1}^tz_k^2} = \sqrt{\frac{1}{t}\sum_{k \in U}\|\nabla F(\bxk)\|^2} \leq \sqrt{R}t^{-\kappa/2}.
    \end{equation*} Therefore, we have
    \begin{equation*}
        \min_{k \in [t]}\|\nabla F(\bxk)\| \leq \frac{1}{t}\sum_{k = 1}^t\|\nabla F(\bxk)\| \leq \sqrt{R}t^{-\kappa/2} + Rt^{-\kappa},
    \end{equation*} or equivalently
    \begin{equation*}
        \min_{k \in [t]}\|\nabla F(\bxk)\|^2 \leq (\sqrt{R}t^{-\kappa/2}+Rt^{-\kappa})^2 \leq 2Rt^{-\kappa}+2R^2t^{-2\kappa} = \mathcal{O}(t^{-\kappa}).
    \end{equation*}
    Noticing that, if we have a tail inequality of the form 
    \begin{equation*}
        \mathbb{P}\Big(\frac{1}{t}\sum_{k = 1}^t\min\{\|\nabla F(\bxk)\|,\|\nabla F(\bxk)\|^2\}\geq R_{\delta}t^{-\delta}\Big) \leq \delta,
    \end{equation*} where $R_\delta$ now possibly depends on $\delta$, for any $\delta \in (0,1)$, it then follows that   
    \begin{align*}
        \mathbb{P}(\min_{k \in [t]}\|\nabla F(\bxk)\|^2 \geq 2R_{\delta}t^{-\kappa} &+ 2R_{\delta}^2t^{-2\kappa}) \leq \mathbb{P}(\min_{k \in [t]}\|\nabla F(\bxk)\| \geq \sqrt{R_{\delta}}t^{-\kappa/2}+R_{\delta}t^{-\kappa}) \\ 
        &\leq \mathbb{P}\Big(\frac{1}{t}\sum_{k = 1}^t\min\{\|\nabla F(\bxk)\|,\|\nabla F(\bxk)\|^2\}\geq R_{\delta}t^{-\kappa}\Big) \leq \delta. 
    \end{align*} {The claim now follows by noticing that, for any $t \geq R_\delta^{\frac{1}{\kappa}}$, we have
    \begin{equation*}
        \mathbb{P}(\min_{k \in [t]}\|\nabla F(\bxk)\|^2 \geq 2R_{\delta}t^{-\kappa} + 2R_{\delta}^2t^{-2\kappa}) \geq \mathbb{P}(\min_{k \in [t]}\|\nabla F(\bxk)\|^2 \geq 4R_{\delta}t^{-\kappa}).
    \end{equation*}}
\end{proof}

\subsubsection{Proofs for \nsge}\label{subapp:proof-thm2}

In this section we prove Lemma \ref{lm:sge-symmetrize} and Theorem \ref{thm:n-sge}. We start by proving Lemma \ref{lm:sge-symmetrize}.

\begin{proof}[Proof of Lemma \ref{lm:sge-symmetrize}]
    First note that the noise vectors $\{\wbzt\}_{t \in \N}$ are IID, since $\{\bz_1^t\}_{t \in \N}$ and $\{\bz_2^t\}_{t \in \N}$ are. We now want to verify that the noise has a PDF which is symmetric and positive in a neighbourhood of zero. To do so, we use the convolution formula, see, e.g., \cite[Theorem 9.4]{gut-probability}. For ease of notation, let $\bz = \bx - \by$, where $\bx, \by \in \R^d$ are two IID random vectors with PDF $P$. Then, for any $A \subset \R^d$, we have
    \begin{equation*}
        \mathbb{P}_{\bz}(A) = \mathbb{P}(\bz \in A) = \mathbb{P}(\bx - \by \in A) = \mathbb{P}(\bx \in A_{\by}),
    \end{equation*} where $A_{\by} = \{\bg + \by: \: \bg \in A\}$. Using the fact that $\bx$ and $\by$ are independent, we then have
    \begin{equation*}
        \mathbb{P}_{\bz}(A) = \int_{\by \in \R^d}\int_{\bx \in A_{\by}}P(\bx)d\bx P(\by)d\by = \int_{\by \in \R^d}\int_{\mathbf{s} \in A}P(\bs + \by)d\bs P(\by)d\by,
    \end{equation*} where in the second equality we used the substitution $\bs = \bx - \by$. Rearranging, we then get
    \begin{equation*}
         \mathbb{P}_{\bz}(A) = \int_{\bs \in A}\int_{\by \in \R^d}P(\bs + \by)P(\by)d\by d\bs,
    \end{equation*} readily impying that the PDF of $\bz$ is given by $\widetilde{P}(\bs) = \int_{\by \in \R^d}P(\bs + \by)P(\by)d\by$. Therefore, the PDF of $\wbzt$ is given by $\widetilde{P}$. Introducing the substitution $\bu = \bs + \by$, we can easily see that $\widetilde{P}(\bs) = \int_{\bu \in \R^d}P(\bu)P(\bu - \bs)d\bu$. For any $\bs \in \R^d$, we then have
    \begin{equation}\label{eq:symmetric-PDF}
        \widetilde{P}(-\bs) = \int_{\bu \in \R^d}P(\bu)P(\bu + \bs)d\bx = \int_{\by \in \R^d}P(\by+s)P(\by)d\by = \widetilde{P}(\bs), 
    \end{equation} where in the second equality we used the substitution $\by = \bu - \bs$. Equation \eqref{eq:symmetric-PDF} readily implies that $\widetilde{P}$ is symmetric. Next, let $\bx \in \R^d$ be such that $\|\bx\| \leq E_0$. From Assumption \ref{asmpt:non-sym}, we have 
    \begin{equation*}
        \widetilde{P}(\bx) = \int_{\bs \in \R^d}P(\bx)P(\bs - \bx)d\bs \geq \int_{\bs \in S_{\bx}}P(\bs)P(\bs - \bx)d\bs > 0, 
    \end{equation*} where the last inequality follows from the fact that $P(\bs), P(\bs - \bx) > 0$ for all $\bs \in S_{\bx}$ and that $S_{\bx}$ is of positive Lebesgue measure. We can therefore conclude that $\widetilde{P}(\bx) > 0$, for any $\|\bx\| \leq E_0$, completing the proof.
\end{proof}  

We are now ready to prove Theorem \ref{thm:n-sge}.

\begin{proof}[Proof of Theorem \ref{thm:n-sge}]
    Recalling the \nsge update rule \eqref{eq:nsge} and using Lemma \ref{lm:sge-symmetrize}, it is now apparent that \nsge symmetrizes the noise, effectively behaving like the \nsgd method under symmetric noise, even if the original noise is not symmetric. We can then use Lemma \ref{lm:key-iid} to characterize the behaviour of the resulting denoised nonlinearity $\bPhi(\nabla F(\bxt) ) = \E\left[\bPsi(\nabla F(\bxt) + \wbzt)\right]$. The rest of the proof then follows the same steps as in Theorem \ref{thm:n-sgd}, with constants $\beta_1,\beta_2,\beta$ replaced by $\gamma_1,\gamma_2,\gamma$. We omit it, for brevity.
\end{proof}

\subsubsection{Proofs for \nmsge}\label{subapp:proof-thm3}

In this section we prove Lemma \ref{lm:gap-control} and Theorem \ref{thm:n-msge}. To that end, recall that the \nmsge update rule is given by $\bx^{t+1} = \bxt - \alpha_t\bPsi(\bg^t)$, where $\bg^t = \nabla F(\bxt) + \bz_1^t - \bz_2^t + \frac{1}{B_t}\sum_{j = 3}^{B_t+2}\bz_j^t$. For ease of exposition, we define $\widetilde{\bz}^t = \bz_1^t - \bz_2^t$ and $\overline{\bz}^t = \frac{1}{B_t}\sum_{j = 3}^{B_t+2}\bz_j^t$. We know from the previous section that $\widetilde{\bz}^t$ is the ``symmetrized'' noise component, however, we now need to deal with an additional, potentially non-symmetric, noise component $\overline{\bz}^t$. To facilitate the analysis, we define the following mappings $\bPhi(\bx) = \E_{\widetilde{\bz}}[\bPsi(\bx + \widetilde{\bz}) ]$ and $\overline{\bPhi}(\bx) = \E_{\widetilde{\bz},\overline{\bz}}[\bPsi(\bx + \widetilde{\bz} + \overline{\bz}) ] = \E_{\overline{\bz}}[\bPhi(\bx + \overline{\bz})]$ $\bPhit = \E_{\widetilde{\bz}^t}[\bPsi(\nabla F(\bxt) + \widetilde{\bz}^t) ]$, where the second equality follows from the fact that $\widetilde{\bz}$ and $\overline{\bz}$ are independent. Note that the mapping $\bPhi$ only accounts for the symmetric noise component, while $\overline{\bPhi}$ accounts for the entire noise. Using this notation, we introduce the shorthands $\bPsit \triangleq \bPsi(\bg^t) = \bPsi(\nabla F(\bxt) + \widetilde{\bz}^t + \overline{\bz}^t)$, $\bPhit \triangleq \bPhi(\nabla F(\bxt))$ and $\obPhit \triangleq \E_{\widetilde{\bz}^t,\overline{\bz}^t}[\bPsit] = \E_{\overline{\bz}^t}[\bPhi(\nabla F(\bxt) + \overline{\bz}^t)]$. The operator $\bPhit$ is the ``ideal'' search direction, in the sense that the expectation is taken with respect to the symmetric noise component only, while $\obPhit$ is the operator where the expectation is taken with respect to the full noise. Additionally, we define the effective noise $\bet = \bPsit - \obPhit$ and $\brt = \obPhit - \bPhit$. To bound $\brt$, we use an intermediary result from \cite{liu2025nonconvex}, stated next.

\begin{lemma}[Lemma 4.3 in \cite{liu2025nonconvex}]\label{lm:bounding-the-noise}
    Given a sequence of integrable random vectors $\bz^1,\ldots,\bzk \in \R^d$, $k \in \N$, such that $\E[\bzk \: \vert \: \mathcal{Z}_{k-1}] = \mathbf{0}$, where $\mathcal{Z}_{k-1} = \sigma(\{\bz^1,\ldots,\bz^{k-1}\}$ is the natural filtration, then for any $p \in [1,2]$, we have
    \begin{equation*}
        \E\Big\|\sum_{s = 1}^k\bz^s\Big\| \leq 2\sqrt{2}\E\left[\left(\sum_{s = 1}^k\|\bz^s\|^p \right)^{\frac{1}{p}}\right].
    \end{equation*}
\end{lemma}

We next prove Lemma \ref{lm:gap-control}.
{
\begin{proof}[Proof of Lemma \ref{lm:gap-control}]
    Note that $\brt = \obPhit - \bPhit = \E_{\obzt}[\bPhi(\nabla F(\bxt) + \obzt) - \bPhi(\nabla F(\bxt))]$, where the second equality follows from the fact that $\obzt$ and $\wbzt$ are mutually independent. Recall that $\bPhi: \R^d \mapsto \R^d$ is continuously differentiable\footnote{Follows from the fact that $\bPsi$ is continuously differentiable and the Lebesgue dominated convergence theorem.} and consider individual component functions $\bPhi_i: \R^d \mapsto \R$, $i \in [d]$, with $\bPhi(\bx) = [\bPhi_1(\bx),\ldots,\bPhi_d(\bx)]$. Applying the mean-value theorem for individual component functions, it follows that, for each $i \in [d]$, there exists a $\lambda_i \in (0,1)$, such that 
    \begin{equation*}
        \bPhi_i(\nabla F(\bxt) + \obzt) - \bPhi_i(\nabla F(\bxt)) = \langle \nabla \bPhi_i(\by^t_i),\obzt \rangle,
    \end{equation*} where $\by^t_i \triangleq \nabla F(\bxt) + \lambda_i \obzt$. Let $\partial \bPhi^t \in \R^{d \times d}$ be a matrix whose rows are the gradients of the component functions, i.e., 
    \begin{equation*}
        \partial \bPhi^t = \begin{bmatrix}
            \nabla \bPhi_1(\by^t_1)^\top \\ \vdots \\ \nabla \bPhi_d(\by^t_d)^\top   
        \end{bmatrix},
    \end{equation*} it then follows that $\bPhi(\nabla F(\bxt) + \obzt) - \bPhi(\nabla F(\bxt)) = \partial \bPhi^t \obzt$, hence $\brt = \E_{\obzt}[\partial \bPhi^t \obzt]$. Using this identity, we then have
    \begin{align*}
        \|\brt\| \stackrel{(a)}{\leq} \E_{\obzt}\|\partial \bPhi^t\obzt\| &\stackrel{(b)}{\leq} \sqrt{d}K\E\|\obzt\| \stackrel{(c)}{\leq} \frac{2\sqrt{2d}K}{B_t}\E\left[\left(\sum_{j = 2}^{B_t+2}\|\bz^t_j\|^p \right)^{\frac{1}{p}}\right] \\
        &\stackrel{(d)}{\leq} \frac{2\sqrt{2d}K}{B_t}\left(\sum_{j = 2}^{B_t+2}\E\|\bz^t_j\|^p \right)^{\frac{1}{p}} \stackrel{(e)}{\leq} 2\sqrt{2d}\sigma KB_t^{\frac{1-p}{p}}, 
    \end{align*} almost surely, where $(a)$ follows from Jensen's inequality, $(b)$ follows from Cauchy-Schwartz inequality, the bound $\|\partial \bPhi^t\|^2 \leq \sum_{i \in [d]}\|\nabla \bPhi_i(\by_i^t)\|^2$ and the uniform bound on the derivative of $\bPhi$ from Assumption \ref{asmpt:nonlin-nmsge}, which implies the same bound on the gradients of the component functions, $(c)$ follows from Lemma \ref{lm:bounding-the-noise}, $(d)$ follows from the fact that $p > 1$ and the reverse Jensen's inequality, while $(e)$ follows from applying the bound on the $p$-th moment of noise from Assumption \ref{asmpt:non-sym-p-moment}, completing the proof.
\end{proof}
}

We are now ready to prove Theorem \ref{thm:n-msge}. 

\begin{proof}[Proof of Theorem \ref{thm:n-msge}] 
    Using $L$-smoothness of $F$ and the update rule for \nmsge, we get 
    \begin{align}
        F(\bx^{t+1}) &\leq F(\bxt) - \alpha_t\langle \nabla F(\bxt), \bPsi^t \rangle + \frac{\alpha_t^2L}{2}\|\bPsi^t\|^2 \nonumber \\
        &\leq F(\bxt) - \alpha_t\langle \nabla F(\bxt), \bPhi^t \rangle - \alpha_t\langle \nabla F(\bxt), \overline{\bPhi}^t -\bPhi^t \rangle - \alpha_t\langle \nabla F(\bxt), \bPsi^t - \overline{\bPhi}^t \rangle + \frac{\alpha_t^2LC^2}{2} \nonumber \\
        &\leq F(\bxt) - \alpha_tG_t - \alpha_t\langle \nabla F(\bxt), \be^t \rangle - \alpha_t\langle \nabla F(\bxt), \brt \rangle + \frac{\alpha_t^2LC^2}{2}, \label{chap-beyond-sym-eq:nmsge-start}
    \end{align} where $G_t \triangleq \min\left\{\gamma_1\|\nabla F(\bxt)\|,\gamma_2\|\nabla F(\bxt)\|^2 \right\}$, $\bet = \bPsit - \obPhit$, $\brt = \overline{\bPhi}^t - \bPhi^t$, the second inequality follows by adding and subtracting $\bPhi^t$, $\overline{\bPhi}^t$ and using the uniform noise bound from Assumption \ref{asmpt:nonlin-nmsge}, while the third inequality follows from Lemma \ref{lm:key-iid}. Summing up the first $t$ terms and rearranging, we get
    \begin{equation}\label{chap-beyond-sym-eq:intermed-nonsym}
        \sum_{k = 1}^t\alpha_k G_k  - \Delta + \sum_{k = 1}^t\alpha_k\langle \nabla F(\bxk),\brk\rangle - \frac{LC^2}{2}\sum_{k = 1}^t\alpha_k^2 \leq -\sum_{k = 1}^t\alpha_k\langle \nabla F(\bxk), \bek \rangle.
    \end{equation} Define $Z_t \triangleq \sum_{k = 1}^t\alpha_k (M_k + \langle \nabla F(\bxk), \brk \rangle) - \Delta - \frac{LC^2}{2}\sum_{k = 1}^t\alpha_k^2$, where $M_k \triangleq G_k  - 4\alpha_kC^2\|\nabla F(\bxk)\|^2$, and consider the MGF of $Z_t$. Using \eqref{chap-beyond-sym-eq:intermed-nonsym} and following similar steps as in the proof of Theorem \ref{thm:n-sgd}, it readily follows that $\E[\exp(Z_t)] \leq 1$. Using the exponential Markov inequality, we have $\mathbb{P}(Z_t > \epsilon) \leq \exp(-\epsilon)\E[\exp(Z_t)] \leq \exp(-\epsilon)$, for any $\epsilon > 0$, or equivalently, for any $\delta \in (0,1)$
    \begin{equation}\label{chap-beyond-sym-eq:tail-bound}
        \mathbb{P}\left(\sum_{k = 1}^t\alpha_kM_k \leq \log(\nicefrac{1}{\beta}) + \Delta - \sum_{k = 1}^t\alpha_k\langle \nabla F(\bx^k), \brk\rangle + \frac{LC^2}{2}\sum_{k = 1}^t\alpha_k^2\right) \geq 1-\delta.
    \end{equation} Next, we focus on bounding the quantity $\sum_{k = 1}^t\alpha_k\langle \nabla F(\bx^k), \brk \rangle$. Using Young's inequality, i.e., $ab \leq \frac{\epsilon a^2}{2} + \frac{b^2}{2\epsilon}$, we get
    \begin{align*}
        -\sum_{k = 1}^t\alpha_k\langle \nabla F(\bx^k),\brk \rangle &\leq \frac{1}{2}\sum_{k = 1}^t\alpha_k^2\|\nabla F(\bx^k)\|^2 + \frac{1}{2}\sum_{k = 1}^t\|\brk\|^2 \\
        &\leq \frac{1}{2}\sum_{k = 1}^t\alpha_k^2\|\nabla F(\bx^k)\|^2 + 4\sigma^2dK^2\sum_{k = 1}^tB_k^{\frac{2(1-p)}{p}},
    \end{align*} almost surely, where we use $\epsilon = \alpha_k$ in Young's inequality and Lemma \ref{lm:gap-control} in the last inequality. Recalling the batch-size choice $B_k = k^{\frac{p}{2(p-1)}}$, we get, almost surely
    \begin{align}\label{chap-beyond-sym-eq:bdd-inner-prod}
        -\sum_{k = 1}^t\alpha_k\langle \nabla F(\bx^k), \brk \rangle &\leq \frac{1}{2}\sum_{k = 1}^t\left(8\sigma^2dK^2k^{-1} + \alpha_k^2\|\nabla F(\bx^k)\|^2\right) \nonumber \\
        &\leq 4\sigma^2dK^2\ln(t) + \frac{1}{2}\sum_{k = 1}^t\alpha_k^2\|\nabla F(\bx^k)\|^2.
    \end{align} Plugging \eqref{chap-beyond-sym-eq:bdd-inner-prod} into \eqref{chap-beyond-sym-eq:tail-bound}, we get, with probability at least $1 - \delta$
    \begin{equation*}
        \sum_{k = 1}^t\alpha_k \widetilde{M}_k \leq \log(\nicefrac{1}{\delta}) + \Delta + 4\sigma^2dK^2\ln(t) + \frac{LC^2}{2}\sum_{k = 1}^t\alpha_k^2,
    \end{equation*} where $\widetilde{M}_k \triangleq M_k - \frac{\alpha_k}{2}\|\nabla F(\bx^k)\|^2 = G_k - \alpha_k(4C^2+1/2)\|\nabla F(\bx^k)\|^2$. Combining with the previous inequality, using the fact that the sequence of step-sizes in non-increasing and dividing both sides by $\alpha_tt$, we get
    \begin{equation*}
        \frac{1}{t}\sum_{k = 1}^t \widetilde{M}_k \leq 2a^{-1}t^{\eta-1}\left(\log(\nicefrac{1}{\delta}) + \Delta + 4\sigma^2dK^2\ln(t) + \frac{a^2LC^2}{2}\sum_{k = 1}^t(k+1)^{-2\eta}\right).
    \end{equation*} We can bound the quantity $\widetilde{M}_k$ the same way as in Theorem \ref{thm:n-sgd}, with the  modified step-size requirement given by $a \leq \min\left\{ 1,\frac{\gamma_2}{8C^2+1}, \frac{(1-\eta)\gamma_1}{(8C^2+1)(LC+(1-\eta)\widetilde{\nabla})}\right\}$. Multiplying both sides by $2\gamma^{-1}$ completes the proof.
\end{proof}

{
\section{Computing $\beta,\gamma$}\label{app:example}

In this section we provide further insight into computing the constants $\gamma_1,\gamma_2$ and $\beta_1,\beta_2$ for specific noise and choice of nonlinearity. In particular, we focus on the setting from Examples \ref{ex:example-1} and \ref{ex:example-2}, namely the noise with PDF $P(\bx) = \prod_{i \in [d]}\rho(x_i)$, where $\rho(x) = \frac{\alpha - 1}{2(1 + |x|)^\alpha}$, for some $\alpha > 2$. We consider component-wise sign and smooth sign nonlinearities, given by $\psi_1(x) = \textnormal{sign}(x)$ and $\psi_2(x) = \tanh(x)$. Recalling Lemmas \ref{lm:key-iid} and \ref{lm:key-unified-state}, we know that $\gamma_1 = \frac{\phi_1^\prime(0)\xi}{2\sqrt{d}}$, $\gamma_2 = \frac{\phi_1^\prime(0)}{2d}$, with $\beta_1 = \frac{(\phi_2^\prime(0))^2}{4K_2\sqrt{d}}$ and $\beta_2 = \frac{\phi_2^\prime(0)}{2d}$. For the specific noise and sign nonlinearity, we proceed as follows. Applying Lemma 6 in \cite{polyak-adaptive-estimation}, it follows that $\phi_1^\prime(0) = (\textnormal{sign}(0^+) - \textnormal{sign}(0^-))\rho(0) = \alpha - 1$, where $\textnormal{sign}(0^+) = \lim_{x \rightarrow 0^+}\textnormal{sign}(x)$ (and same for $\textnormal{sign}(0^-)$). To estimate $\xi$, we first note that, for any $x \in \R$
\begin{align}\label{eq:phi-closed-form}
    \phi_1(x) = \int_{-\infty}^\infty\textnormal{sign}(x+z)\rho(z)dz &= \int_{-\infty}^{-|x|}-\rho(z)dz + \int_{-|x|}^{|x|}\textnormal{sign}(x)\rho(z)dz + \int_{|x|}^{\infty}\rho(z)dz \nonumber \\
    &= 2\textnormal{sign}(x)\int_{0}^{|x|}\rho(z)dz = \textnormal{sign}(x)\bigg(1 - \frac{1}{(1 + |x|)^{\alpha - 1}}\bigg),  
\end{align} where the third inequality follows from the identity $\int_{a}^b-f(x)dx = -\int_{-b}^{-a}f(-x)dx$ and the fact that $\rho$ is even. Using \eqref{eq:phi-closed-form}, it can be seen that, for any $x \in \R \setminus \{0\}$, we have $\phi_1^\prime(x) = \frac{(\alpha-1)}{(1+|x|)^{\alpha}}$. Using Taylor's expansion, it follows that
\begin{equation*}
    \phi_1(x) = \underbrace{\phi_1(0)}_{=0} + \phi^\prime_1(0)x + h_1(x)x = \phi_1^\prime(0)x + h_1(x)x,
\end{equation*} where $h_1: \R \mapsto \R$ is defined as
\begin{equation}\label{eq:reminder}
    h_1(x) = \phi^\prime_1(\tau_x x) - \phi^\prime_1(0) = (\alpha-1)\bigg(\frac{1}{(1+|\tau_xx|)^\alpha} - 1 \bigg),
\end{equation} for some $\tau_x \in (0,1)$. Recall that $\xi$ is defined in \cite[Lemma 3.2]{armacki2023high} as a strictly positive value such that $|h_1(x)| \leq \frac{\phi_1^\prime(0)}{2}$, for all $|x| \leq \xi$, hence, from \eqref{eq:reminder}, we have
\begin{equation*}
    |h_1(x)| \leq \frac{\phi_1^\prime(0)}{2} \iff 1 - \frac{1}{(1+|\tau_x x|)^\alpha} \leq \frac{1}{2} \iff |x\tau_x| \leq 2^{\frac{1}{\alpha}}-1.
\end{equation*} Since $\tau_x \in (0,1)$, it follows that $|x\tau_x| \leq |x|$, hence for any $|x| \leq 2^{\frac{1}{\alpha}}-1$, we have $|h_1(x)| \leq \frac{\phi_1^\prime(0)}{2}$, therefore, we can take $\xi$ to be $\xi = 2^{\frac{1}{\alpha}} - 1$. Finally, to simplify exposition, we can note that, for any $b > 0$, we have $2^{b} = 1 + b + h(b)$, where $h(b) = \sum_{k = 2}^{\infty}\binom{b}{k}$ and $h(b) \rightarrow 0$ as $b \rightarrow 0$, therefore $\xi = 2^{\frac{1}{\alpha}}-1 \approx \frac{1}{\alpha}$, for $\alpha$ sufficiently large. It then readily follows that $\gamma_1 = \frac{\alpha-1}{2\alpha\sqrt{d}}$ and $\gamma_2 = \frac{\alpha-1}{2d}$. Next, recalling the proof of Lemma \ref{lm:polyak-tsypkin}, it follows that 
\begin{align*}
    \phi_2^\prime(0) = \int_{-\infty}^{\infty}\psi_2^\prime(z)\rho(z)dz &= 1 - \int_{-\infty}^{\infty}\tanh^2(z)\rho(z)dz \stackrel{(a)}{=} 1 - (\alpha-1)\int_{0}^{\infty}\frac{\tanh^2(z)}{(z+1)^{\alpha}}dz \\ &\stackrel{(b)}{=} 1 - (\alpha-1)\int_{0}^\infty\left( \frac{1}{(z+1)^{\alpha}} - \frac{4e^{2z}}{(e^{2z}+1)^2(z+1)^{\alpha}} \right) dz \\
    &\stackrel{(c)}{=} 1 - (\alpha-1)\left(\frac{1}{\alpha - 1} - J_{\alpha} \right) = (\alpha - 1)J_{\alpha},
\end{align*} where $(a)$ follows from the evenness of $\tanh^2$ and $\rho$, in $(b)$ we use the identity $\tanh^2(z) = 1 - \frac{4e^{2z}}{(e^{2z}+1)^2}$, while $(c)$ follows from $\int_{0}^{\infty}\frac{1}{(z+1)^\alpha}dz = \frac{1}{\alpha-1}$ and defining $J_{\alpha} \triangleq \int_{0}^{\infty}\frac{4e^{2z}}{(e^{2z}+1)^2(z+1)^\alpha}dz$. Moreover, recalling that $\psi^{\prime\prime}$ is uniformly bounded, it can then be shown that $|\psi^{\prime\prime}(x)| \leq \frac{4\sqrt{3}}{9} \triangleq K_2$. Using these facts, it now readily follows that $\beta_1 = \frac{3\sqrt{3}(\alpha-1)^2J^2_{\alpha}}{16\sqrt{d}}$ and $\beta_2 = \frac{(\alpha-1)J_\alpha}{2d}$, completing the derivations.

\section{An illustrative example}\label{app:bounds}

In this section we specialize the rates from Theorem \ref{thm:n-sgd} for specific choices of nonlinearity and noise, showing how our theory can be used to inform the choice of nonlinearity for specific problem instances. A similar example appears in \cite{armacki2023high}. We consider the noise with PDF from Example \ref{ex:example-1}, for some $\alpha > 2$ and step-size with $\eta = 1/2$. Consider the component-wise and joint clipping, with thresholds $m > 1$ and $M > 0$, respectively. It can be shown that $C_{cc} = m\sqrt{d}$, $\beta_{1,cc} = \frac{[1-(m+1)^{-\alpha}](m-1)}{2\sqrt{d}}$, $\beta_{2,cc} = \frac{1-(m+1)^{-\alpha}}{2d}$ for component-wise and $C_{jc} = M$, $\beta_{1,jc} = \left[\frac{\alpha - 1}{2}\right]^d\min\{1/2,M/2\}$, $\beta_{2,jc} = \left[\frac{\alpha - 1}{2}\right]^d\min\{1,M\}$ for joint clipping, see \cite{armacki2023high} for details. As such, for $d$ sufficiently large, we have $\beta_{cc} = \frac{1-(m+1)^{-\alpha}}{2d}$ and $\beta_{jc} = \left[\frac{\alpha - 1}{2}\right]^d\min\{1/2,M/2\}$. We focus on the problem related constants that figure in the bound, ignoring the rate and global constants. Moreover, noting that the second term is of the form $\mathcal{O}\Big(\frac{a}{\beta}\Big)$, with $a \leq \min\left\{ 1,\frac{\beta_2}{8C^2}, \frac{(1-\eta)\beta_1}{8C^2(LC+(1-\eta)\widetilde{\nabla})}\right\}$, it then follows that the dependence on $\beta$ cancels out in the second term and we can focus on the first term only. For simplicity, let $a = \frac{\beta}{C^2}$, ignoring the other problem related constants. We then have the following expressions
\begin{align*}
    \text{Component clipping: } &\frac{md^3(\log(\nicefrac{1}{\delta}) + \Delta)}{[1-(m+1)^{-\alpha}]^2}, \\
    \text{Joint clipping: } &\frac{M^2(\log(\nicefrac{1}{\delta}) + \Delta)}{[(\alpha-1)/2]^{2d}\min\{1,M^2\}}.
\end{align*} Note that the leading term for component clip shows a dependence on problem dimension, in the form of $d^3$, while the leading term for joint clip has an exponential dependence on $d$, via $[(\alpha-1)/2]^{-2d}$. As $\alpha$ is an inherent property of the noise, whenever $\alpha \in (2,3)$, (i.e., variance is unbounded and noise is heavy-tailed), we have $[(\alpha-1)/2]^{-2d} \rightarrow \infty$ as $d \rightarrow \infty$, at an exponential rate. Thus, joint clipping has a significantly worse dependence on problem dimension compared to component-wise clipping, indicating that component-wise clipping is preferred under this noise and high-dimensional problems. This is further confirmed in \cite{zhang2020adaptive}, who empirically show that component clipping shows better dimension dependence for some noise instances than the joint clipping operator.

\section{On the noise from Example \ref{ex:example-1}}\label{app:power-law}

As discussed in the main body, the noise considered in Example 1 is an instance of noise with a power-law tail decay, i.e., a noise $\bz \in \R^d$, such that $\Prob(\|\bz\| > t) = Ct^{-\alpha}$, for some $\alpha,C > 0$ and any $t > 0$. Instances of noise with power-law tail decay include, among others, Pareto and symmetric $\alpha$-stable distributions, which have been widely observed in practice, see, e.g., \cite{simsekli2019tail,heavy-tail-phenomena,heavy-tail-book,pmlr-v238-battash24a}, and references therein. To see that the noise from Example \ref{ex:example-1} indeed has power-law tail decay, consider first a random variable $Z$ with PDF $\rho(z) = \frac{\alpha-1}{2(|z|+1)^\alpha}$, where $\alpha > 2$. Then for any $t > 0$, we have
\begin{equation*}
    \Prob(Z > t) = \int_{t}^\infty\frac{\alpha-1}{2(z+1)^\alpha}dz = \frac{\alpha-1}{2}\int_{t+1}^{\infty}\frac{1}{s^{\alpha}}ds = \frac{\alpha-1}{2}\times\frac{(t+1)^{-\alpha + 1}}{\alpha - 1} = \frac{1}{2}(t+1)^{-\beta},
\end{equation*} where $\beta \triangleq \alpha - 1 > 1$ (since $\alpha > 2$). Next, consider the multivariate version from Example \ref{ex:example-1}, namely $\bz = [Z_1,\ldots,Z_d]$, where $Z_i$ are IID, with PDF $\rho$. Using the fact that $\max_{i \in [d]}|Z_i| \leq \|\bz\| \leq \sqrt{d}\max_{i \in [d]}|Z_i|$, we get, for any $t > 0$
\begin{align*}
    \Prob(\|\bz\| > t) &\leq \Prob(\max_{i \in [d]}|Z_i| > t/\sqrt{d}) = 1 - \Prob(\max_{i \in [d]}|Z_i| \leq t/\sqrt{d}) \\ 
    &= 1 - \big(1 - \Prob(|Z_1| > t/\sqrt{d})\big)^d = 1 - \big(1 - (t/\sqrt{d}+1)^{-\beta})^d \leq d(t/\sqrt{d} + 1)^{-\beta},
\end{align*} where in the last inequality we used the fact that $1 - (1 - x)^d \leq dx$, for any $x \in [0,1]$ and $d \geq 1$ (which itself stems from Bernoulli's inequality, i.e., $(1 + y)^r \geq 1 + ry$, for any $r \in \R \setminus (0,1)$ and $y \geq -1$). Additionally, for any $t \geq d^{\frac{1}{\beta}}$, we have
\begin{equation*}
    \Prob(\|\bz\| > t) \geq \Prob(\max_{i \in [d]}|Z_i| > t) = 1 - (1 - (t+1)^{-\beta})^d \geq \frac{d}{2}(t+1)^{-\beta},
\end{equation*} where the last inequality follows from $1 - (1-x)^d \geq \frac{dx}{1+dx}$ (again stemming from another variant of Bernoulli's inequality, i.e., $(1+y)^r \leq \frac{1}{1-ry}$, for $y \in [-1,\frac{1}{r})$ and $r \geq 0$) and the threshold $t \geq d^{\frac{1}{\beta}}$. Therefore, for $t$ sufficiently large, we have
\begin{equation*}
    \frac{d}{2}(t+1)^{-\beta} \leq \Prob(\|\bz\| > t) \leq d(t/\sqrt{d} + 1)^{-\beta},
\end{equation*} demonstrating a power-law tail decay.
}

\end{document}